\documentclass[12pt]{article}
\usepackage{microtype}
\usepackage{natbib}
\ifdefined\usebigfont

\usepackage{times}
\usepackage[fontsize=13pt]{scrextend}
\usepackage[left=1.5in,right=1.5in,top=1.75in,bottom=1.75in]{geometry}
\else
\usepackage{fullpage}
\usepackage{times}
\fi

\usepackage{enumerate}
\usepackage{times}
\usepackage[OT1]{fontenc}
\usepackage{amsmath,amssymb}
\usepackage[usenames]{color}

\usepackage{smile}
\usepackage{multirow}
\usepackage[colorlinks,
            linkcolor=red,
            anchorcolor=blue,
            citecolor=blue
            ]{hyperref}

\def\tr{\mathop{\text{tr}}\kern.2ex}

\DeclareMathOperator*{\softmax}{softmax}
\long\def\comment#1{}

\def\vec{\mathop{\text{vec}}}
\def\tr{\mathop{\text{Tr}}}

\def\cS{{\mathcal{S}}}

\def\cX{{\mathcal{X}}}

\def\cP{{\mathcal{P}}}

\def\cT{{\mathcal{T}}}

\def\tr{{\text{Tr}}}
\def\tf{{\text{F}}}

\def\pie{{\pi_{\text{exp}}}}
\def\mue{{\mu_{\text{exp}}}}

\newcommand{\bel}{\begin{eqnarray}\label}
\newcommand{\eel}{\end{eqnarray}}
\newcommand{\bes}{\begin{eqnarray*}}
\newcommand{\ees}{\end{eqnarray*}}

\def\real{{\mathbb{R}}}

\newcommand{\la}{\langle}
\newcommand{\ra}{\rangle}

\usepackage{mathrsfs}

\usepackage{hyperref}
\usepackage{breakcites}
\def\##1\#{\begin{align}#1\end{align}}
\def\$#1\${\begin{align*}#1\end{align*}}
\usepackage{enumitem}

\begin{document}

% !TEX root = neural_TD.tex
\title{ Neural Temporal-Difference and Q-Learning\\ Provably Converge to Global Optima}
\author{
	Qi Cai
	\thanks{Department of Industrial Engineering and Management Sciences, Northwestern University}
	\qquad
	Zhuoran Yang
	\thanks{Department of Operations Research and Financial Engineering, Princeton University}
	\qquad
	Jason D. Lee
	\thanks{Department of Electrical Engineering, Princeton University}
	\qquad
	Zhaoran Wang
	\footnotemark[1]
}
\date{}
\maketitle 

\begin{abstract}
%\begin{flushleft}
Temporal-difference learning (TD), coupled with neural networks, is among the most fundamental building blocks of deep reinforcement learning. However, due to the nonlinearity in value function approximation, such a coupling leads to nonconvexity and even divergence in optimization. As a result, the global convergence of neural TD remains unclear. In this paper, we prove for the first time that neural TD converges at a sublinear rate to the global optimum of the mean-squared projected Bellman error for policy evaluation. In particular, we show how such global convergence is enabled by the overparametrization of neural networks, which also plays a vital role in the empirical success of neural TD. Beyond policy evaluation, we establish the global convergence of neural (soft) Q-learning, which is further connected to that of policy gradient algorithms. 
%\end{flushleft}
\end{abstract}
% !TEX root = neural_TD.tex

\section{Introduction}
%\begin{flushleft}
Given a policy, temporal-different learning (TD) \citep{sutton1988learning} aims to learn the corresponding (action-)value function by following the semigradients of the mean-squared Bellman error in an online manner. As the most-used policy evaluation algorithm, TD serves as the ``critic'' component of many reinforcement learning algorithms, such as the actor-critic algorithm \citep{konda2000actor} and trust-region policy optimization \citep{schulman2015trust}. In particular, in deep reinforcement learning, TD is often applied to learn value functions parametrized by neural networks \citep{lillicrap2015continuous, mnih2016asynchronous, haarnoja2018soft}, which gives rise to neural TD. As policy improvement relies crucially on policy evaluation, the optimization efficiency and statistical accuracy of neural TD are critical to the performance of deep reinforcement learning. Towards theoretically understanding deep reinforcement learning, the goal of this paper is to characterize the convergence of neural TD.

Despite the broad applications of neural TD, its convergence remains rarely understood. Even with linear value function approximation, the nonasymptotic convergence of TD remains open until recently \citep{bhandari2018finite, lakshminarayanan2018linear, dalal2018finite, srikant2019finite}, although its asymptotic convergence is well understood \citep{jaakkola1994convergence, tsitsiklis1997analysis, borkar2000ode, kushner2003stochastic, borkar2009stochastic}. Meanwhile, with nonlinear value function approximation, TD is known to diverge in general \citep{baird1995residual, boyan1995generalization, tsitsiklis1997analysis, chung2018two, achiam2019towards}. To remedy such an issue, \cite{bhatnagar2009convergent} propose nonlinear (gradient) TD, which uses the tangent vectors of nonlinear value functions in place of the feature vectors in linear TD. Unlike linear TD, which converges to the global optimum of the mean-squared projected Bellman error (MSPBE), nonlinear TD is only guaranteed to converge to a local optimum asymptotically. As a result, the statistical accuracy of the value function learned by nonlinear TD remains unclear. In contrast to such conservative theory, neural TD, which straightforwardly combines TD with neural networks without the explicit local linearization in nonlinear TD, often learns a desired value function that generalizes well to unseen states in practice \citep{duan2016benchmarking, amiranashvili2018td, henderson2018deep}. Hence, a gap separates theory from practice. %schulman2015high

There exist three obstacles towards closing such a theory-practice gap: (i) MSPBE has an expectation with respect to the transition dynamics within the squared loss, which forbids the construction of unbiased stochastic gradients \citep{sutton2018reinforcement}. As a result, even with linear value function approximation, TD largely eludes the classical optimization framework, as it follows biased stochastic semigradients. (ii) When the value function is parametrized by a neural network, MSPBE is nonconvex in the weights of the neural network, which may introduce undesired stationary points such as local optima and saddle points \citep{jain2017non}. As a result, even an ideal algorithm that follows the population gradients of MSPBE may get trapped. (iii) Due to the interplay between the bias in stochastic semigradients and the nonlinearity in value function approximation, neural TD may even diverge \citep{baird1995residual, boyan1995generalization, tsitsiklis1997analysis}, instead of converging to an undesired stationary point, as it lacks the explicit local linearization in nonlinear TD \citep{bhatnagar2009convergent}. Such divergence is also not captured by the classical optimization framework.

\vskip4pt
\noindent{\bf Contribution:} Towards bridging theory and practice, we establish the first nonasymptotic global rate of convergence of neural TD. In detail, we prove that randomly initialized neural TD converges to the global optimum of MSPBE at the rate of $1/T$ with population semigradients and at the rate of $1/\sqrt{T}$ with stochastic semigradients. Here $T$ is the number of iterations and the (action-)value function is parametrized by a sufficiently wide multi-layer neural network. Moreover, we prove that the projection in MSPBE allows for a sufficiently rich class of functions, which has the same representation power of a reproducing kernel Hilbert space associated with the random initialization. As a result, for a broad class of reinforcement learning problems, neural TD attains zero MSPBE. Beyond policy evaluation, we further establish the global convergence of neural (soft) Q-learning, which allows for policy improvement. In particular, we prove that, under stronger regularity conditions, neural (soft) Q-learning converges at the same rate of neural TD to the global optimum of MSPBE for policy optimization. Also, by exploiting the connection between (soft) Q-learning and policy gradient algorithms \citep{schulman2017equivalence, haarnoja2018soft}, we establish the global convergence of a variant of the policy gradient algorithm \citep{williams1992simple, szepesvari2010algorithms, sutton2018reinforcement}. 

At the core of our analysis is the overparametrization of the multi-layer neural network for value function approximation, which enables us to circumvent the three obstacles above. In particular, overparametrization leads to an implicit local linearization that varies smoothly along the solution path, which mirrors the explicit one in nonlinear TD \citep{bhatnagar2009convergent}. Such an implicit local linearization enables us to circumvent the third obstacle of possible divergence. Moreover, overparametrization allows us to establish a notion of one-point monotonicity \citep{harker1990finite, facchinei2007finite} for the semigradients followed by neural TD, which ensures its evolution towards the global optimum of MSPBE along the solution path. Such a notion of monotonicity enables us to circumvent the first and second obstacles of bias and nonconvexity. Broadly speaking, our theory backs the empirical success of overparametrized neural networks in deep reinforcement learning. In particular, we show that instead of being a curse, overparametrization is indeed a blessing for minimizing MSPBE in the presence of bias, nonconvexity, and even divergence. 

\vskip4pt
\noindent{\bf More Related Work:} There is a large body of literature on the convergence of linear TD under both asymptotic \citep{jaakkola1994convergence, tsitsiklis1997analysis, borkar2000ode, kushner2003stochastic, borkar2009stochastic} and nonasymptotic \citep{bhandari2018finite, lakshminarayanan2018linear, dalal2018finite, srikant2019finite} regimes. See \cite{dann2014policy} for a detailed survey. In particular, our analysis is based on the recent breakthrough in the nonasymptotic analysis of linear TD \citep{bhandari2018finite} and its extension to linear Q-learning \citep{zou2019finite}. An essential step of our analysis is bridging the evolution of linear TD and neural TD through the implicit local linearization induced by overparametrization. See also the concurrent work of \cite{brandfonbrener2019geometric, brandfonbrener2019expected, agazzi2019temporal} on neural TD, which however requires the state space to be finite. 

To incorporate nonlinear value function approximation into TD, \cite{bhatnagar2009convergent} propose the first convergent nonlinear TD based on explicit local linearization, which however only converges to a local optimum of MSPBE. See \cite{geist2013algorithmic, bertsekas2019feature} for a detailed survey. In contrast, we prove that, with the implicit local linearization induced by overparametrization, neural TD, which is simpler to implement and more widely used in deep reinforcement learning than nonlinear TD, provably converges to the global optimum of MSPBE. 

There exist various extensions of TD, including least-squares TD \citep{bradtke1996linear, boyan1999least, lazaric2010finite, ghavamzadeh2010lstd, tu2017least} and gradient TD \citep{sutton2009fast, sutton2009convergent, bhatnagar2009convergent, liu2015finite, du2017stochastic, wang2017finite, touati2017convergent}. In detail, least-squares TD is based on batch update, which loses the computational and statistical efficiency of the online update in TD. Meanwhile, gradient TD follows unbiased stochastic gradients, but at the cost of introducing another optimization variable. Such a reformulation leads to bilevel optimization, which is less stable in practice when combined with neural networks \citep{pfau2016connecting}. As a result, both extensions of TD are less widely used in deep reinforcement learning \citep{duan2016benchmarking, amiranashvili2018td, henderson2018deep}. Moreover, when using neural networks for value function approximation, the convergence to the global optimum of MSPBE remains unclear for both extensions of TD.  

Our work is also related to the recent breakthrough in understanding overparametrized neural networks, especially their generalization error \citep{zhang2016understanding, neyshabur2018towards, li2018learning, allen2018learninga, allen2018learningb, allen2018convergence, zou2018stochastic, arora2019fine, cao2019bounds, cao2019generalization}. See \cite{fan2019selective} for a detailed survey. In particular, \cite{daniely2017sgd, chizat2018note, jacot2018neural, li2018learning, allen2018learninga, allen2018learningb, allen2018convergence, zou2018stochastic, arora2019fine, cao2019bounds, cao2019generalization, lee2019wide} characterize the implicit local linearization in the context of supervised learning, where we train an overparametrized neural network by following the stochastic gradients of the mean-squared error. In contrast, neural TD does not follow the stochastic gradients of any objective function, hence leading to possible divergence, which makes the convergence analysis more challenging. 
%\end{flushleft}

%ref: Policy Evaluation with Temporal Differences: A Survey and Comparison
%learning to predict by the methods of temporal differences
%data-efficiency (online)
%abeel's empirical paper actor-critic
%Algorithms for reinforcement learning

% !TEX root = neural_TD.tex

\section{Background}
%\begin{flushleft} 
In Section \ref{peirl}, we briefly review policy evaluation in reinforcement learning. In Section \ref{onl}, we introduce the corresponding optimization formulations.
\subsection{Policy Evaluation}\label{peirl}
We consider a Markov decision process $(\cS,\cA,\cP,r,\gamma)$, in which an agent interacts with the environment to learn the optimal policy that maximizes the expected total reward. At the $t$-th time step, the agent has a state $s_t\in\cS$ and takes an action $a_t\in\cA$. Upon taking the action, the agent enters the next state $s_{t+1}\in\cS$ according to the transition probability $\cP(\cdot\,|\,s_t,a_t)$ and receives a random reward $r_t=r(s_t,a_t)$ from the environment. The action that the agent takes at each state is decided by a policy $\pi:\cS\rightarrow\Delta$, where $\Delta$ is the set of all probability distributions over $\cA$. The performance of policy $\pi$ is measured by the expected total reward, $J(\pi)=\EE[ \sum_{t=0}^\infty \gamma^t r_t \,|\,a_t\sim\pi(s_t) ]$,
where $\gamma < 1$ is the discount factor. 

Given policy $\pi$, policy evaluation aims to learn the following two functions, the value function $V^{\pi}(s)=\EE[ \sum_{t=0}^\infty \gamma^t r_t \,|\,s_0=s, a_t\sim\pi(s_t)]$ and the action-value function (Q-function) $Q^{\pi}(s,a)=\EE[ \sum_{t=0}^\infty \gamma^t r_t \,|\,s_0=s,a_0=a, a_t\sim\pi(s_t)]$. Both functions form the basis for policy improvement. Without loss of generality, we focus on learning the Q-function in this paper. We define the Bellman evaluation operator, 
\#\label{bmeqn}
\cT^{\pi}Q(s,a)=\EE[r(s,a)+\gamma Q(s',a')\,|\,s'\sim\cP(\cdot\,|\,s,a), a'\sim \pi(s')],
\#
for which $Q^{\pi}$ is the fixed point, that is, the solution to the Bellman equation $Q = \cT^{\pi} Q$.

\subsection{Optimization Formulation}\label{onl}
Corresponding to \eqref{bmeqn}, we aim to learn $Q^\pi$ by minimizing the mean-squared Bellman error (MSBE), 
\#\label{513940}
\min_{\theta}\text{MSBE}(\theta)=\EE_{(s,a)\sim\mu}\bigl[ \bigl(\hat{Q}_{\theta}(s,a)-\cT^{\pi}\hat{Q}_{\theta}(s,a)\bigr)^2\bigr],
\#
where the Q-function is parametrized by $\hat{Q}_{\theta}$ with parameter $\theta$. Here $\mu$ is the stationary distribution of $(s, a)$ corresponding to policy $\pi$. Due to Q-function approximation, we focus on minimizing the following surrogate of MSBE, namely the projected mean-squared Bellman error (MSPBE), 
\#\label{5131118}
\min_{\theta}\text{MSPBE}(\theta)=\EE_{(s,a)\sim\mu}\bigl[ \bigl(\hat{Q}_{\theta}(s,a)-\Pi_{\cF}\cT^{\pi}\hat{Q}_{\theta}(s,a)\bigr)^2\bigr].
\#
Here $\Pi_{\cF}$ is the projection onto a function class $\cF$. For example, for linear Q-function approximation \citep{sutton1988learning}, $\cF$ takes the form $\{\hat{Q}_{\theta'}:\theta'\in\Theta\}$, where $\hat{Q}_{\theta'}$ is linear in $\theta'$ and $\Theta$ is the set of feasible parameters. As another example, for nonlinear Q-function approximation \citep{bhatnagar2009convergent}, $\cF$ takes the form $\{\hat{Q}_\theta+\nabla_\theta\hat{Q}_\theta^\top(\theta'-\theta):\theta'\in\Theta\}$, which consists of the local linearization of $\hat{Q}_{\theta'}$ at $\theta$.

Throughout Sections \ref{secalgo}-\ref{qlearning}, we assume that we are able to sample tuples in the form of $(s,a,r,s',a')$ from the stationary distribution of policy $\pi$ in an independent and identically distributed manner. Our analysis is extended to handle temporal dependence in Appendix \ref{secmarkov} using the proof techniques of \cite{bhandari2018finite}. With a slight abuse of notation, we use $\mu$ to denote the stationary distribution of $(s,a,r,s',a')$ corresponding to policy $\pi$ and any of its marginal distributions. 
%\end{flushleft}

\section{Neural Temporal-Difference Learning}\label{secalgo}
%\begin{flushleft}
TD updates the parameter $\theta$ of the Q-function by taking the stochastic semigradient descent step \citep{sutton1988learning, szepesvari2010algorithms, sutton2018reinforcement}, 
\#\label{5131033}
\theta'\leftarrow\theta-\eta\cdot \big( \hat{Q}_{\theta}(s,a)-r(s,a)-\gamma\hat{Q}_{\theta}(s',a') \bigr)\cdot \nabla_{\theta}\hat{Q}_{\theta}(s,a),
\#
which corresponds to the MSBE in \eqref{513940}. Here $(s,a,r,s',a') \sim \mu$ and $\eta > 0$ is the stepsize. In a more general context, \eqref{5131033} is referred to as TD(0). In this paper, we focus on TD(0), which is abbreviated as TD, and leave the extension to TD($\lambda$) to future work. 

In the sequel, we consider $\cS$ to be continuous and $\cA$ to be finite. We represent the state-action pair $(s,a)\in\cS\times\cA$ by a vector $x=\psi(s,a)\in\cX\subseteq\real^d$ with $d>2$, where $\psi$ is a given one-to-one feature map. With a slight abuse of notation, we use $(s,a)$ and $x$ interchangeably. Without loss of generality, we assume that $\|x\|_2=1$ and $|r(x)|$ is upper bounded by a constant $\overline{r}>0$ for any $x\in\cX$. We use a two-layer neural network 
\#\label{nnpara}
\hat{Q}(x;W)=\frac{1}{\sqrt{m}}\sum_{r=1}^m b_r\sigma(W_r^\top x)
\#
to parametrize the Q-function, which is extended to a multi-layer neural network in Appendix \ref{deep}. Here $\sigma$ is the rectified linear unit (ReLU) activation function $\sigma(y)=\max\{0,y\}$ and the parameter $\theta = (b_1,\ldots, b_m, W_1,\ldots,W_m)$ are initialized as $b_r\sim\text{Unif}(\{-1,1\})$ and $W_r\sim N(0,I_d/d)$ for any $r\in[m]$ independently. During training, we only update $W=(W_1,\ldots,W_m) \in \RR^{md}$, while keeping $b = (b_1,\ldots, b_m) \in \RR^m$ fixed as the random initialization. To ensure global convergence, we incorporate an additional projection step with respect to $W$. See Algorithm \ref{td0} for a detailed description.

\begin{algorithm}
\caption{Neural TD}
\begin{algorithmic}[1]
\STATE \textbf{Initialization:} $b_r\sim\text{Unif}(\{-1,1\})$, $W_r(0)\sim N(0,I_d/d)$ $(r\in[m])$, $\overline{W}=W(0)$,
 \label{Istep}\\
{\color{white}\textbf{Initialization:}} $S_B=\{W\in\real^{md}:\|W-W(0)\|_2\le B\}$ $(B>0)$
\STATE \textbf{For} {$t=0$ to $T-2$}:
\STATE \hspace{0.15in} Sample a tuple $(s,a,r,s',a')$ from the stationary distribution $\mu$ of policy $\pi$
\STATE \hspace{0.15in} Let $x=(s,a)$, $x'=(s',a')$
\STATE \hspace{0.15in} Bellman residual calculation: $\delta \leftarrow \hat{Q}(x;W(t))-r-\gamma\hat{Q}(x';W(t))$
\STATE \hspace{0.15in} TD update: $\tilde{W}(t+1)\leftarrow W(t)-\eta\delta\cdot\nabla_{W}\hat{Q}(x;W(t))$  \label{algo_td}\\
\STATE \hspace{0.15in} Projection: $W(t+1)\leftarrow \argmin_{W\in S_B}\|W-\tilde{W}(t+1)\|_2$
\STATE \hspace{0.15in} Averaging: $\overline{W}\leftarrow \frac{t+1}{t+2}\cdot \overline{W}+\frac{1}{t+2}\cdot W(t+1)$
\STATE \textbf{End For} 
\STATE \textbf{Output:}  $\hat{Q}_\text{out}(\cdot)\leftarrow \hat{Q}(\cdot\,;\overline{W})$
\end{algorithmic}\label{td0}
\end{algorithm} 

To understand the intuition behind the global convergence of neural TD, note that for the TD update in \eqref{5131033}, we have from \eqref{bmeqn} that
\# \label{55311}
&\EE_{(s, a, r, s', a') \sim \mu}\bigl[
\bigl(
\hat{Q}_\theta(s,a)-r(s,a)-\gamma\hat{Q}_{\theta}(s',a')\bigr) \cdot \nabla_{\theta}  \hat{Q}_\theta(s,a)
\bigr]\notag\\
&\quad = \EE_{(s, a) \sim \mu}\bigl[
\bigl(
\hat{Q}_\theta(s,a)- \EE[r(s,a)+\gamma Q(s',a')\,|\,s'\sim\cP(\cdot\,|\,s,a), a'\sim \pi(s')]\bigr) \cdot \nabla_{\theta}  \hat{Q}_\theta(s,a)
\bigr]\notag\\
&\quad = \EE_{(s, a) \sim \mu}\bigl[
\underbrace{\bigl(
\hat{Q}_\theta(s,a)-\cT^{\pi}\hat{Q}_{\theta}(s,a)\bigr)}_{\displaystyle\text{(i)}} \cdot 
\underbrace{\nabla_{\theta}\hat{Q}_\theta(s,a)}_{\displaystyle\text{(ii)}}
\bigr].
\#
Here (i) is the Bellman residual at $(s,a)$, while (ii) is the gradient of the first term in (i). Although the TD update in \eqref{5131033} resembles the stochastic gradient descent step for minimizing a mean-squared error, it is not an unbiased stochastic gradient of any objective function. However, we show that the TD update yields a descent direction towards the global optimum of the MSPBE in \eqref{5131118}. Moreover, as the neural network becomes wider, the function class $\cF$ that $\Pi_{\cF}$ projects onto in \eqref{5131118} becomes richer. Correspondingly, the MSPBE reduces to the MSBE in \eqref{513940} as the projection becomes closer to identity, which implies the recovery of the desired Q-function $Q^\pi$ such that $Q^\pi = \cT^{\pi} Q^\pi$. See Section \ref{tr} for a more rigorous characterization. 
%\end{flushleft}

\section{Main Results}\label{tr}
%\begin{flushleft}
In Section \ref{5131049}, we characterize the global optimality of the stationary point attained by Algorithm \ref{td0} in terms of minimizing the MSPBE in \eqref{5131118} and its other properties. In Section \ref{5131051}, we establish the nonasymptotic global rates of convergence of neural TD to the global optimum of the MSPBE when following the population semigradients in \eqref{55311} and the stochastic semigradients in \eqref{5131033}, respectively. Throughout Section \ref{tr}, we focus on two-layer neural networks. In Appendix \ref{deep}, we present the extension to multi-layer neural networks.

We use the subscript $\EE_{\mu}[\cdot]$ to denote the expectation with respect to the randomness of the tuple $(s,a,r,s,a')$ (or its concise form $(x,r,x')$) conditional on all other randomness, e.g., the random initialization and the random current iterate. Meanwhile, we use the subscript $\EE_{\text{init}, \mu}[\cdot]$ when we are taking expectation with respect to all randomness, including the random initialization. 

\subsection{Properties of Stationary Point}\label{5131049}
We consider the population version of the TD update in Line \ref{algo_td} of Algorithm \ref{td0}, 
\# \label{mp}
\tilde{W}(t+1)\leftarrow W(t)-\eta\cdot\EE_{\mu}\bigl[\delta\bigl(x, r, x';W(t)\bigr)\cdot\nabla_{W}\hat{Q}\bigl(x;W(t)\bigr)\bigr],
\# 
where $\mu$ is the stationary distribution and $\delta(x, r, x';W(t)) = \hat{Q}(x;W(t))-r-\gamma\hat{Q}(x';W(t))$ is the Bellman residual at $(x, r, x')$. The stationary point $W^\dagger$ of \eqref{mp} satisfies the following stationarity condition,
\#\label{truesp}
\EE_{\mu}[\delta(x, r, x';W^\dagger)\cdot\nabla_{W}\hat{Q}(x;W^\dagger)]^\top (W-W^\dagger)\ge0,~~\text{for any}~W\in S_B.
\#
Also, note that 
\$
\hat{Q}(x;W)=\frac{1}{\sqrt{m}}\sum_{r=1}^m b_r\sigma(W_r^\top x) = \frac{1}{\sqrt{m}}\sum_{r=1}^m b_r\ind\{W_r^\top x>0\} W_r^\top x
\$ 
and $\nabla_{W_r}\hat{Q}(x;W)=b_r\ind\{W_r^\top x>0\} x$ almost everywhere in $\real^{md}$. 
%Thus by \eqref{truesp} we have
%\#\label{45507}
%\EE_{\mu}\bigl[\delta(x, r, x';W^\dagger)\cdot \bigl(\textstyle{\sum_{r=1}^m} b_r\ind\{(W^\dagger_r)^\top x>0\} W_r^\top x-\hat{Q}(x;W^\dagger)\bigr)\bigr]\ge0,~~\forall~W\in S_B.
%\#
Meanwhile, recall that $S_B=\{W\in\real^{md}:\|W-W(0)\|_2\le B\}$. We define the function class
\#\label{514022}
\cF^\dagger_{B,m}=\biggl\{\frac{1}{\sqrt{m}}\sum_{r=1}^m b_r\ind\{(W^\dagger_r)^\top x>0\} W_r^\top x: W\in S_B\biggr\},
\#
 which consists of the local linearization of $\hat{Q}(x;W)$ at $W=W^\dagger$. Then \eqref{truesp} takes the following equivalent form
\# \label{56202}
\bigl\langle \hat{Q}(\cdot\,;W^\dagger)-\cT^{\pi}\hat{Q}(\cdot\,;W^\dagger),f(\cdot) - \hat{Q}(\cdot\,;W^\dagger) \bigr\rangle_{\mu}\ge0,~~\text{for any}~f\in\cF^\dagger_{B,m},
\#
which implies $\hat{Q}(\cdot\, ;W^\dagger)=\Pi_{\cF^\dagger_{B,m}}\cT^{\pi}\hat{Q}(\cdot\, ;W^\dagger)$ by the definition of the projection induced by $\langle\cdot,\cdot\rangle_{\mu}$. By \eqref{5131118}, $\hat{Q}(\cdot\,;W^\dagger)$ is the global optimum of the MSPBE that corresponds to the projection onto $\cF^\dagger_{B,m}$. 

Intuitively, when using an overparametrized neural network with width $m \rightarrow \infty$, the average variation in each $W_r$ diminishes to zero. Hence, roughly speaking, we have $\ind\{W_r(t)^\top x>0\}=\ind\{W_r(0)^\top x>0\}$ with high probability for any $t\in [T]$. As a result, the function class $\cF^\dagger_{B,m}$ defined in \eqref{514022} approximates  
\#\label{initker}
\cF_{B,m}=\biggl\{\frac{1}{\sqrt{m}}\sum_{r=1}^m b_r\ind\{W_r(0)^\top x>0\} W_r^\top x: W\in S_B\biggr\}.
\#
In the sequel, we show that, to characterize the global convergence of Algorithm \ref{td0} with a sufficiently large $m$, it suffices to consider $\cF_{B,m}$ in place of $\cF^\dagger_{B,m}$, which simplifies the analysis, since the distribution of $W(0)$ is given. To this end, we define the approximate stationary point $W^*$ with respect to the function class $\cF_{B,m}$ defined in \eqref{initker}.
\begin{definition}[Approximate Stationary Point $W^*$]\label{def1}
If $W^*=(W^*_1,\ldots,W^*_m)\in\real^{md}$ satisfies
\# \label{spdef}
\EE_{\mu}[\delta_0(x, r, x';W^*)\cdot\nabla_{W}\hat{Q}_0(x;W^\ast)]^\top (W-W^*)\ge0,~~\text{for any}~W\in S_B,
\#
where we define
\#
&\hat{Q}_0(x;W)=\frac{1}{\sqrt{m}}\sum_{r=1}^m b_r\ind\{W_r(0)^\top x>0\} W_r^\top x,\label{q0}\\
&\delta_0(x, r, x';W)=\hat{Q}_0(x;W)-r-\gamma\hat{Q}_0(x';W),\label{d0}
\#
then we say that $W^*$ is an approximate stationary point of the population update in \eqref{mp}. Here $W^*$ depends on the random initialization $b = (b_1, \ldots, b_m)$ and $W(0)=(W_1(0),\ldots,W_m(0))$.
\end{definition}
The next lemma proves that such an approximate stationary point uniquely exists, since it is the fixed point of the operator $\Pi_{\cF_{B,m}}\cT^{\pi}$, which is a contraction in the $\ell_2$-norm associated with the stationary distribution $\mu$. 
\begin{lemma}[Existence, Uniqueness, and Optimality of $\hat{Q}_0(\cdot;W^*)$]\label{unique}
There exists an approximate stationary point $W^*$ for any $b \in \RR^{m}$ and $W(0)\in\real^{md}$. Also, $\hat{Q}_0(\cdot\,;W^*)$ is unique almost everywhere and is the global optimum of the MSPBE that corresponds to the projection onto $\cF_{B,m}$ in \eqref{initker}.
\end{lemma}
\begin{proof}
See Appendix \ref{uniquep} for a detailed proof.
\end{proof}

\subsection{Global Convergence}\label{5131051}
In this section, we establish the main results on the global convergence of neural TD in Algorithm \ref{td0}. We first lay out the following regularity condition on the stationary distribution $\mu$. 
\begin{assumption}[Regularity of Stationary Distribution $\mu$]\label{asmp1}
There exists a constant $c_0>0$ such that for any $\tau\ge0$ and $w\in\RR^d$ with $\|w\|_2=1$, it holds that
\$%\label{44145}
\PP\bigl(|w^\top x|\le \tau \bigr)\le  c_0\cdot \tau,
\$
where $x\sim\mu$.
\end{assumption}
Assumption \ref{asmp1} regularizes the density of $\mu$ in terms of the marginal distribution of $x$. In particular, it is straightforwardly implied when the marginal distribution of $x$ has a uniformly upper bounded probability density over the unit sphere. 

\vskip4pt
\noindent{\bf Population Update:} The next theorem establishes the nonasymptotic global rate of convergence of neural TD when it follows population semigradients. Recall that the approximate stationary point $W^*$ and the corresponding $\hat{Q}_0(\cdot \,;W^*)$ are defined in Definition \ref{def1}. Also, $B$ is the radius of the set of feasible $W$, which is defined in Algorithm \ref{td0}, $T$ is the number of iterations, $\gamma$ is the discount factor, and $m$ is the width of the neural network in \eqref{nnpara}. 
\begin{theorem}[Convergence of Population Update]\label{mainthm1}
We set $\eta=(1-\gamma)/8$ in Algorithm \ref{td0} and replace the TD update in Line \ref{algo_td} by the population update in \eqref{mp}. Under Assumption \ref{asmp1}, the output $\hat{Q}_{\text{out}}$ of Algorithm \ref{td0} satisfies 
\$
\EE_{\text{init}, \mu}\bigl[\bigl(\hat{Q}_{\text{out}}(x)-\hat{Q}_0(x;W^*)\bigr)^2\bigr]\le\frac{16B^2}{(1-\gamma)^2T}+O(B^3m^{-1/2}+B^{5/2}m^{-1/4}),
\$
where the expectation is taken with respect to all randomness, including the random initialization and the stationary distribution $\mu$.
\end{theorem}
\begin{proof}
The key to the proof of Theorem \ref{mainthm1} is the one-point monotonicity of the population semigradient $\overline{g}(t)$, which is established through the local linearization $\hat{Q}_0(x;W)$ of $\hat{Q}(x;W)$. See Appendix \ref{thm1p} for a detailed proof.
\end{proof}

\vskip4pt
\noindent{\bf Stochastic Update:} To further prove the global convergence of neural TD when it follows stochastic semigradients, we first establish an upper bound of their variance, which affects the choice of the stepsize $\eta$. For notational simplicity, we define the stochastic and population semigradients as
\#\label{514210}
g(t)=\delta\bigl(x, r, x';W(t)\bigr)\cdot\nabla_{W}\hat{Q}\bigl(x;W(t)\bigr), \quad \overline{g}(t)&=\EE_{\mu}[g(t)].
\#
\begin{lemma}[Variance Bound]\label{bdvar}
There exists $\sigma^2_g=O(B^2)$ such that the variance of the stochastic semigradient is upper bounded as $\EE_{\text{init}, \mu}[\|g(t)-\overline{g}(t)\|^2_2]\le\sigma^2_g$ for any $t\in[T]$.
\end{lemma}
\begin{proof}
See Appendix \ref{bdvara} for a detailed proof.
\end{proof}
Based on Theorem \ref{mainthm1} and Lemma \ref{bdvar}, we establish the global convergence of neural TD in Algorithm \ref{td0}.
\begin{theorem}[Convergence of Stochastic Update]\label{mainthm2}
We set $\eta=\min\{(1-\gamma)/8,1/\sqrt{T}\}$ in Algorithm \ref{td0}. Under Assumption \ref{asmp1}, the output $\hat{Q}_{\text{out}}$ of Algorithm \ref{td0} satisfies
\$
\EE_{\text{init},\mu}\bigl[
\bigl(\hat{Q}_{\text{out}}(x)-\hat{Q}_0(x;W^*)\bigr)^2\bigr]&\le \frac{16(B^2+\sigma^2_g)}{(1-\gamma)^2\sqrt{T}}
+O(B^3m^{-1/2}+B^{5/2}m^{-1/4}).
\$
\end{theorem}
\begin{proof}
See Appendix \ref{thm2p} for a detailed proof.
\end{proof}
As the width of the neural network $m \rightarrow \infty$, Lemma \ref{unique} implies that $\hat{Q}_0(\cdot\,;W^*)$ is the global optimum of the MSPBE in \eqref{5131118} with a richer function class $\cF_{B,\infty}$ to project onto. In fact, the function class $\cF_{B,\infty}-\hat{Q}(\cdot\,;W(0))$ is a subset of an RKHS with $\cH$-norm upper bounded by $B$. Here $\hat{Q}(\cdot\,;W(0))$ is defined in \eqref{nnpara}. See Appendix \ref{rkhsapprox} for a more detailed discussion on the representation power of $\cF_{B,\infty}$. Therefore, if the desired Q-function $Q^{\pi}(\cdot)$ falls into $\cF_{B,\infty}$, it is the global optimum of the MSPBE. By Lemma \ref{unique} and Theorem \ref{mainthm2}, we approximately obtain $Q^{\pi}(\cdot) = \hat{Q}_0(\cdot\,;W^*)$ through $\hat{Q}_\text{out}(\cdot)$.

More generally, the following proposition quantifies the distance between $\hat{Q}_0(\cdot\,;W^*)$ and $Q^\pi(\cdot)$ in the case that $Q^\pi(\cdot)$ does not fall into the function class $\cF_{B,m}$. In particular, it states that the $\ell_2$-norm distance $\|\hat{Q}_0(\cdot\,;W^*)-Q^\pi(\cdot)\|_{\mu}$ is upper bounded by the distance between $Q^\pi(\cdot)$ and $\cF_{B,m}$. 
\begin{proposition}[Convergence of Stochastic Update to $Q^\pi$]\label{57533}
It holds that $\|\hat{Q}_0(\cdot\,;W^*)-Q^\pi(\cdot)\|_{\mu}\le (1-\gamma)^{-1} \cdot\|\Pi_{\cF_{B,m}}Q^\pi(\cdot)-Q^\pi(\cdot)\|_{\mu}$, which by Theorem \ref{mainthm2} implies 
\$
\EE_{\text{init},\mu}\bigl[
\bigl(\hat{Q}_{\text{out}}(x)-Q^\pi(x)\bigr)^2\bigr]&\le \frac{32(B^2+\sigma^2_g)}{(1-\gamma)^2\sqrt{T}}+ \frac{2\EE_{\text{init},\mu}\bigl[
\bigl(\Pi_{\cF_{B,m}}Q^\pi(x)-Q^\pi(x)\bigr)^2\bigr]}{(1-\gamma)^2}\\
&\quad\qquad+O(B^3m^{-1/2}+B^{5/2}m^{-1/4}).
\$
\end{proposition}
\begin{proof}
See Appendix \ref{57533a} for a detailed proof.
\end{proof}
Proposition \ref{57533} implies that if $Q^\pi(\cdot) \in \cF_{B,\infty}$, then $\hat{Q}_{\text{out}}(\cdot) \rightarrow Q^\pi(\cdot)$ as $T, m\rightarrow \infty$. In other words, neural TD converges to the global optimum of the MSPBE in \eqref{5131118}, or equivalently, the MSBE in \eqref{513940}, both of which have objective value zero.

\section{Proof Sketch}\label{proofsketch}
In the sequel, we sketch the proofs of Theorems \ref{mainthm1} and \ref{mainthm2} in Section \ref{tr}. 

\subsection{Implicit Local Linearization via Overparametrization}\label{secwnn}
%\begin{flushleft}
Recall that as defined in \eqref{q0}, $\hat{Q}_0(x;W)$ takes the form
\$
&\hat{Q}_0(x;W)=\Phi(x)^\top W,\\
&\text{where}~\Phi(x)=\frac{1}{\sqrt{m}}\cdot \bigl(\ind\{W_1(0)^\top x>0\}x, \ldots, \ind\{W_m(0)^\top x>0\}x\bigr) \in \RR^{md},
\$
which is linear in the feature map $\Phi(x)$. In other words, with respect to $W$, $\hat{Q}_0(x;W)$ linearizes the neural network $\hat{Q}(x;W)$ defined in \eqref{nnpara} locally at $W(0)$. The following lemma characterizes the difference between $\hat{Q}(x;W(t))$, which is along the solution path of neural TD in Algorithm \ref{td0}, and its local linearization $\hat{Q}_0(x;W(t))$. In particular, we show that the error of such a local linearization diminishes to zero as $m\rightarrow \infty$. For notational simplicity, we use $\hat{Q}_t(x)$ to denote $\hat{Q}(x;W(t))$ in the sequel. Note that by \eqref{q0} we have $\hat{Q}_0(x) = \hat{Q}(x;W(0)) = \hat{Q}_0(x;W(0))$. Recall that $B$ is the radius of the set of feasible $W$ in \eqref{initker}.

\begin{lemma}[Local Linearization of Q-Function]\label{qdiff}
There exists a constant $c_1>0$ such that for any $t\in [T]$, it holds that
\$
\EE_{\text{init},\mu}\Big[\bigl|\hat{Q}_t(x)-\hat{Q}_0\bigl(x;W(t)\bigr)\bigr|^2\Bigr] \le
4c_1B^3\cdot m^{-1/2}.
\$
\end{lemma}
\begin{proof}
See Appendix \ref{qdiffa} for a detailed proof.
\end{proof}

As a direct consequence of Lemma \ref{qdiff}, the next lemma characterizes the effect of local linearization on population semigradients. Recall that $\overline{g}(t)$ is defined in \eqref{514210}. We denote by $\overline{g}_0(t)$ the locally linearized population semigradient, which is defined by replacing $\hat{Q}_t(x)$ in $\overline{g}(t)$ with its local linearization $\hat{Q}_0(x;W(t))$. In other words, by \eqref{514210}, \eqref{q0}, and \eqref{d0}, we have
\#
\overline{g}(t)&=\EE_{\mu}\bigl[\delta\bigl(x, r, x';W(t)\bigr)\cdot\nabla_{W}\hat{Q}\bigl(x;W(t)\bigr)\bigr],\label{518430}\\
\overline{g}_0(t)&=\EE_{\mu}\bigl[\delta_0\bigl(x, r, x';W(t)\bigr)\cdot\nabla_{W}\hat{Q}_0\bigl(x;W(t)\bigr)\bigr].\label{518431}
\#

\begin{lemma}[Local Linearization of Semigradient]\label{gdiff}
Let $\overline{r}$ be the upper bound of the reward $r(x)$ for any $x\in \cX$. There exists a constant $c_2>0$ such that for any $t\in [T]$, it holds that
\$
\EE_{\text{init}}\bigl[\|\overline{g}(t)-\overline{g}_0(t)\|^2_2\bigr]\le (56c_1B^3+24c_2B+6c_1B\overline{r}^2)\cdot m^{-1/2}.
\$
\end{lemma}
\begin{proof}
See Appendix \ref{gdiffa} for a detailed proof.
\end{proof}
Lemmas \ref{qdiff} and \ref{gdiff} show that the error of local linearization diminishes as the degree of overparametrization increases along $m$. As a result, we do not require the explicit local linearization in nonlinear TD \citep{bhatnagar2009convergent}. Instead, we show that such an implicit local linearization suffices to ensure the global convergence of neural TD. 

\subsection{Proofs for Population Update}\label{secddmp}
The characterization of the locally linearized Q-function in Lemma \ref{qdiff} and the locally linearized population semigradients in Lemma \ref{gdiff} allows us to establish the following descent lemma, which extends Lemma 3 of \cite{bhandari2018finite} for characterizing linear TD.  

\begin{lemma}[Population Descent Lemma] \label{dl}
For $\{W(t)\}_{t\in [T]}$ in Algorithm \ref{td0} with the TD update in Line \ref{algo_td} replaced by the population update in \eqref{mp}, it holds that 
\$
\|W(t+1)-W^*\|_2^2&\le \|W(t)-W^*\|_2^2-\bigl(2\eta(1-\gamma)-8\eta^2\bigr)\cdot \EE_\mu\Bigl[
\Bigl(\hat{Q}_0\bigl(x;W(t)\bigr)-\hat{Q}_0(x;W^*)\Bigr)^2\Bigr]\\
&\qquad +\underbrace{2\eta^2\cdot\|\overline{g}(t)-\overline{g}_0(t)\|^2_2+2\eta B\cdot\|\overline{g}(t)-\overline{g}_0(t)\|_2}_{\displaystyle \text{Error of Local Linearization}}.
\$
\end{lemma}
\begin{proof}
See Appendix \ref{dla} for a detailed proof.
\end{proof}
Lemma \ref{dl} shows that, with a sufficiently small stepsize $\eta$, $\|W(t)-W^*\|_2$ decays at each iteration up to the error of local linearization, which is characterized by Lemma \ref{gdiff}. By combining Lemmas \ref{gdiff} and \ref{dl} and further plugging them into a telescoping sum, we establish the convergence of $\hat{Q}_{\text{out}}(\cdot)$ to the global optimum $\hat{Q}_0(\cdot\,;W^*)$ of the MSPBE. See Appendix \ref{thm1p} for a detailed proof.

\subsection{Proofs for Stochastic Update}\label{secpsu}
Recall that the stochastic semigradient $g(t)$ is defined in \eqref{514210}. In parallel with Lemma \ref{dl}, the following lemma additionally characterizes the effect of the variance of $g(t)$, which is induced by the randomness of the current tuple $(x,r,x')$. We use the subscript $\EE_W[\cdot]$ to denote the expectation with respect to the randomness of the current iterate $W(t)$ conditional on the random initialization $b$ and $W(0)$. Correspondingly, $\EE_{W, \mu}[\cdot]$ is with respect to the randomness of both the current tuple $(x,r,x')$ and the current iterate $W(t)$ conditional on the random initialization.

%We use the subscript $\EE_{W, \mu}[\cdot]$ to denote the expectation over the randomness of both $(x,r,x')$ and the current iterate $W(t)$ conditional on the random initialization $b$ and $W(0)$.

%

\begin{lemma}[Stochastic Descent Lemma] \label{sdl}
For $\{W(t)\}_{t\in [T]}$ in Algorithm \ref{td0}, it holds that
\$
&\EE_{W, \mu}\bigl[\|W(t+1)-W^*\|_2^2\bigr]\\
&\quad\le \EE_{W}\bigl[\|W(t)-W^*\|_2^2\bigr]-\bigl(2\eta(1-\gamma)-8\eta^2\bigr)\cdot\EE_{W, \mu}\Bigl[
\Bigl(\hat{Q}_0\bigl(x;W(t)\bigr)-\hat{Q}_0(x;W^*)\Bigr)^2\Bigr]\\
&\quad\qquad +\underbrace{\EE_{W}\bigl[2\eta^2\cdot\|\overline{g}(t)-\overline{g}_0(t)\|^2_2+2\eta B\cdot\|\overline{g}(t)-\overline{g}_0(t)\|_2 \bigr]}_{\displaystyle \text{Error of Local Linearization}}+\underbrace{\EE_{W, \mu}\bigl[\eta^2\cdot\|g(t)-\overline{g}(t)\|^2_2\bigr]}_{\displaystyle \text{Variance of Semigradient}}.
\$
\end{lemma}
\begin{proof}
See Appendix \ref{sdla} for a detailed proof.
\end{proof}

To ensure the global convergence of neural TD in the presence of the variance of $g(t)$, we rescale the stepsize to be of order $T^{-1/2}$. The rest proof of Theorem \ref{mainthm2} mirrors that of Theorem \ref{mainthm1}. See Appendix \ref{thm2p} for a detailed proof. 
%\end{flushleft}

\section{Extension to Policy Optimization}\label{qlearning}

With the Q-function learned by TD, policy iteration may be applied to learn the optimal policy. Alternatively, Q-learning more directly learns the optimal policy and its Q-function using temporal-difference update. Compared with TD, Q-learning aims to solve the projected Bellman optimality equation
\#\label{opbmeqn}
Q=\Pi_{\cF}\cT Q,~~\text{with}~~
\cT Q(s,a)=\EE \big[ r(s,a)+\gamma\max_{a'\in\cA}Q(s',a') \,\big|\, s'\sim\cP(\cdot\,|\,s,a) \bigr],
\#
which replaces the Bellman evaluation operator $\cT^{\pi}$ in \eqref{5131118} with the Bellman optimality operator $\cT$. When $\Pi_{\cF}$ is identity, the fixed-point solution to \eqref{opbmeqn} is the Q-function $Q^{\pi^*}(s,a)$ of the optimal policy $\pi^*$, which maximizes the expected total reward \citep{szepesvari2010algorithms, sutton2018reinforcement}. Compared with TD, the max operator in $\cT$ makes the analysis more challenging and hence requires stronger regularity conditions. In the following, we first introduce neural Q-learning and then establish its global convergence. Finally, we discuss the corresponding implication for policy gradient algorithms. Throughout Section \ref{qlearning}, we focus on two-layer neural networks. Our analysis can be extended to handle multi-layer neural networks using the proof techniques in  Appendix \ref{deep}.

\subsection{Neural Q-Learning}\label{qlearningalgo}
In parallel with \eqref{5131033}, we update the parameter $\theta$ of the optimal Q-function by
\# \label{430421}
\theta'\leftarrow\theta-\eta\cdot \big( \hat{Q}_{\theta}(s,a)-r(s,a)-\gamma\max_{a'\in\cA}\hat{Q}_{\theta}(s',a') \bigr)\cdot \nabla_{\theta}\hat{Q}_{\theta}(s,a),
\#
where the tuple $(s,a,r,s')$ is sampled from the stationary distribution $\mue$ of an exploration policy $\pie$ in an independent and identically distributed manner. Our analysis can be extended to handle temporal dependence using the proof techniques in Appendix \ref{secmarkov}. We present the detailed neural Q-learning algorithm in Algorithm \ref{tdqlearning}. Similar to Definition \ref{def1}, we define the approximate stationary point $W^*$ of Algorithm \ref{tdqlearning} by
\#\label{4251216}
\EE_{\mue}[\delta_0(x, r, x'; W^*)\cdot\nabla_{W}\hat{Q}_0(x;W^\ast)]^\top (W-W^*)\ge0,~~\text{for any}~W\in S_B,
\#
where the Bellman residual is now $\delta_0(x, r, x';W)=\hat{Q}_0(x;W)-r-\gamma\max_{a'\in\cA}\hat{Q}_0(s',a';W)$. Following the same analysis of neural TD in Lemma \ref{unique}, we have that $\hat{Q}_0(\cdot\,;W^*)$ is the unique fixed-point solution to the projected Bellman optimality equation $Q=\Pi_{\cF_{B,m}}\cT Q$, where the function class $\cF_{B,m}$ is define in \eqref{initker}.

\begin{algorithm}
\caption{Neural Q-Learning}
\begin{algorithmic}[1]
\STATE \textbf{Initialization:} $b_r\sim\text{Unif}(\{-1,1\})$, $W_r(0)\sim N(0,I_d/d)$ $(r\in[m])$, $\overline{W}=W(0)$,
 \label{Istep2}\\
{\color{white}\textbf{Initialization:}} $S_B=\{W\in\real^{md}:\|W-W(0)\|_2\le B\}$ $(B>0)$,\\
{\color{white}\textbf{Initialization:}} exploration policy $\pie$ such that $\pie(a\,|\,s)>0$ for any $(s,a)\in\cS\times\cA$ \\
\STATE \textbf{For} {$t=0$ to $T-2$}:
\STATE \hspace{0.15in} Sample a tuple $(s, a, r, s')$ from the stationary distribution $\mue$ of the exploration policy $\pie$
\STATE \hspace{0.15in} Let $x=(s,a)$, $x'=(s',\argmax_{a'\in\cA}\hat{Q}(s',a';W(t)))$ \label{greedy-a}
\STATE \hspace{0.15in} Bellman residual calculation: $\delta \leftarrow \hat{Q}(x;W(t))-r-\gamma\hat{Q}(x';W(t))\label{qltdstep}$
\STATE \hspace{0.15in} TD update: $\tilde{W}(t+1)\leftarrow W(t)-\eta\delta\cdot\nabla_{W}\hat{Q}(x;W(t))$  \label{algo_td}\\
\STATE \hspace{0.15in} Projection: $W(t+1)\leftarrow \argmin_{W\in S_B}\|W-\tilde{W}(t+1)\|_2$
\STATE \hspace{0.15in} Averaging: $\overline{W}\leftarrow \frac{t+1}{t+2}\cdot \overline{W}+\frac{1}{t+2}\cdot W(t+1)$
\STATE \textbf{End For} 
\STATE \textbf{Output:}  $\hat{Q}_\text{out}(\cdot)\leftarrow \hat{Q}(\cdot\,;\overline{W})$
\end{algorithmic}\label{tdqlearning}
\end{algorithm} 

\subsection{Global Convergence}\label{qlearningconvergence}
To establish the global convergence of neural Q-learning, we lay out an extra regularity condition on the exploration policy $\pie$, which is not required by neural TD. Such a regularity condition ensures that $x'=(s',a')$ with the greedy action $a'$ in Line \ref{greedy-a} of Algorithm \ref{tdqlearning} follows a similar distribution to that of $x=(s,a)$, which is the stationary distribution $\mue$ of the exploration policy $\pie$. Recall that $\hat{Q}_0(x;W)$ is defined in \eqref{q0} and $\gamma$ is the discount factor.
\begin{assumption}[Regularity of Exploration Policy $\pie$]\label{cond1} There exists a constant $\nu>0$ such that for any $W_1,W_2\in S_B$, it holds that
\#\label{430312}
\EE_{x\sim\mue}\bigl[\bigl(\hat{Q}_0(x;W_1)-\hat{Q}_0(x;W_2)\bigr)^2\bigr]\ge(\gamma+\nu)^2\cdot\EE_{s\sim\mue}\bigl[\bigl(\hat{Q}_0^{\sharp}(s;W_1)-\hat{Q}_0^{\sharp}(s;W_2)\bigr)^2\bigr],
\#
where $\hat{Q}_{0}^{\sharp}(s;W)=\max_{a\in\cA}\hat{Q}_0(s,a;W)$.
\end{assumption}
We remark that \cite{melo2008analysis, zou2019finite} establish the global convergence of linear Q-learning based on an assumption that implies \eqref{430312}. Although Assumption \ref{cond1} is strong, we are not aware of any weaker regularity condition in the literature, even for linear Q-learning. As our focus is to go beyond linear Q-learning to analyze neural Q-learning, we do not attempt to weaken such a regularity condition in this paper. 

%Note that as long as the exploration policy $\pie$ is close to both $\epsilon$-greedy policies with regard to $\hat{Q}(\cdot\,;W_1)$ and $\hat{Q}(\cdot\,;W_2)$, we have $\hat{Q}_0(x)\approx\hat{Q}^{\sharp}_0(s)$ with high probability, and thus $\nu$ exists since $|\gamma|<1$. 

The following regularity condition on $\mue$ mirrors Assumption \ref{asmp1}, but additionally accounts for the max operator in the Bellman optimality operator. 
\begin{assumption}[Regularity of Stationary Distribution $\mue$]\label{asmp2}
There exists a constant $c_3>0$ such that for any $\tau\ge0$ and $w\in\RR^d$ with $\|w\|_2=1$, it holds that
\$%\label{44145}
\PP\big( |w^\top \psi(s,a)|\le  \tau, \,\text{for all}\, a\in\cA \bigr) \le  c_3\cdot \tau,
\$
where $(s,a)\sim \mue$.
\end{assumption} 

In parallel with Theorem \ref{mainthm2}, the following theorem establishes the global convergence of neural Q-learning in Algorithm \ref{tdqlearning}.

\begin{theorem}[Convergence of Stochastic Update]\label{thm3}
We set $\eta$ to be of order $T^{-1/2}$ in Algorithm \ref{tdqlearning}. Under Assumptions \ref{cond1} and \ref{asmp2}, the output $\hat{Q}_{\text{out}}$ of Algorithm \ref{tdqlearning} satisfies
\$
\EE_{\text{init},\mue}\bigl[ \bigl( \hat{Q}_{\text{out}}(x)- \hat{Q}_0(x;W^*) \bigr)^2 \bigr]= O(B^2T^{-1/2}+B^3m^{-1/2}+B^{5/2}m^{-1/4}).
\$
\end{theorem}
\begin{proof}
See Appendix \ref{proof:thm3} for a detailed proof.
\end{proof}
Corresponding to Proposition \ref{57533}, Theorem \ref{thm3} also implies the convergence to $Q^{\pi^*}(s,a)$, which is omitted due to space limitations.

\subsection{Implication for Policy Gradient}\label{qlearningconvergence}
Theorem \ref{thm3} can be further extended to handle neural soft Q-learning, where the max operator in the Bellman optimality operator is replaced by a more general softmax operator \citep{haarnoja2017reinforcement, neu2017unified}. By exploiting the equivalence between soft Q-learning and policy gradient algorithms \citep{schulman2017equivalence, haarnoja2018soft}, we establish the global convergence of a variant of the policy gradient algorithm. Due to space limitations, we defer the discussion to Appendix \ref{secfqs}, throughout which we focus on two-layer neural networks. Our analysis can be extended to handle multi-layer neural networks using the proof techniques in Appendix \ref{deep}.

%\end{flushleft}

\section{Conclusions}
In this paper we prove that neural TD converges at a sublinear rate to the global optimum of the MSPBE for policy evaluation. In particular, we show how such global convergence is enabled by the overparametrization of neural networks. Moreover, we extend the convergence result to policy optimization, including (soft) Q-learning and policy gradient. Our results shed new light on the theoretical understanding of RL with neural networks, which is widely employed in practice.

\bibliographystyle{ims}
\bibliography{graphbib}

\newpage
\begin{appendix}	
%\begin{flushleft}

\section{Representation Power of $\cF_{B,m}$}
\subsection{Background on RKHS}\label{bcg}
We consider the following kernel function
\#\label{515150}
K(x,y)=\int_{\cW} \phi(x;w)\phi(y;w)p(w)dw.
\#
Here $\phi$ is a random feature map parametrized by $w$, which follows a distribution with density $p(\cdot)$ \citep{rahimi2008random}. Any function in the RKHS induced by $K(\cdot,\cdot)$ takes the form
\#\label{423800}
f_c(x)=\int_{\cW} c(w)\phi(x;w)p(w)dw,
\#
such that each $c(\cdot)$ corresponds to a function $f_c(\cdot)$. The following lemma connects the $\cH$-norm of $f_c(\cdot)$ to the $\ell_2$-norm of $c(\cdot)$ associated with the density $p(\cdot)$, denoted by $\|c\|_p$.
\begin{lemma}\label{57131} It holds that
$\|f_c\|_{\cH}^2=\|c\|_p^2=\int c(w)^2p(w)dw$.
\end{lemma}
\begin{proof}
Recall if $f(x)=\int_{\cX} a(y)K(x,y)dy$, then by the reproducing property \citep{hofmann2008kernel}, we have 
\$
\|f\|_{\cH}^2=\int_{\cX\times\cX} a(x)a(y)K(x,y)dxdy.
\$
Now we write $f(\cdot)$ in the form of \eqref{423800}. By \eqref{515150}, we have
\$
f(x)&=\int_{\cX} a(y)K(x,y)dy\\
&=\int_{\cX} a(y)\int_{\cW} \phi(x;w)\phi(y;w)p(w)dwdy\\
&=\int_{\cW} \underbrace{\Bigl( \int_{\cX} a(y)\phi(y;w)dy \Bigr)}_{\textstyle c(w)} \phi(x;w)p(w)dw.
\$
Thus, for $c(w)=\int_{\cX} a(y)\phi(y;w)dy$, we have
\$
\|f\|_{\cH}^2&=\int_{\cX\times\cX} a(y)a(x)K(x,y)dxdy\\
&=\int_{\cX\times\cX} a(y)a(x)\Bigl(\int_{\cW} \phi(x;w)\phi(y;w)p(w)dw\Bigr)dxdy\\
&=\int_{\cW} \Bigl(\int_{\cX} a(y)\phi(y;w)dy\Bigr)\Bigl( \int_{\cX} a(x)\phi(x;w)dx\Bigr)p(w)dw\\
&=\int_{\cW} c(w)^2p(w)dw = \|c\|_p^2,
\$
which completes the proof of Lemma \ref{57131}.
\end{proof}

\subsection{$\cF_{B,\infty}$ as RKHS}\label{rkhsapprox}
We characterize the approximate stationary point $W^*$ and the corresponding $\hat{Q}_0(x;W^*)$ defined in Definition \ref{def1}, which are attained by Algorithm \ref{td0} according to Theorems \ref{mainthm1} and \ref{mainthm2}. We focus on its representation power when $m\rightarrow \infty$. We first write $\cF_{B,m}$ in \eqref{initker} as
\#\label{516300}
\cF_{B,m}=\biggl\{f(x)=\hat{Q}\bigl(x;W(0)\bigr)+\sum_{r=1}^m \phi_r(x)^\top \bigl(W_r-W_r(0)\bigr): W\in S_B\biggr\},
\#
where the feature map $\{\phi_r(x)\}_{r\in [m]}$ is defined as 
\$
\phi_r(x)=\frac{1}{\sqrt{m}}\cdot\phi\bigl(x;W_r(0)\bigr)=\frac{1}{\sqrt{m}}\cdot \ind\{W_r(0)^\top x>0\}x ~~\text{for any}~r\in [m].
\$ 
As $m\rightarrow \infty$, the empirical distribution supported on $\{\phi_r(x)\}_{r\in [m]}$, which has sample size $m$, converges to the corresponding population distribution. 
%Note that $\phi(x;W(0))$ only depends on the direction of $W(0)$ and recall that $W_r(0) \sim N(0,I_d/d) \ (r\in [m])$. Hence, $\{\phi_r(x)\}_{r=1}^m$ can be equivalently generated with $W_r(0) \ (r\in [m])$ drawn uniformly at random from the unit sphere in $\RR^d$. 
Therefore, from \eqref{516300} we obtain 
\$
\cF_{B,\infty}=\biggl\{f(x)=f_0(x)+\int \phi(x;w)^\top \alpha(w)\cdot p(w)dw:\int \|\alpha(w)\|_2^2\cdot p(w)dw\le B^2\biggr\}.
\$
%the uniform distribution on the unit sphere
Here $p(w)$ is the density of $N(0,I_d/d)$ and $f_0(x) = \lim_{m\rightarrow \infty}\hat{Q}(x;W(0))$, which by the central limit theorem is a Gaussian process indexed by $x$. Furthermore, as discussed in Appendix \ref{bcg}, $\phi(x;W)$ induces an RKHS, namely $\cH$, which is the completion of the set of all functions that take the form
\$
&f(x)=\sum_{i=1}^N a_i K(x,x_i),~~x_i\in\cX,~a_i\in\real,~N\in\mathbb{N},\\
&\text{where }K(x,y)=\EE_{w\sim N(0,I_d/d)}\bigl[\ind\{w^\top x>0,w^\top y>0\}x^\top y\bigr].
\$
In particular, $\cH$ is equipped with the inner product induced by $\langle K(\cdot,x_i),K(\cdot,x_j) \rangle_{\cH}=K(x_i,x_j)$. \cite{rahimi2008uniform} prove that, similar to Lemma \ref{57131}, for any $f_1(\cdot)=\int \phi(\cdot\,;w)^\top \alpha_1(w)\cdot p(w)dw$ and $f_2(\cdot)=\int \phi(\cdot\,;w)^\top \alpha_2(w)\cdot p(w)dw$, we have $f_1,f_2\in\cH$, and moreover, their inner product has the following equivalence 
\$
\langle f_1,f_2 \rangle_{\cH}=\int \alpha_1(w)^\top\alpha_2(w)\cdot p(w)dw.
\$
As a result, we have 
\$ %\label{fbprime}
\cF_{B,\infty}=\bigl\{f=f_0+h: \|h\|_{\cH}\le B\bigr\},
\$
%Moreover, $\cF_{B, \infty}$ becomes a dense subset of $\cH$ as $B\rightarrow \infty$. Also, the projection onto $\cF_{B,\infty}$ is equivalent to the projection onto its completion
%\# \label{fbprime}
%\cF'_{B,\infty}=\{f=f_0+h: \|h\|_{\cH}\le B \}.
%\#
which is known to be a rich function class \citep{hofmann2008kernel}. As $m\rightarrow\infty$, $\hat{Q}_0(\cdot\,;W^*)$ becomes the fixed-point solution to the projected Bellman equation
\$
Q=\Pi_{\cF_{B,\infty}}\cT^{\pi}Q,
\$
which also implies that $\hat{Q}_0(\cdot\,;W^*)$ is the global optimum of the MSPBE 
\$
\EE_{\mu}\bigl[\bigl(Q(x)-\Pi_{\cF_{B,\infty}}\cT^{\pi} Q(x)\bigr)^2\bigr].
\$ 
If we further assume that the Bellman evaluation operator $\cT^{\pi}$ satisfies $\cT^{\pi}\hat{Q}_0(\cdot\,;W^*)-f_0(\cdot)\in\cH$ and $B$ is sufficiently large such that $\|\cT^{\pi}\hat{Q}_0(\cdot\,;W^*)-f_0(\cdot)\|_{\cH}\le B$, then the projection $\Pi_{\cF_{B,\infty}}$ reduces to identity at $\cT^{\pi}\hat{Q}_0(\cdot\,;W^*)$, which implies $\hat{Q}_0(\cdot\,;W^*)=Q^{\pi}(\cdot)$ as they both solve the Bellman equation $Q = \cT^{\pi}Q$. In other words, if the Bellman evaluation operator is closed with respect to $\cF_{B, \infty}$, which up to the intercept of $f_0(\cdot)$ is a ball with radius $B$ in $\cH$, $\hat{Q}_0(\cdot\,;W^*)$ is the unique fixed-point solution to the Bellman equation or equivalently the global optimum of the MSBE
\$
\EE_{\mu}\bigl[\bigl(Q(x)-\cT^{\pi} Q(x)\bigr)^2\bigr].
\$

%{\red current}
\section{Proofs for Section \ref{tr}}
%Since the main results in Section \ref{tr} are further proved in Section \ref{proofsketch} and Appendix \ref{proofsketcha}, in this section we only present the proofs of several lemmas.

\subsection{Proof of Lemma \ref{unique}}\label{uniquep}
\begin{proof}
Following the same argument for $W^\dagger$ in \eqref{56202} and the definition of $W^*$ in \eqref{spdef}, we know that $\hat{Q}_0(\cdot\,;W^*)$ is a fixed-point solution to the projected Bellman equation
\#\label{410124}
Q=\Pi_{\cF_{B,m}}\cT^{\pi}Q.
\#
Meanwhile, the Bellman evaluation operator $\cT^{\pi}$ is a $\gamma$-contraction in the $\ell_2$-norm $\|\cdot\|_{\mu}$ with $\gamma<1$, since
\begin{align*}
\EE_{x\sim\mu}\bigl[ \bigl(\cT^{\pi}Q_1(x)-\cT^{\pi}Q_2(x)\bigr)^2 \bigr]&=\gamma^2\EE_{x\sim\mu}\bigl[ \bigl(\EE[Q_1(x')-Q_2(x') \,|\,s'\sim\cP(\cdot\,|\,s,a), a'\sim \pi(s')]\bigr)^2 \bigr]\\&\leq \gamma^2\EE_{x\sim\mu}\bigl[ \bigl(Q_1(x)-Q_2(x)\bigr)^2 \bigr],
\end{align*}
where the second equality follows from H\"{o}lder's inequality and the fact that marginally $x'$ and $x$ have the same stationary distribution. Since the projection onto a convex set is nonexpansive, $\Pi_{\cF_{B,m}}\cT^{\pi}$ is also a $\gamma$-contraction. Thus, the projected Bellman equation in \eqref{410124} has a unique fixed-point solution $\hat{Q}_0(\cdot\,;W^*)$ in $\cF_{B,m}$, which corresponds to the approximate stationary point $W^*$.
\end{proof}

\subsection{Proof of Lemma \ref{bdvar}}\label{bdvara}

\begin{proof}
It suffices to show that $\EE_{\text{init},\mu}[\|g(t)\|_2^2]$ is upper bounded. By \eqref{514210}, we have
\#\label{5201008}
\EE_{\text{init},\mu}\bigl[\|g(t)\|_2^2\bigr]=\EE_{\text{init},\mu}\Bigl[\bigl\|\delta\bigl(x, r, x';W(t)\bigr)\cdot \nabla_W\hat{Q}_t(x)\bigr\|_2^2\Bigr]\le\EE_{\text{init},\mu}\Bigl[\bigl|\delta\bigl(x, r, x';W(t)\bigr)\bigr|^2\Bigr],
\#
where the inequality follows from the fact that, for any $W\in S_B$,
\#\label{520938}
\|\nabla_W\hat{Q}(x;W)\|_2= \frac{1}{m} \sum_{r=1}^m \ind\{W^\top x>0\}\|x\|_2^2 \le 1
\#
almost everywhere. Using the fact that $x$ and $x'$ have the same marginal distribution we obtain
\#\label{520940}
\EE_{\text{init},\mu}\Bigl[ \bigl|\delta\bigl(x, r, x';W(t)\bigr)\bigr|^2\Bigr]\le \EE_{\text{init},\mu}\bigl[
3\bigl(\hat{Q}_t(x)^2+\overline{r}^2+\hat{Q}_t(x')^2\bigr)
 \bigr]=  \EE_{\text{init},\mu}[
6\hat{Q}_t(x)^2+3\overline{r}^2].
\#
By \eqref{520938}, we know that $\hat{Q}(x;W)$ is $1$-Lipschitz continuous with respect to $W$. Therefore, we have
\#\label{520946}
|\hat{Q}_t(x)-\hat{Q}_0(x)|\le \|W(t)-W(0)\|_2\le B,
\#
Plugging \eqref{520946} into \eqref{520940} and using the Cauchy-Schwarz inequality we obtain
\#\label{5201005}
\EE_{\text{init},\mu}\Bigl[ \bigl|\delta\bigl(x, r, x';W(t)\bigr)\bigr|^2\Bigr]
\le\EE_{\text{init},\mu}[
12\hat{Q}_0(x)^2+12B^2+3\overline{r}^2].
\#
Note that by the initialization of $\hat{Q}_0(x)$ as defined in \eqref{nnpara}, we have
\#\label{5201006}
\EE_{\text{init},\mu}[\hat{Q}_0(x)^2]=\frac{1}{m}\sum_{r=1}^m \EE_{\text{init}}\bigl[\sigma\bigl({W_r(0)^\top x}\bigr)^2]\le 
\EE_{w\sim N(0,I_d/d)}\bigl[\|w\|^2_2\bigr]=1.
\#
Combining \eqref{5201008}, \eqref{5201005}, and \eqref{5201006} we obtain $\EE_{\text{init},\mu}[\|g(t)\|_2^2]=O(B^2)$.
Since 
\$
\EE_{\text{init},\mu}\bigl[ \|g(t)-\overline{g}(t)\|_2^2\bigr]&=\EE_{\text{init}}\Bigl[\EE_{\mu}\bigl[ \|g(t)-\overline{g}(t)\|_2^2\bigr]\Bigr] \notag \\
&\le \EE_{\text{init}}\Bigl[\EE_{\mu}\bigl[\|g(t)\|_2^2\bigr]\Bigr]=\EE_{\text{init},\mu}\bigl[ \|g(t)\|_2^2\bigr],
\$
 we conclude the proof of Lemma \ref{bdvar}.
\end{proof}

\subsection{Proof of Proposition \ref{57533}}\label{57533a}
\begin{proof}
By the triangle inequality, we have
\# \label{55550}
\|\hat{Q}_0(\cdot\,;W^*)-Q^\pi(\cdot)\|_{\mu}\le \|\hat{Q}_0(\cdot\,;W^*)-\Pi_{\cF_{B,m}}Q^\pi(\cdot)\|_{\mu}+\|\Pi_{\cF_{B,m}}Q^\pi(\cdot)-Q^\pi(\cdot)\|_{\mu}.
\#
Since $Q^\pi(\cdot)$ is the fixed-point solution to the Bellman equation, we replace $Q^\pi(\cdot)$ by $\cT^\pi Q^\pi(\cdot)$ and obtain 
\#\label{55546}\Pi_{\cF_{B,m}}Q^\pi(\cdot)=\Pi_{\cF_{B,m}}\cT^\pi Q^\pi(\cdot).\#
Meanwhile, by Lemma \ref{unique}, $\hat{Q}_0(\cdot\,;W^*)$ is the solution to the projected Bellman equation, that is,
\#\label{55547}\hat{Q}_0(\cdot\,;W^*)=\Pi_{\cF_{B,m}}\cT^\pi \hat{Q}_0(\cdot\,;W^*).\#
Combining \eqref{55546} and \eqref{55547}, we obtain
\# \label{55551}
\|\hat{Q}_0(\cdot\,;W^*)-\Pi_{\cF_{B,m}}Q^\pi(\cdot)\|_{\mu}&=\|\Pi_{\cF_{B,m}}\cT^\pi \hat{Q}_0(\cdot\,;W^*)-\Pi_{\cF_{B,m}}\cT^\pi Q^\pi(\cdot)\|_{\mu}\notag\\
&\le  \gamma\cdot\|\hat{Q}_0(\cdot\,;W^*)-Q^\pi(\cdot)\|_{\mu},
\#
where the inequality follows from the fact that $\Pi_{\cF_{B,m}}\cT^\pi$ is a $\gamma$-contraction, as discussed in the proof of Lemma \ref{unique}. Plugging \eqref{55551} into \eqref{55550}, we obtain
\$
(1-\gamma)\cdot \|\hat{Q}_0(\cdot\,;W^*)-Q^\pi(\cdot)\|_{\mu}\le\|\Pi_{\cF_{B,m}}Q^\pi(\cdot)-Q^\pi(\cdot)\|_{\mu},
\$
which completes the proof of Proposition \ref{57533}.
\end{proof}

\section{Proofs for Section \ref{proofsketch}}\label{proofsketcha}

\subsection{Proof of Lemma \ref{qdiff}}\label{qdiffa}
\begin{proof}
By the definition that $\hat{Q}_t(x) = \hat{Q}(x;W(t))$ and the definition of $\hat{Q}_0(x;W(t))$ in \eqref{q0}, we have
\#\label{43143}
&\bigl|\hat{Q}_t(x)-\hat{Q}_0\bigl(x;W(t)\bigr)\bigr|\\
&\quad=\frac{1}{\sqrt{m}}\Bigl|\sum_{r=1}^m\bigl(\ind\{W_r(t)^\top x>0\}-\ind\{W_r(0)^\top x>0\}\bigr)\cdot b_r W_r(t)^\top x\Bigr|\notag\\
&\quad\le \frac{1}{\sqrt{m}}\sum_{r=1}^m |\ind\{W_r(t)^\top x>0\}-\ind\{W_r(0)^\top x>0\}|\cdot \bigl(|W_r(0)^\top x|+\|W_r(t)-W_r(0)\|_2\bigr), \notag
\#
where we use the fact that $\|x\|_2 = 1$. Note that $\ind\{W_r(t)^\top x>0\}\neq \ind\{W_r(0)^\top x>0\}$ implies
\$
|W_r(0)^\top x|\le |W_r(t)^\top x-W_r(0)^\top x|\le \|W_r(t)-W_r(0)\|_2.
\$
Thus, we obtain
\#\label{43142}
|\ind\{W_r(t)^\top x>0\}-\ind\{W_r(0)^\top x>0\}|\le \ind\{|W_r(0)^\top x|\le \|W_r(t)-W_r(0)\|_2 \}.
\#
Plugging \eqref{43142} into \eqref{43143}, we obtain the following upper bound,
\$
&\bigl|\hat{Q}_t(x)-\hat{Q}_0\bigl(x;W(t)\bigr)\bigr|\\
&\quad\le 
\frac{1}{\sqrt{m}}\sum_{r=1}^m \ind\{|W_r(0)^\top x|\le \|W_r(t)-W_r(0)\|_2 \}\cdot \bigl(|W_r(0)^\top x|+ \|W_r(t)-W_r(0)\|_2\bigr) \\
&\quad\le \frac{2}{\sqrt{m}}\sum_{r=1}^m \ind\{|W_r(0)^\top x|\le  \|W_r(t)-W_r(0)\|_2\} \cdot\|W_r(t)-W_r(0)\|_2.
\$
Here the second inequality follows from the fact that
\$
\ind\{|x|\le y\}|x|\le \ind\{|x|\le y\}y
\$ for any $x$ and $y>0$. To characterize $\EE_{\text{init},\mu}[|\hat{Q}_t(x)-\hat{Q}_0(x;W(t))|^2]$, we first invoke the Cauchy-Schwarz inequality and the fact that $\|W(t)-W(0)\|_2\le B$, which gives 
\$
\bigl|\hat{Q}_t(x)-\hat{Q}_0\bigl(x;W(t)\bigr)\bigr|^2&\le 
\frac{4B^2}{m} \sum_{r=1}^m \ind\{|W_r(0)^\top x|\le  \|W_r(t)-W_r(0)\|_2\}.
\$
Taking expectation on both sides, by Lemma \ref{mulm1} we obtain
\$
\EE_{\text{init},\mu}\Bigl[\bigl|\hat{Q}_t(x)-\hat{Q}_0\bigl(x;W(t)\bigr)\bigr|^2\Bigr] \le
4c_1B^3\cdot m^{-1/2}.
\$
Thus, we finish the proof of Lemma \ref{qdiff}.
\end{proof}

\subsection{Proof of Lemma \ref{gdiff}}\label{gdiffa}
\begin{proof}
By the definition of $\overline{g}(t)$ and $\overline{g}_0(t)$ in \eqref{518430} and \eqref{518431}, respectively, we have
\#\label{518425}
\|\overline{g}(t)-\overline{g}_0(t)\|_2&=\bigl\|\EE_{\mu}\bigl[
\delta\bigl(x, r, x';W(t)\bigr)\cdot\nabla_W\hat{Q}_t(x)-\delta_0\bigl(x, r, x';W(t)\bigr)\cdot\nabla_W\hat{Q}_0\bigl(x;W(t)\bigr)
\bigr]\bigr\|_2\notag\\
&\le \Bigl\|
\EE_{\mu}\Bigl[
\Bigl(\delta\bigl(x, r, x';W(t)\bigr)-\delta_0\bigl(x, r, x';W(t)\bigr)\Bigr)\cdot \nabla_W\hat{Q}_t(x)\notag\\
&\qquad+
\delta_0\bigl(x, r, x';W(t)\bigr)\cdot\Bigl(\nabla_W\hat{Q}_t(x)- \nabla_W\hat{Q}_0\bigl(x;W(t)\bigr)\Bigr)
\Bigr]
\Bigr\|_2\notag\\
&\le \EE_{\mu}\Bigl[
\bigl|\delta\bigl(x, r, x';W(t)\bigr)-\delta_0\bigl(x, r, x';W(t)\bigr)\bigr|\\
&\qquad
+\bigl|\delta_0\bigl(x, r, x';W(t)\bigr)\bigr|\cdot\bigl\|\nabla_W\hat{Q}_t(x)-\nabla_W\hat{Q}_0\bigl(x;W(t)\bigr)\bigr\|_2
\Bigr].\notag
\#
Here to obtain the second inequality, we use the fact that, for any $t\in [T]$,
\$
\|\nabla_W\hat{Q}_t(x)\|_2 \leq \| x \|_2 = 1.
\$
Taking expectation with respect to the random initialization on the both sides of \eqref{518425}, we obtain
\#\label{57200}
&\EE_{\text{init}}\bigl[ \|\overline{g}(t)-\overline{g}_0(t)\|_2^2\bigr]  \notag\\
&\quad\le \underbrace{2\EE_{\text{init},\mu} \Bigl[ \bigl| \delta\bigl(x, r, x';W(t)\bigr)-\delta_0\bigl(x, r, x';W(t)\bigr)\bigr|^2 \Bigr]}_{\displaystyle\text{(i)}}\\
&\quad\qquad
+2\EE_{\text{init}}\biggl[\underbrace{\EE_{\mu} \Bigl[
\bigl|\delta_0\bigl(x, r, x';W(t)\bigr)\bigr|^2\Bigr]}_{\displaystyle\text{(iii)}}\cdot
\underbrace{\EE_{\mu} \Bigl[
\bigl\|\nabla_W\hat{Q}_t(x)-\nabla_W\hat{Q}_0\bigl(x;W(t)\bigr)\bigr\|_2^2
 \Bigr]}_{\displaystyle\text{(ii)}}\biggr].\notag
\#
In the following, we characterize the three terms on the right-hand side of \eqref{57200}. 

For (i) in \eqref{57200}, note that 
\#\label{518450}
&\bigl|\delta\bigl(x, r, x';W(t)\bigr)-\delta_0\bigl(x, r, x';W(t)\bigr)\bigr|^2\notag\\
&\quad = \Bigl| \bigl(\hat{Q}_t(x)-r-\gamma\hat{Q}_t(x')\bigr)-\Bigl(\hat{Q}_0\bigl(x;W(t)\bigr)-r-\gamma\hat{Q}_0\bigl(x';W(t)\bigr)\Bigr)\Bigr|^2\notag\\
&\quad=\Bigl|\Bigl(\hat{Q}_t(x)-\hat{Q}_0\bigl(x;W(t)\bigr)\Bigr)-\gamma\Bigl(\hat{Q}_t(x')-\hat{Q}_0\bigl(x';W(t)\bigr)\Bigr)\Bigr|^2\notag\\
&\quad\le 2\Bigl(\hat{Q}_t(x)-\hat{Q}_0\bigl(x;W(t)\bigr)\Bigr)^2+2\Bigl(\hat{Q}_t(x')-\hat{Q}_0\bigl(x';W(t)\bigr)\Bigr)^2.
\#
Since $x$ and $x'$ follow the same stationary distribution $\mu$ on the right-hand side of \eqref{518450}, by Lemma \ref{qdiff} we have
\#\label{45930}
&\EE_{\text{init},\mu}\Bigl[
\bigl| \delta\bigl(x, r, x';W(t)\bigr)-\delta_0\bigl(x, r, x';W(t)\bigr)\bigr|^2
\Bigr]\notag\\
&\quad\le 4\EE_{\text{init},\mu}\Bigl[\bigl|\hat{Q}_t(x)-\hat{Q}_0\bigl(x;W(t)\bigr)\bigr|^2\Bigr]\le 16c_1B^3\cdot m^{-1/2}.
\#

For (ii) in \eqref{57200}, we have
\#\label{57147}
\bigl\|\nabla_W\hat{Q}_t(x)-\nabla_W\hat{Q}_0\bigl(x;W(t)\bigr)\bigr\|_2^2&=\frac{1}{m}\sum_{r=1}^m\bigl( \ind\{W_r(t)^\top x>0\}-\ind\{W_r(0)^\top x>0\}\bigr)^2 \cdot \|x\|_2^2\notag\\
&\le \frac{1}{m}\sum_{r=1}^m \ind\{|W_r(0)^\top x|\le  \|W_r(t)-W_r(0)\|_2 \},
\#
where the inequality follows from \eqref{43142} and the fact that $\|x\|_2 = 1$. 

For (iii) in \eqref{57200}, we have
\#\label{517312}
\bigl| \delta_0\bigl(x, r, x';W(t)\bigr)\bigr|^2 \le 3\Bigl( \hat{Q}_0\bigl(x;W(t)\bigr)^2+\overline{r}^2+\gamma^2 \hat{Q}_0\bigl(x';W(t)\bigr)^2 \Bigr).
\#
To obtain an upper bound of the right-hand side of \eqref{517312}, we use the fact that
\$
\bigl|\hat{Q}_0\bigl(x;W(t)\bigr)-\hat{Q}_0(x)\bigr|\le \|W(t)-W(0)\|_2\cdot\|x\|_2\le B,
\$
which follows from \eqref{q0}, and obtain
\$
\EE_{\mu}\bigl[
\hat{Q}_0\bigl(x;W(t)\bigr)^2
\bigr] 
&= \EE_{\mu}\Bigl[
\Bigl(\hat{Q}_0(x)+\hat{Q}_0\bigl(x;W(t)\bigr)-\hat{Q}_0(x)\Bigr)^2
\Bigr] \le2\EE_{\mu}[\hat{Q}_0(x)^2]+2B^2.
\$
Since $x$ and $x'$ follow the same stationary distribution $\mu$ on the right-hand side of \eqref{517312} and $|\gamma|<1$, we have
\#\label{45931}
\EE_{\mu}\Bigl[
\bigl| \delta_0\bigl(x, r, x';W(t)\bigr)\bigr|^2 
\Bigr]\le
12\EE_{\mu}[\hat{Q}_0(x)^2]+12B^2+3\overline{r}^2.
\#

Plugging \eqref{45930}, \eqref{57147}, and \eqref{45931} into \eqref{57200}, we obtain
\$
&\EE_{\text{init}}\bigl[ \|\overline{g}(t)-\overline{g}_0(t)\|_2^2\bigr] \le 32c_1B^3\cdot m^{-1/2}\\
&\quad +2\EE_{\text{init}}\Bigl[
\bigl(12\EE_{\mu}[\hat{Q}_0(x)^2]+12B^2+3\overline{r}^2\bigr)
\cdot\Bigl(
\frac{1}{m}\sum_{r=1}^m \ind\{|W_r(0)^\top x|\le  \|W_r(t)-W_r(0)\|_2 \}
\Bigr)
\Bigr].
\$
Invoking Lemmas \ref{mulm1} and \ref{mulm2}, we obtain
\$
\EE_{\text{init}}\bigl[\|\overline{g}(t)-\overline{g}_0(t)\|^2_2\bigr] \le (56c_1B^3+24c_2B+6c_1B\overline{r}^2)\cdot m^{-1/2},
\$
which finishes the proof of Lemma \ref{gdiff}.
\end{proof}

\subsection{Proof of Lemma \ref{dl}}\label{dla}
\begin{proof}
Recall that
\#
\overline{g}(t)&=\EE_{\mu}\bigl[\delta\bigl(x, r, x';W(t)\bigr)\cdot\nabla_{W}\hat{Q}\bigl(x;W(t)\bigr)\bigr], \notag\\
\overline{g}_0(t)&=\EE_{\mu}\bigl[\delta_0\bigl(x, r, x';W(t)\bigr)\cdot\nabla_{W}\hat{Q}_0\bigl(x;W(t)\bigr)\bigr]\label{516727}.
\#
We denote the locally linearized population semigradient $\overline{g}_0(t)$ evaluated at the approximate stationary point $W^*$ by
\#\label{516720}
\overline{g}_0^*=\EE_{\mu}[\delta_0(x, r, x';W^*)\cdot\nabla_{W}\hat{Q}_0(x;W^*)].
\#
For any $W(t)\ (t \in [T])$, by the convexity of $S_B$, we have
\#\label{57136}
\|W(t+1)-W^*\|^2_2&=\bigl\|\Pi_{S_B}\bigl(W(t)-\eta\cdot\overline{g}(t)\bigr)-\Pi_{S_B}(W^*-\eta\cdot\overline{g}^*_0)\bigr\|_2^2\\
&\le \bigl\|\bigl(W(t)-\eta\cdot\overline{g}(t)\bigr)-(W^*-\eta\cdot\overline{g}^*_0)\bigr\|_2^2\notag\\
&=\|W(t)-W^*\|_2^2-2\eta\cdot\bigl(\overline{g}(t)-\overline{g}_0^*\bigr)^\top \bigl(W(t)-W^*\bigr)+\eta^2\cdot\|\overline{g}(t)-\overline{g}_0^*\|_2^2.\notag
\#

We decompose the inner product $(\overline{g}(t)-\overline{g}_0^*)^\top (W(t)-W^*)$ on the right-hand side of \eqref{57136} into two terms,
\#\label{5161023}
\bigl(\overline{g}(t)-\overline{g}_0^*\bigr)^\top \bigl(W(t)-W^*\bigr)&= \bigl(\overline{g}_0(t)-\overline{g}_0^*\bigr)^\top \bigl(W(t)-W^*\bigr)+\bigl(\overline{g}(t)-\overline{g}_0(t)\bigr)^\top \bigl(W(t)-W^*\bigr)\notag\\
&\geq \bigl(\overline{g}_0(t)-\overline{g}_0^*\bigr)^\top \bigl(W(t)-W^*\bigr)-B \cdot \|\overline{g}(t)-\overline{g}_0(t)\|_2.
\#
It remains to characterize the first term $(\overline{g}_0(t)-\overline{g}_0^*)^\top (W(t)-W^*)$ on the right-hand side of \eqref{5161023}, since the second term $\|\overline{g}(t)-\overline{g}_0(t)\|_2$ is characterized by Lemma \ref{gdiff}. Note that by \eqref{516727} and \eqref{516720}, we have
\#\label{516734}
\overline{g}_0(t)-\overline{g}_0^*=\EE_\mu\Bigl[\Bigl(\delta_0\bigl(x, r, x';W(t)\bigr)-\delta_0(x, r, x';W^*)\Bigr) \cdot \nabla_W\hat{Q}_0\bigl(x;W(0)\bigr)\Bigr],
\#
where we use the following consequence of \eqref{q0},
\$
\nabla_W\hat{Q}_0\bigl(x;W(0)\bigr) = \nabla_W\hat{Q}_0(x;W^*).
\$
Moreover, by \eqref{d0} it holds that
\#\label{516735}
&\delta_0\bigl(x, r, x';W(t)\bigr)-\delta_0(x, r, x';W^*)\notag\\
&\quad=\Bigl(\hat{Q}_0\bigl(x;W(t)\bigr)-\hat{Q}_0(x;W^*)\Bigr)-\gamma\Bigl(\hat{Q}_0\bigl(x';W(t)\bigr)-\hat{Q}_0(x';W^*)\Bigr).
\#
Combining \eqref{q0}, \eqref{516734}, and \eqref{516735}, we have
\#\label{57137}
&\bigl(\overline{g}_0(t)-\overline{g}^*_0\bigr)^\top \bigl(W(t)-W^*\bigr)\notag\\
&\quad=\EE_\mu\Bigl[ \Bigl(\delta_0\bigl(x, r, x';W(t)\bigr)-\delta_0(x, r, x';W^*)\Bigr)\cdot \Bigl( \nabla_W\hat{Q}_0\bigl(x;W(0)\bigr)^\top\bigl(W(t)-W^*\bigr) \Bigr) \Bigr]\notag\\
&\quad=\EE_\mu\Bigl[
\Bigl(\hat{Q}_0\bigl(x;W(t)\bigr)-\hat{Q}_0(x;W^*)\Bigr)^2\notag\\
&\quad\qquad-\gamma\Bigl(\hat{Q}_0\bigl(x;W(t)\bigr)-\hat{Q}_0(x;W^*)\Bigr)\cdot\Bigl(\hat{Q}_0\bigl(x';W(t)\bigr)-\hat{Q}_0(x';W^*)\Bigr)
\Bigr]\notag\\
&\quad\ge (1-\gamma) \cdot \EE_\mu\Bigl[
\Bigl(\hat{Q}_0\bigl(x;W(t)\bigr)-\hat{Q}_0(x;W^*)\Bigr)^2\Bigr],
\#
where the last inequality is from the fact that $x$ and $x'$ have the same marginal distribution under $\mu$ and therefore by the Cauchy-Schwarz inequality,
\$
&\EE_{\mu}\Bigl[\Bigl(\hat{Q}_0\bigl(x;W(t)\bigr)-\hat{Q}_0(x;W^*)\Bigr)\cdot\Bigl(\hat{Q}_0\bigl(x';W(t)\bigr)-\hat{Q}_0(x';W^*)\Bigr)
\Bigr]\\
&\quad\le \EE_{\mu}\Bigl[\Bigl(\hat{Q}_0\bigl(x;W(t)\bigr)-\hat{Q}_0(x;W^*)\Bigr)^2\Bigr]^{1/2}\cdot\EE_{\mu}\Bigl[\Bigl(\hat{Q}_0\bigl(x';W(t)\bigr)-\hat{Q}_0(x';W^*)\Bigr)^2\Bigr]^{1/2}\\
&\quad=\EE_{\mu}\Bigl[\Bigl(\hat{Q}_0\bigl(x;W(t)\bigr)-\hat{Q}_0(x;W^*)\Bigr)^2\Bigr].
\$
The inequality in \eqref{57137} is the key to our convergence result. It shows that the locally linearized population semigradient update $\overline{g}_0(t)$ is one-point monotone with respect to the approximate stationary point $W^*$.

Also, for $\|\overline{g}(t)-\overline{g}_0^*\|^2_2$ on the right-hand side of \eqref{57136}, we have
\#\label{57138}
\|\overline{g}(t)-\overline{g}_0^*\|_2^2\le2\|\overline{g}_0(t)-\overline{g}_0^*\|_2^2+2\|\overline{g}(t)-\overline{g}_0(t)\|_2^2.
\#
For the first term on the right-hand side of \eqref{57138}, by \eqref{516734}, \eqref{516735}, and the Cauchy-Schwarz inequality, we have
\#\label{57139}
\|\overline{g}_0(t)-\overline{g}_0^*\|_2^2&=\Bigl\|\EE_{\mu}\Bigl[\Bigl(  \delta_0\bigl(x, r, x';W(t)\bigr)-\delta_0(x, r, x';W^*)\Bigr)\cdot \nabla_{W}\hat{Q}_0\bigl(x;W(0)\bigr) \Bigr]\Bigr\|^2\notag\\
&\le\EE_{\mu}\Bigl[\Bigl(
\hat{Q}_0\bigl(x;W(t)\bigr)-\hat{Q}_0(x;W^*)-\gamma\hat{Q}_0\bigl(x';W(t)\bigr)+\gamma\hat{Q}_0(x';W^*)\Bigr)^2
\Bigr]\notag\\
&\le 4\EE_\mu\Bigl[
\Bigl(\hat{Q}_0\bigl(x;W(t)\bigr)-\hat{Q}_0(x;W^*)\Bigr)^2\Bigr],
\#
where the first inequality follows from the fact 
\$
\bigl\|\nabla_{W}\hat{Q}_0\bigl(x;W(0)\bigr)\bigr\|_2 \leq \| x \|_2 = 1.
\$
Plugging \eqref{57137}, \eqref{57138}, and \eqref{57139} into \eqref{57136}, we finish the proof of Lemma \ref{dl}.
\end{proof}

\subsection{Proof of Lemma \ref{sdl}}\label{sdla}
\begin{proof}
For any $W(t)\ (t\in [T])$, by the convexity of $S_B$, \eqref{514210}, and \eqref{516720}, we have
\# \label{1223731}
\|W(t+1)-W^*\|^2_2&=\bigl\|\Pi_{S_B}\bigl(W(t)-\eta\cdot g(t)\bigr)-\Pi_{S_B}(W^*-\eta\cdot\overline{g}^*_0)\bigr\|_2^2 \\
&\le \bigl\|\bigl(W(t)-\eta\cdot g(t)\bigr)-(W^*-\eta\cdot\overline{g}^*_0)\bigr\|_2^2 \notag\\
&=\|W(t)-W^*\|_2^2-2\eta\cdot\bigl(g(t)-\overline{g}_0^*\bigr)^\top \bigl(W(t)-W^*\bigr)+\eta^2\cdot\|g(t)-\overline{g}_0^*\|_2^2.\notag
\#
Taking expectation on both sides conditional on $W(t)$, we obtain
\#\label{5161100}
&\EE_\mu\bigl[\|W(t+1)-W^*\|_2^2\,\big|\,W(t)\bigr]\\
&\quad \le\|W(t)-W^*\|_2^2-2\eta\cdot\bigl(\overline{g}(t)-\overline{g}_0^*\bigr)^\top \bigl(W(t)-W^*\bigr) +\eta^2\cdot \EE_\mu\bigl[\|g(t)-\overline{g}_0^*\|^2_2\,\big|\,W(t)\bigr].\notag
\#
For the inner product $(\overline{g}(t)-\overline{g}_0^*)^\top (W(t)-W^*)$ on the right-hand side of \eqref{5161100}, it follows from \eqref{5161023} and \eqref{57137} that
\$
\bigl(\overline{g}(t)-\overline{g}_0^*\bigr)^\top \bigl(W(t)-W^*\bigr) \ge (1-\gamma)\cdot \EE_\mu\Bigl[
\Bigl(\hat{Q}_0\bigl(x;W(t)\bigr)-\hat{Q}_0(x;W^*)\Bigr)^2\Bigr] -B \cdot \|\overline{g}(t)-\overline{g}_0(t)\|_2.
\$
Meanwhile, for $\EE_\mu[\|g(t)-\overline{g}_0^*\|^2_2\,|\,W(t)]$ on the right-hand side of \eqref{5161100}, we have the decomposition
\$
&\EE_\mu\bigl[
\|g(t)-\overline{g}_0^*\|^2_2\,\big|\,W(t)
\bigr]=
\|\overline{g}(t)-\overline{g}_0^*\|^2_2+
\EE_{\mu}\bigl[
\|g(t)-\overline{g}(t)\|^2_2
\,\big|\,W(t) \bigr]\\
&\quad \le 8\EE_\mu\Bigl[
\Bigl(\hat{Q}_0\bigl(x;W(t)\bigr)-\hat{Q}_0(x;W^*)\Bigr)^2\,\Big|\,W(t)\Bigr] + 2\|\overline{g}(t)-\overline{g}_0(t)\|_2^2 + \EE_{\mu}\bigl[
\|g(t)-\overline{g}(t)\|^2_2
\,\big|\,W(t) \bigr],
\$
where the inequality follows from \eqref{57138} and \eqref{57139}. 
Taking expectation on the both sides of \eqref{5161100} with respect to $W(t)$, we complete the proof of Lemma \ref{sdl}.
\end{proof}

\subsection{Proof of Theorem \ref{mainthm1}}\label{thm1p}
\begin{proof}
 By Lemma \ref{gdiff} we have
 \#
&\EE_{\text{init}}\bigl[\|\overline{g}(t)-\overline{g}_0(t)\|_2^2\bigr]=O(B^3m^{-1/2}),\label{520858}\\
&\EE_{\text{init}}\bigl[B\cdot\|\overline{g}(t)-\overline{g}_0(t)\|_2\bigr]=O(B^{5/2}m^{-1/4}). \label{520902}
 \#
Setting $\eta=(1-\gamma)/8$ in Algorithm \ref{td0}, by \eqref{520858}, \eqref{520902}, and Lemma \ref{dl}, we have
\# \label{441207}
\EE_{\text{init},\mu}\Bigl[
\Bigl(\hat{Q}_0\bigl(x;W(t)\bigr)-\hat{Q}_0(x;W^*)\Bigr)^2\Bigr]&=\frac{\EE_{\text{init}}\bigl[\|W(t)-W^*\|^2_2-\|W(t+1)-W^*\|^2_2\bigr]}{(1-\gamma)^2/8}\\
&\qquad+O(B^3m^{-1/2}+B^{5/2}m^{-1/4}).\notag
\#
Telescoping \eqref{441207} for $t=0,\ldots,T-1$, we obtain
\$
&\frac{1}{T}\sum_{t=0}^{T-1}\EE_{\text{init},\mu}\Bigl[
\Bigl(\hat{Q}_0\bigl(x;W(t)\bigr)-\hat{Q}_0(x;W^*)\Bigr)^2\Bigr]\\
&\quad=
\frac{\EE_{\text{init}}\bigl[\|W(0)-W^*\|_2^2-\|W(T)-W^*\|_2^2\bigr]}{T(1-\gamma)^2/8}+O(B^3m^{-1/2}+B^{5/2}m^{-1/4})\\
&\quad\le \frac{8B^2}{T(1-\gamma)^2}+O(B^3m^{-1/2}+B^{5/2}m^{-1/4}).
\$
Recall that as define in \eqref{q0}, $\hat{Q}_0(\cdot\,;W)$ is linear in $W$. By Jensen's inequality, we have
\#
\EE_{\text{init},\mu}\bigl[
\bigl(\hat{Q}_0(x;\overline{W})-\hat{Q}_0(x;W^*)\bigr)^2\bigr]\le\frac{8B^2}{T(1-\gamma)^2}+O(B^3m^{-1/2}+B^{5/2}m^{-1/4}).\label{57249}
\#
Next we characterize the output $\hat{Q}_{\text{out}}(\cdot) = \hat{Q}(\cdot\,;\overline{W})$ of Algorithm \ref{td0}. Since $ S_B$ is convex and $\overline{W}\in S_B$, by Lemma \ref{qdiff} we have
\#\label{520913}
\EE_{\text{init},\mu}\bigl[ \bigl(\hat{Q}_0(x;\overline{W})-\hat{Q}_0(x;W^*)\bigr)^2\bigr]=O(B^3m^{-1/2}).
\#
Using the Cauchy-Schwarz inequality we have
\$
&\EE_{\text{init},\mu}\bigl[
\bigl(\hat{Q}_\text{out}(x)-\hat{Q}_0(x;W^*)\bigr)^2\bigr] \\
&\quad\le
\EE_{\text{init},\mu}\bigl[
2\bigl(\hat{Q}(x;\overline{W})-\hat{Q}_0(x;\overline{W})\bigr)^2+2\bigl(\hat{Q}_0(x;\overline{W})-\hat{Q}_0(x;W^*)\bigr)^2\bigr].
\$
Here we plug in \eqref{57249} and \eqref{520913} and obtain
\#\label{451032}
\EE_{\text{init},\mu}\bigl[
\bigl(\hat{Q}_\text{out}(x)-\hat{Q}_0(x;W^*)\bigr)^2\bigr]\le\frac{16B^2}{T(1-\gamma)^2}+O(B^3m^{-1/2}+B^{5/2}m^{-1/4}),
\#
which completes the proof of Theorem \ref{mainthm1}.
\end{proof}

\subsection{Proof of Theorem \ref{mainthm2}}\label{thm2p}
\begin{proof}
Similar to \eqref{441207}, by Lemmas \ref{bdvar}, \ref{gdiff}, and \ref{sdl} we have
\# \label{451010}
&\EE_{\text{init},\mu}\Bigl[
\Bigl(\hat{Q}_0\bigl(x;W(t)\bigr)-\hat{Q}_0(x;W^*)\Bigr)^2\Bigr]\notag\\
&\quad\le
\frac{ \EE_{\text{init}}\bigl[\|W(t)-W^*\|_2^2\bigr]-\EE_{\text{init}}\bigl[\|W(t+1)-W^*\|^2_2\bigr]+\eta^2\cdot\sigma^2_g}{2\eta(1-\gamma)-8\eta^2}\\
&\quad\qquad+O(B^3m^{-1/2}+B^{5/2}m^{-1/4}).\notag
\#
Telescoping \eqref{451010} for $t=0,\ldots,T-1$, by $\eta^2\leq 1/T$ we have
\# \label{451014}
&\frac{1}{T}\sum_{t=0}^{T-1}\EE_{\text{init},\mu}\Bigl[
\Bigl(\hat{Q}_0\bigl(x;W(t)\bigr)-\hat{Q}_0(x;W^*)\Bigr)^2\Bigr]\notag\\
&\quad\le
\frac{ \EE_{\text{init}}\bigl[\|W(t)-W^*\|^2_2\bigr]+\sigma^2_g}{T\cdot \bigl( 2\eta(1-\gamma)-8\eta^2\bigr)}+O(B^3m^{-1/2}+B^{5/2}m^{-1/4})\notag\\
&\quad\le \frac{B^2+\sigma^2_g}{\sqrt{T}}\cdot\frac{1}{\sqrt{T}\cdot \bigl(2\eta(1-\gamma)-8\eta^2\bigr)}
+O(B^3m^{-1/2}+B^{5/2}m^{-1/4}),
\#
where $\eta=\min\{1/\sqrt{T},(1-\gamma)/8\}$. Note that when $T\ge (8/(1-\gamma))^2$, we have $\eta=1/\sqrt{T}$ and 
\$
\sqrt{T}\cdot \bigl(2\eta(1-\gamma)-8\eta^2\bigr)=2(1-\gamma)-8/\sqrt{T}\ge 1-\gamma.
\$
Meanwhile, when $T<(8/(1-\gamma))^2$, we have $\eta=(1-\gamma)/8$ and 
\$
\sqrt{T}\cdot \bigl(2\eta(1-\gamma)-8\eta^2\bigr)=\sqrt{T}\cdot (1-\gamma)^2/8 \ge (1-\gamma)^2/8.
\$
Since $|1-\gamma|<1$, we obtain that for any $T \in \NN$,
\#\label{518800}
\frac{1}{\sqrt{T}\cdot \bigl(2\eta(1-\gamma)-8\eta^2\bigr)}\le \frac{8}{(1-\gamma)^2}.
\#
Similar to \eqref{57249} and \eqref{451032}, by combining \eqref{451014} and \eqref{518800} with Lemma \ref{qdiff}, we obtain
\$
\EE_{\text{init},\mu}\bigl[
\bigl(\hat{Q}_{\text{out}}(x)-\hat{Q}_0(x;W^*)\bigr)^2\bigr]\le \frac{16(B^2+\sigma^2_g)}{\sqrt{T}\cdot (1-\gamma)^2}
+O(B^3m^{-1/2}+B^{5/2}m^{-1/4}),
\$
which completes the proof of Theorem \ref{mainthm2}.
\end{proof}

%%%%%%%%%%%%%%%%%%%%%%%%%%%%%%%%%%%%%%%%%%%%%%%%%%

\section{Proofs for Section \ref{qlearning}}\label{qlearningp}
Similar to the population semigradient $\overline{g}(t)$ in policy evaluation, we define
\#
z(t)&=\Bigl(\hat{Q}\bigl(x;W(t)\bigr)-\cT\hat{Q}\bigl(x;W(t)\bigr)\Bigr) \cdot \nabla_W \hat{Q}\bigl(x;W(t)\bigr),\label{517710}\\
\overline{z}(t)&=\EE_{\mue}\Bigl[ \Bigl(\hat{Q}\bigl(x;W(t)\bigr)-\cT\hat{Q}\bigl(x;W(t)\bigr)\Bigr) \cdot \nabla_W \hat{Q}\bigl(x;W(t)\bigr) \Bigr],\label{0517538}\\
\overline{z}_0(t)&=\EE_{\mue}\Bigl[ \Bigl(\hat{Q}_0\bigl(x;W(t)\bigr)-\cT\hat{Q}_0\bigl(x;W(t)\bigr)\Bigr) \cdot \nabla_W \hat{Q}_0\bigl(x;W(t)\bigr) \Bigr],\label{0517539}\\
\overline{z}_0^*&=\EE_{\mue}\bigl[ \bigl(\hat{Q}_0(x;W^*)-\cT\hat{Q}_0(x;W^*)\bigr) \cdot \nabla_W \hat{Q}_0(x;W^*) \bigr].\label{0517540}
\#
Our proof extends that of Theorem 2 in \cite{zou2019finite} for characterizing linear Q-learning. We additionally incorporate the error of local linearization and also handle soft Q-learning in the next section. The following lemma is analogous to Lemma \ref{unique}.
\begin{lemma} \label{qlsue}
Under Assumption \ref{cond1}, there exists an approximate stationary point $W^*$ that satisfies \eqref{4251216}. Also, $\hat{Q}_0(\cdot\,;W^*)$ is unique almost everywhere.
\end{lemma}
\begin{proof}
We prove the lemma by showing that $\cT$ is a contraction in $\|\cdot\|_{\mue}$ for any $Q_1,Q_2\in\cF_{B,m}$. By the definition of the Bellman optimality operator $\cT$, we have
\$
\|\cT Q_1-\cT Q_2\|_2^2
&=\gamma^2\EE_{s\sim\mue}\bigl[ \bigl( \max_{a\in\cA}Q_1(s,a)-\max_{a\in\cA}Q_2(s,a) \bigr)^2 \bigr].
\$
Under Assumption \ref{cond1}, for any $Q_1,Q_2\in\cF_{B,m}$, we have
\$
\EE_{s\sim\mue}\bigl[\bigl( \max_{a\in\cA}Q_1(s,a)-\max_{a\in\cA}Q_2(s,a)\bigr)^2 \bigr]\le (\gamma+\nu)^{-2}\cdot \EE_{\mue}\bigl[\bigl(Q_1(x)-Q_2(x)\bigr)^2\bigr].
\$
Therefore, $\Pi_{\cF_{B,m}}\cT$ is a ${\gamma}/{(\gamma+\nu)}$-contraction in $\|\cdot\|_{\mue}$, since $\Pi_{\cF_{B,m}}$ is nonexpansive. Since the set $S_B$ of feasible $W$ is closed and bounded, $\cF_{B,m}$ is complete under $\|\cdot\|_{\mue}$. Thus, $\Pi_{\cF_{B,m}}\cT$ has a unique fixed point $\hat{Q}_0(\cdot\,;W^*)$ in $\cF_{B,m}$, which corresponds to $W^*$.
\end{proof}

The following lemma is analogous to Lemma \ref{gdiff} with a similar proof.
\begin{lemma}\label{zdiff}
For any $t \in [T]$, we have
\$
\EE_{\text{init}}\bigl[\|\overline{z}(t)-\overline{z}_0(t)\|^2_2\bigr]=O(B^3m^{-1/2}).
\$
\end{lemma}
\begin{proof}
By the definitions of $\overline{z}(t)$ and $\overline{z}_0(t)$ in \eqref{0517538} and \eqref{0517539}, respectively, we have
\#\label{518830}
&\|\overline{z}(t)-\overline{z}_0(t)\|_2\notag\\
&\quad=\Bigl\|\EE_{\mue}\Bigl[
\bigl(\hat{Q}_t(x)-\cT\hat{Q}_t(x)\bigr)\cdot \nabla_W\hat{Q}_t(x)-\Bigl(\hat{Q}_0\bigl(x;W(t)\bigr)-\cT\hat{Q}_0\bigl(x;W(t)\bigr)\Bigr)\cdot \nabla_W\hat{Q}_0\bigl(x;W(t)\bigr)
\Bigr]\Bigr\|_2\notag\\
&\quad= \Bigl\|
\EE_{\mue}\Bigl[
\Bigl(\hat{Q}_t(x)-\cT\hat{Q}_t(x)-\hat{Q}_0\bigl(x;W(t)\bigr)+\cT\hat{Q}_0\bigl(x;W(t)\bigr) \Bigr)\cdot \nabla_W\hat{Q}_t(x)\\
&\quad\qquad+
\Bigl(\hat{Q}_0\bigl(x;W(t)\bigr)-\cT\hat{Q}_0\bigl(x;W(t)\bigr)\Bigr)\cdot \Bigl(\nabla_W\hat{Q}_t(x)-\nabla_W\hat{Q}_0\bigl(x;W(t)\bigr)\Bigr)
\Bigr]
\Bigr\|_2.\notag
\#
For notational simplicity, we define $\hat{Q}^{\sharp}_t(s)=\max_{a\in\cA}\hat{Q}_t(s,a)$. Recall that $\hat{Q}^{\sharp}_0(s;W)$ is similarly defined in Assumption \ref{cond1}. Then on the right-hand side of \eqref{518830}, we have 
\$
&\EE_{\mue}\Bigl[
\Bigl(\hat{Q}_t(x)-\cT\hat{Q}_t(x)-\hat{Q}_0\bigl(x;W(t)\bigr)+\cT\hat{Q}_0\bigl(x;W(t)\bigr) \Bigr) \cdot \nabla_W\hat{Q}_t(x)\Big]\\
&\quad=\EE_{\mue}\Bigl[
\Bigl(\hat{Q}_t(x)-\gamma\hat{Q}_t^{\sharp}(s')-\hat{Q}_0\bigl(x;W(t)\bigr)+\gamma\hat{Q}^{\sharp}_0\bigl(s';W(t)\bigr) \Bigr) \cdot \nabla_W\hat{Q}_t(x)\Big].
\$
Thus, from \eqref{518830} we obtain
\#\label{518832}
\|\overline{z}(t)-\overline{z}_0(t)\|^2_2& \le 2\EE_{\mue}\Bigl[
\bigl|\hat{Q}_t(x)-\gamma\hat{Q}_t^{\sharp}(s')-\hat{Q}_0\bigl(x;W(t)\bigr)+\gamma\hat{Q}^{\sharp}_0\bigl(s';W(t)\bigr)\bigr|^2\Bigr]\\
&\qquad
+2\EE_{\mue}\Bigl[\bigl|\hat{Q}_0\bigl(x;W(t)\bigr)-r(x)-\gamma\hat{Q}^{\sharp}_0\bigl(s';W(t)\bigr)\bigr|^2\Bigr]\notag\\
&\qquad\qquad\cdot\EE_{\mue}\Bigl[\bigl\|\nabla_W\hat{Q}_t(x)-\nabla_W\hat{Q}_0\bigl(x;W(t)\bigr)\bigr\|^2_2
\Bigr].\notag
\#
Here we use the fact that, for any $t\in [T]$,
\$
\|\nabla_W\hat{Q}_t(x)\|_2 \leq \| x \|_2 = 1.
\$ 
Taking expectation on the both sides of \eqref{518832} with respect to the random initialization, we obtain
\#\label{517609}
&\EE_{\text{init}}\bigl[ \|\overline{z}(t)-\overline{z}_0(t)\|_2^2\bigr] \notag\\
& \le \underbrace{4\EE_{\text{init},\mue} \Bigl[ \bigl|\hat{Q}_t(x)-\hat{Q}_0\bigl(x;W(t)\bigr)\bigr|^2 \Bigr]}_{\displaystyle \text{(i)}}+\underbrace{4\gamma^2\EE_{\text{init},\mue} \Big[\bigl|\hat{Q}_t^{\sharp}(s)-\hat{Q}_0^{\sharp}\bigl(s;W(t)\bigr)\bigr|^2 \Bigr]}_{\displaystyle \text{(ii)}}\\
&\quad
+\underbrace{2\EE_{\text{init}}\biggl[\EE_{\mue} \Bigl[
\bigl|\hat{Q}_0\bigl(x;W(t)\bigr)-r(x)-\gamma\hat{Q}^{\sharp}_0\bigl(s';W(t)\bigr)\bigr|^2\Bigr]\cdot
\EE_{\mue} \Bigl[
\bigl\|\nabla_W\hat{Q}_t(x)-\nabla_W\hat{Q}_0\bigl(x;W(t)\bigr)\bigr\|_2^2
 \Bigr]\biggr]}_{\displaystyle \text{(iii)}}.\notag
\#
Similar to the proof of Lemma \ref{gdiff}, we characterize the three terms on the right-hand side of \eqref{517609}. For (i) in \eqref{517609}, recall that Lemma \ref{qdiff} gives 
\$
\EE_{\text{init},\mue}\Bigl[ \bigl|\hat{Q}_t(x)-\hat{Q}_0\bigl(x;W(t)\bigr)\bigr|^2 \Bigr]\le 4c_1B^3\cdot m^{-1/2}.
\$
We establish a similar upper bound of (ii) in \eqref{517609}. Note that
\#\label{517320}
\bigl| \hat{Q}_t^{\sharp}(s)-\hat{Q}_0^{\sharp}\bigl(s;W(t)\bigr)\bigr|\le
\max_{a\in\cA}
\bigl| \hat{Q}_t(s,a)-\hat{Q}_0\bigl(s,a;W(t)\bigr)\bigr|.
\#
Similar to Lemma \ref{qdiff}, we have the following lemma for characterizing the right-hand side of \eqref{517320}.
\begin{lemma}\label{maxqdiff}
Under Assumption \ref{asmp2}, there exists a constant $c_4>0$ such that for any $t \in [T]$, it holds that
\$
\EE_{\text{init},s\sim\mue}\Bigl[ \max_{a\in\cA} \bigl| \hat{Q}_t(s,a)-\hat{Q}_0\bigl(s,a;W(t)\bigr)\bigr|^2 \Bigr]\le c_4B^3\cdot m^{-1/2}.
\$
\end{lemma}
\begin{proof}
See Appendix \ref{proof:aux} for a detailed proof.
\end{proof}
By Lemma \ref{maxqdiff}, (ii) in \eqref{517609} satisfies
\$
\EE_{\text{init},s\sim\mue} \Bigl[\bigl|\hat{Q}_t^{\sharp}(s)-\hat{Q}_0^{\sharp}\bigl(s;W(t)\bigr)\bigr|^2 \Bigr]\le \EE_{\text{init},s\sim\mue}\Bigl[ \max_{a\in\cA} \bigl| \hat{Q}_t(s,a)-\hat{Q}_0\bigl(s,a;W(t)\bigr)\bigr|^2 \Bigr] \leq c_4B^3\cdot m^{-1/2}.
\$
For (iii) in \eqref{517609}, in the proof of Lemma \ref{gdiff}, we show that
\$
\EE_{\mue} \Bigl[
\bigl\|\nabla_W\hat{Q}_t(x)-\nabla_W\hat{Q}_0\bigl(x;W(t)\bigr)\bigr\|_2^2
 \Bigr]\le  \EE_{\mue}\Bigl[\frac{1}{m}\sum_{r=1}^m \ind\{|W_r(0)^\top x|\le  \|W_r(t)-W_r(0)\|_2 \} \Bigr].
\$
Meanwhile, by the definition that $\hat{Q}_{0}^{\sharp}(s;W)=\max_{a\in\cA}\hat{Q}_0(s,a;W)$ and \eqref{q0}, we have
\$
&\bigl|\hat{Q}_0\bigl(x;W(t)\bigr)-r(x)-\gamma\hat{Q}^{\sharp}_0\bigl(s';W(t)\bigr)\bigr|^2 \\
&\quad \le 3\bigl|\hat{Q}_0\bigl(x;W(t)\bigr)\bigr|^2+3\overline{r}^2+3\bigl|\hat{Q}_0\bigl(s',a_{\text{max}}';W(t)\bigr)\bigr|^2\\
&\quad \le 6\bigl|\hat{Q}_0\bigl(x;W(0)\bigr)\bigr|^2+3\overline{r}^2+6\bigl|\hat{Q}_0\bigl(s',a_{\text{max}}';W(0)\bigr)\bigr|^2+12B^2,
\$
where $a_{\text{max}}' = \argmax_{a'\in\cA}\hat{Q}_0(s',a';W(t))$. Then applying Lemmas \ref{mulm1}, \ref{mulm2}, and \ref{mulm3}, we obtain that (iii) in \eqref{517609} satisfies
\#\label{518125}
&2\EE_{\text{init}}\biggl[\EE_{\mue} \Bigl[
\bigl|\hat{Q}_0\bigl(x;W(t)\bigr)-r(x)-\gamma\hat{Q}^{\sharp}_0\bigl(s';W(t)\bigr)\bigr|^2\Bigr]\cdot
\EE_{\mue} \Bigl[
\bigl\|\nabla_W\hat{Q}_t(x)-\nabla_W\hat{Q}_0\bigl(x;W(t)\bigr)\bigr\|_2^2
 \Bigr]\biggr]\notag\\
 &\quad =O(B^3m^{-1/2}).
\#
Combining the upper bounds of (i)-(iii) in \eqref{517609}, we complete the proof of Lemma \ref{zdiff}.
\end{proof}

\subsection{Proof of Theorem \ref{thm3}}\label{proof:thm3}
\begin{proof}
Recall that $z(t)$, $\overline{z}(t)$, $\overline{z}_0(t)$, and $\overline{z}^*_0$ are defined in \eqref{517710}-\eqref{0517540}, respectively. Similar to the proof of Theorem \ref{mainthm1}, we have
\# \label{429301}
\|W(t+1)-W^*\|_2^2&=\bigl\|\Pi_{ S_B}\bigl(W(t)-\eta\cdot z(t)\bigr)- \Pi_{ S_B}(W^*-\eta\cdot\overline{z}^*_0)\bigr\|^2_2\\
&\le \|W(t)-W^*\|_2^2-2\eta\cdot \bigl(z(t)-\overline{z}^*_0\bigr)^\top \bigl(W(t)-W^*\bigr)+\eta^2\cdot \|z(t)-\overline{z}^*_0 \|_2^2.\notag
\#

To characterize the inner product on the right-hand side of \eqref{429301}, we take conditional expectation and obtain
\#\label{518905}
&\EE_{\mue}\bigl[\bigl(z(t)-\overline{z}_0^*\bigr)^\top \bigl(W(t)-W^*\bigr)\,\big|\,W(t)\bigr]\notag\\
&\quad= \bigl(\overline{z}_0(t)-\overline{z}_0^*\bigr)^\top \bigl(W(t)-W^*\bigr)+\bigl(\overline{z}(t)-\overline{z}_0(t)\bigr)^\top \bigl(W(t)-W^*\bigr)\notag\\
&\quad\geq \bigl(\overline{z}_0(t)-\overline{z}_0^*\bigr)^\top \bigl(W(t)-W^*\bigr)-B \cdot \|\overline{z}(t)-\overline{z}_0(t)\|_2.
\#
 We establish a lower bound of $(\overline{z}_0(t)-\overline{z}^*_0)^\top (W(t)-W^*)$ as follows. By \eqref{0517539} and \eqref{0517540}, we have 
\#\label{517325}
&\bigl(\overline{z}_0(t)-\overline{z}^*_0\bigr)^\top \bigl(W(t)-W^*\bigr)\notag\\
&= \EE_{\mue}\Bigl[ \Bigl(\hat{Q}_0\bigl(x;W(t)\bigr)-\gamma\hat{Q}_0^{\sharp}\bigl(s';W(t)\bigr)-\hat{Q}_0(x;W^*)+\gamma\hat{Q}^{\sharp}_0(s';W^*)\Bigr)\cdot \nabla_W \hat{Q}_0\bigl(x;W(t)\bigr)  \Bigr]^\top \bigl(W(t)-W^*\bigr) \notag\\
&=\EE_{\mue} \Bigl[ \Bigl(\hat{Q}_0\bigl(x;W(t)\bigr)-\hat{Q}_0(x;W^*)\Bigr)^2 \Bigr]  \\
&\qquad -\gamma\EE_{\mue}\Bigl[ \Bigl(\hat{Q}_0^{\sharp}\bigl(s';W(t)\bigr)-\hat{Q}^{\sharp}_0(s';W^*)\Bigr) \cdot \Bigl( \hat{Q}_0\bigl(x;W(t)\bigr)-\hat{Q}_0(x;W^*) \Bigr) \Bigr].\notag
\#
Applying H\"{o}lder's inequality to the second term on the right-hand side of \eqref{517325}, we obtain
\$
&\EE_{\mue}\Bigl[ \Bigl(\hat{Q}_0^{\sharp}\bigl(s';W(t)\bigr)-\hat{Q}^{\sharp}_0(s';W^*)\Bigr) \cdot \Bigl( \hat{Q}_0\bigl(x;W(t)\bigr)-\hat{Q}_0(x;W^*) \Bigr) \Bigr]\\
&\quad\le \EE_{\mue} \Bigl[ \Bigl(\hat{Q}_0^{\sharp}\bigl(s';W(t)\bigr)-\hat{Q}^{\sharp}_0(s';W^*)\Bigr)^2 \Bigr]^{1/2} \cdot \EE_{\mue} \Bigl[ \Bigl( \hat{Q}_0\bigl(x;W(t)\bigr)-\hat{Q}_0(x;W^*) \Bigr)^2\Bigr]^{1/2}.
\$
By Assumption \ref{cond1}, we have
\$
\EE_{\mue} \Bigl[ \Bigl(\hat{Q}^{\sharp}_0\bigl(s';W(t)\bigr)-\hat{Q}^{\sharp}_0(s';W^*)\Bigr)^2 \Bigr]^{1/2}\le 1/(\gamma+\nu)\cdot
\EE_{\mue} \Bigl[ \Bigl(\hat{Q}_0\bigl(x;W(t)\bigr)-\hat{Q}_0(x;W^*)\Bigr)^2 \Bigr]^{1/2},
\$
which implies that, on the right-hand side of \eqref{517325}, 
\$
&\gamma\EE_{\mue}\Bigl[ \Bigl(\hat{Q}_0^{\sharp}\bigl(s';W(t)\bigr)-\hat{Q}^{\sharp}_0(s';W^*)\Bigr) \cdot \Bigl( \hat{Q}_0\bigl(x;W(t)\bigr)-\hat{Q}_0(x;W^*) \Bigr) \Bigr]\\
&\quad\le \gamma/(\gamma+\nu)\cdot\EE_{\mue} \Bigl[ \Bigl(\hat{Q}_0\bigl(x;W(t)\bigr)-\hat{Q}_0(x;W^*)\Bigr)^2 \Bigr].
\$
Therefore, from \eqref{517325} we obtain
\#\label{518910}
\bigl(\overline{z}_0(t)-\overline{z}^*_0\bigr)^\top \bigl(W(t)-W^*\bigr) \ge \nu/(\gamma+\nu)\cdot\EE_{\mue} \Bigl[ \Bigl(\hat{Q}_0\bigl(x;W(t)\bigr)-\hat{Q}_0(x;W^*)\Bigr)^2 \Bigr].
\#

Similar to the proof of Theorem \ref{mainthm1}, for $\|z(t)-\overline{z}^*_0 \|_2^2$ on the right-hand side of \eqref{429301}, we have
\#\label{518915}
&\EE_{W, \mue}\bigl[\|z(t)-\overline{z}_0^*\|_2^2\bigr] \notag\\
&\quad \le \EE_{W, \mue}\bigl[\|z(t)-\overline{z}(t)\|_2^2\bigr]+ \EE_W\bigl[ 2\|\overline{z}(t)-\overline{z}_0(t)\|_2^2+2\|\overline{z}_0(t)-\overline{z}_0^*\|_2^2\bigr],
\#
where the expectation of $\|\overline{z}(t)-\overline{z}_0(t)\|_2^2$ on the right-hand side is characterized by Lemma \ref{zdiff}, while the expectation of $\|\overline{z}_0(t)-\overline{z}_0^*\|_2^2$ has the following upper bound,
\#\label{518920}
&\EE_{W}\bigl[\|\overline{z}_0(t)-\overline{z}_0^*\|_2^2\bigr]\notag\\
&\quad=\EE_{W, \mue}\Bigl[ \Bigl(\hat{Q}_0\bigl(x;W(t)\big)-\gamma\hat{Q}_0^{\sharp}\bigl(s';W(t)\bigr)-\hat{Q}_0(x;W^*)+\gamma\hat{Q}^{\sharp}_0(s';W^*)\Bigr)\cdot \nabla_W \hat{Q}_0\bigl(x;W(t)\bigr)  \Bigr] ^2\notag\\
&\quad\le \EE_{W, \mue}\Bigl[ \Bigl(\hat{Q}_0\bigl(x;W(t)\big)-\gamma\hat{Q}_0^{\sharp}\bigl(s';W(t)\bigr)-\hat{Q}_0(x;W^*)+\gamma\hat{Q}^{\sharp}_0(s';W^*)\Bigr)^2 \Bigr]\notag\\
&\quad\le 4\EE_{W, \mue}\Bigl[ \Bigl(\hat{Q}_0\bigl(x;W(t)\big)-\hat{Q}_0(x;W^*)\Bigr)^2 \Bigr].
\#
Here the last inequality follows from Assumption \ref{cond1} as $\gamma/(\gamma+\nu) < 1$. 

Plugging \eqref{518905}, \eqref{518910}, \eqref{518915}, and \eqref{518920} into \eqref{429301} yields the following inequality, which parallels Lemma \ref{sdl},
\#\label{5176461}
&\EE_{W, \mue}\bigl[\|W(t+1)-W^*\|_2^2\bigr]\\
&\quad\le \EE_W\bigl[ \|W(t)-W^*\|_2^2\bigr]-\bigl({2\eta\nu}/(\gamma+\nu)-8\eta^2\bigr) \cdot \EE_{W, \mue}\Bigl[ \Bigl(\hat{Q}_0\bigl(x;W(t)\big)-\hat{Q}_0(x;W^*)\Bigr)^2 \Bigr]\notag\\
&\quad\qquad +\EE_{W}\bigl[2\eta^2\cdot\|\overline{z}(t)-\overline{z}_0(t)\|^2_2+2\eta B\cdot\|\overline{z}(t)-\overline{z}_0(t)\|_2 \bigr]+\EE_{W, \mue}\bigl[\eta^2\cdot\|z(t)-\overline{z}(t)\|^2_2\bigr].\notag
\#
Rearranging terms in \eqref{5176461}, we obtain
\#\label{517646}
&\EE_{W, \mue}\Bigl[ \Bigl(\hat{Q}_0\bigl(x;W(t)\big)-\hat{Q}_0(x;W^*)\Bigr)^2 \Bigr]\notag\\
&\quad\le
\bigl({2\eta\nu}/(\gamma+\nu)-8\eta^2\bigr)^{-1}\cdot \Bigl(\EE_W\bigl[ \|W(t)-W^*\|_2^2\bigr]-\EE_{W, \mue}\bigl[\|W(t+1)-W^*\|_2^2\bigr] \Bigr)\\
&\quad\qquad +\EE_{W}\bigl[2\eta^2\cdot\|\overline{z}(t)-\overline{z}_0(t)\|^2_2+2\eta B\cdot\|\overline{z}(t)-\overline{z}_0(t)\|_2 \bigr]+\EE_{W, \mue}\bigl[\eta^2\cdot\|z(t)-\overline{z}(t)\|^2_2\bigr].\notag
\#

In parallel with Lemma \ref{bdvar}, we establish an upper bound of the variance $\EE_{W, \mue}[\|z(t)-\overline{z}(t)\|^2_2]$ in the right-hand side of \eqref{517646}, which is independent of $t$ and $m$. Note that by $\|\nabla_W \hat{Q}_t(s,a)\|_2\le1$, we have
\#\label{5181008}
\|z(t)\|^2_2&\le \bigl|\hat{Q}_t(s,a)-r(s,a)-\gamma\max_{a'\in\cA}\hat{Q}_t(s',a')\bigr|^2\cdot \|\nabla_W \hat{Q}_t(s,a)\|_2^2\notag\\
&\le 3\hat{Q}_t(s,a)^2+3\overline{r}^2+3\max_{a'\in\cA} \hat{Q}_t(s',a')^2.
\#
To characterize $\hat{Q}_t(s,a)^2$, we have
\#\label{5181010}
\hat{Q}_t(s,a)^2& = \bigl(\hat{Q}_t(s,a) -\hat{Q}_0(s,a) + \hat{Q}_0(s,a) \bigr)^2\notag\\
&\le 2\bigl|\hat{Q}_t(s,a)-\hat{Q}_0(s,a)\bigr)\bigr|^2+2\hat{Q}_0(s,a)^2\notag \\
&\le 2B^2+2\hat{Q}_0(s,a)^2,
\#
where the second inequality comes from \eqref{520946}. Similarly, to characterize $\max_{a'\in\cA} \hat{Q}_t(s',a')$ in \eqref{5181008}, we take maximum on the both side of \eqref{5181010} and obtain
\#\label{5181011}
\max_{a'\in\cA}\hat{Q}_t(s',a')^2\le 2B^2+2\max_{a'\in\cA}\hat{Q}_0(s',a')^2.
\#
Plugging \eqref{5181010} and \eqref{5181011} into \eqref{5181008}, we obtain
\#\label{517909}
\|z(t)\|^2_2&\le 6\hat{Q}_0(s,a)^2+3\overline{r}^2+6\max_{a'\in\cA} \hat{Q}_0(s',a')^2+12B^2.
\#
To upper bound the expectation of \eqref{517909}, it remains to characterize $\EE_{\text{init},\mue}[\hat{Q}_0(s,a)^2]$ and 
\$\EE_{\text{init},s'\sim\mue}\bigl[\max_{a'\in\cA}\hat{Q}_0(s',a')^2\bigr].\$ 
In fact, for any $(s,a)$, $\EE_{\text{init}}[\hat{Q}_0(s,a)^2]$ has the following uniform upper bound,
\#\label{5181025}
\EE_{\text{init}}[\hat{Q}_0(s,a)^2] = \EE_{w\sim N(0,I_d/d)}[\sigma(w^\top x)^2]\le \EE_{w\sim N(0,I_d/d)} \bigl[\|w\|_2^2\bigr]=1.
\#
For $\EE_{\text{init},s'\sim\mue}[\max_{a'\in\cA}\hat{Q}_0(s',a')^2]$, we use the inequality
\$
&\EE_{\text{init},s'\sim\mue}\bigl[\max_{a'\in\cA}\hat{Q}_0(s',a')^2\bigr]\\
&\quad\le\EE_{\text{init},s'\sim\mue}\Bigl[\sum_{a'\in\cA}\hat{Q}_0(s',a')^2\Bigr]\\
&\quad= \sum_{a'\in\cA}\EE_{s'\sim\mue}\bigl[ \EE_{\text{init}}[\hat{Q}_0(s',a')^2\,|\,(s',a')]  \bigr] \le |\cA|\cdot \EE_{w\sim N(0,I_d/d)} \bigl[\|w\|_2^2\bigr],
\$
which gives us the same upper bound as in \eqref{5181025} but with an additional factor $|\cA|$. Therefore, we know that the variance $\EE_{\text{init},\mue}[\|z(t)-\overline{z}(t)\|^2_2]$ in the expectation of \eqref{517646} has an upper bound $\sigma_z^2=O(B^2)$, according to the fact that $\EE_{\text{init},\mue}[\|z(t)-\overline{z}(t)\|^2_2]\le\EE_{\text{init},\mue}[\|z(t)\|^2_2]$. 

With the variance term upper bounded, we take expectation on \eqref{517646} with respect to the random initialization and obtain
\#\label{524116}
&\EE_{\text{init}, \mue}\Bigl[ \Bigl(\hat{Q}_0\bigl(x;W(t)\big)-\hat{Q}_0(x;W^*)\Bigr)^2 \Bigr]\\
&\quad\le
\bigl({2\eta\nu}/(\gamma+\nu)-8\eta^2\bigr)^{-1}\cdot \Bigl(\EE_{\text{init}}\bigl[ \|W(t)-W^*\|_2^2\bigr]-\EE_{\text{init}}\bigl[\|W(t+1)-W^*\|_2^2\bigr]+\eta^2\sigma_z^2 \Bigr)\notag\\
&\quad\qquad +O(B^3m^{-1/2}+B^{5/2}m^{-1/4}).\notag
\#
We set $\eta=\min\{1/\sqrt{T}, \nu/8(\gamma+\nu) \}$ and telescope \eqref{524116} for $t=0,\ldots,T-1$. Then the rest proof mirrors that of Theorem \ref{mainthm2}. Thus, we complete the proof of Theorem \ref{thm3}.
\end{proof}

%%%%%%%%%%%%%%%%%%%%%%%%%%%%%%%%%%%%%%%%%%%%%%%%%%

\section{From Neural Soft Q-Learning to Policy Gradient}\label{secfqs}
\subsection{Global Convergence of Neural Soft Q-Learning}
We extend the global convergence of neural Q-learning in Section \ref{qlearning} to neural soft Q-learning, where the max operator is replaced by a more general softmax operator. More specifically, we consider the soft Bellman optimality operator
\#\label{5181216}
&\cT_{\beta}(Q)(s,a)=\EE\bigl[r(s,a)+\gamma\cdot\softmax_{a'\in\cA} Q(s',a')\,\big|\, s'\sim\cP(\cdot\,|\,s,a)\bigr],\notag\\
&\text{where}~ \softmax_{a'\in\cA} Q(s',a') = \beta^{-1}\cdot \log\sum_{a'\in\cA}\exp{\bigl(\beta\cdot Q(s',a')}\bigr),
\#
which in parallel with \eqref{430421} corresponds to the update
\$
\theta'\leftarrow\theta-\eta \cdot \big( \hat{Q}_{\theta}(s,a)-r(s,a)-\gamma\cdot\softmax_{a'\in\cA}\hat{Q}_{\theta}(s',a') \bigr)\cdot \nabla_{\theta}\hat{Q}_{\theta}(s,a).
\$
See Algorithm \ref{softq} for a detailed description of such neural soft Q-learning algorithm.

In parallel with Assumption \ref{cond1}, we require the following regularity condition on the exploration policy $\pie$.
\begin{assumption}[Regularity of Exploration Policy $\pie$]\label{cond2} There exists a constant $\nu'>0$ such that for any $W_1,W_2\in S_B$, it holds that
\$
&\EE_{x\sim\mue}\bigl[\bigl( \hat{Q}_0(x;W_1)-\hat{Q}_0(x;W_2) \bigr)^2 \bigr]\notag\\
&\quad\ge (\gamma+\nu')^{2}\cdot\EE_{s\sim\mue}\bigl[\bigl(\softmax_{a\in\cA} \hat{Q}_0(s,a;W_1)-\softmax_{a\in\cA} \hat{Q}_0(s,a;W_2)\bigr)^2\bigr].
\$
\end{assumption}
We remark that, when $\beta\rightarrow\infty$, the softmax operator converges to the max operator, which implies that Assumptions \ref{cond2} and \ref{cond1} are equivalent.

%In parallel with Assumption \ref{asmp2}, we require another regularity condition on $\mue$.
%\begin{assumption}\label{asmp3}
%There exists a constant $c_4>0$ such that for any vector $w\in\real^d$ and $\tau\ge0$, it holds that
%\#\label{44145}
%\EE_{s\sim\mue,a\in\cB_{s}}\bigl[\softmax_{a\in\cA}\ind\{|w^\top (s,a)|\le  b\} \bigr]\le  c_3\cdot b/\|w\|_2.
%\#
%\end{assumption}

The approximate stationary point $W^*$ of the projected soft Q-learning satisfies 
\$
\hat{Q}_0(\cdot\,;W^*)=\Pi_{\cF_{B,m}}\cT_{\beta}\hat{Q}_0(\cdot\,;W^*),
\$
where $\hat{Q}_0(\cdot\,;W^*)$ uniquely exists by the same proof of Lemma \ref{qlsue}. Under the above regularity condition, we can extend Theorem \ref{thm3} to cover neural soft Q-learning. 

\begin{theorem}[Convergence of Stochastic Update]\label{thm4}
We set $\eta$ to be of order $T^{-1/2}$ in Algorithm \ref{softq}. Under Assumptions \ref{cond2} and \ref{asmp2}, the output $\hat{Q}_{\text{out}}$ of Algorithm \ref{softq} satisfies
\$
\EE_{\text{init},\mue}\bigl[ \bigl( \hat{Q}_{\text{out}}(x)-\hat{Q}_0(x;W^*)\bigr)^2 \bigr]= O(B^2T^{-1/2}+B^3m^{-1/2}+B^{5/2}m^{-1/4}).
\$
\end{theorem}
\begin{proof}
The proof of Theorem \ref{thm4} mirrors that of Theorem \ref{thm3} with the max operator replaced by the softmax operator in \eqref{5181216}. We prove that the same claim of Lemma \ref{zdiff} holds under Assumption \ref{asmp2}, for which it suffices to upper bound (i), (ii), and (iii) in \eqref{517609}. Note that (i) does not involve the max operator. For (ii), we lay out the following lemma.
\begin{lemma}\label{softqdiff}
For any $W\in S_B$ and the constant $c_4$ in Lemma \ref{maxqdiff}, we have
\$
\EE_{\text{init},s\sim\mue}\Bigl[ \Bigl( \softmax_{a\in\cA}\hat{Q}_t(s,a)-\softmax_{a\in\cA}\hat{Q}_0\bigl(s,a;W(t)\bigr) \Bigr)^2\Bigr]\le 
c_4 B^3\cdot m^{-1/2}.
\$
\end{lemma}
\begin{proof}
The softmax operator has the following duality. For any function $Q(s,a)$, it holds that
\#\label{518135}
\softmax_{a\in\cA} Q(s,a)= \max_{\pi(s)\in\Delta}  \EE_{a\sim\pi(s)} [ Q(s,a)] + \beta^{-1}\cdot H\bigl(\pi(s)\bigr),
\#
where $\Delta$ is the set of all probability distributions over $\cA$ and $H(\pi(s))$ is the entropy of $\pi(s)$. Hence, we obtain 
\$
\bigl| \softmax_{a\in\cA}\hat{Q}_t(s,a)-\softmax_{a\in\cA}\hat{Q}_0\bigl(s,a;W(t)\bigr)  \bigr| &\le \max_{\pi(s)\in\Delta} \bigl| \EE_{a\sim\pi(s)}\bigl[\hat{Q}_t(s,a)-\hat{Q}_0\bigl(s,a;W(t)\bigr) \bigr]\bigr|\\
&=\max_{a\in\cA}\bigl| \hat{Q}_t(s,a)-\hat{Q}_0\bigl(s,a;W(t)\bigr) \bigr|.
\$
By applying Lemma \ref{maxqdiff}, we complete the proof of Lemma \ref{softqdiff}. 
\end{proof}
Meanwhile, note that the upper bound in \eqref{518125} of (iii) in \eqref{517609} still holds by Lemma \ref{mulm3}. Thus, the claim of Lemma \ref{zdiff} holds for neural soft Q-learning. Moreover, \eqref{517646} also holds for neural soft Q-learning under Assumption \ref{cond2}. To further extend the upper bound of $\EE_{\text{init},\mue}[\|z(t)-\overline{z}(t)\|^2_2]$ to neural soft Q-learning, it remains to upper bound $\softmax_{a'\in\cA}\hat{Q}_t(s',a')$, which replaces $\max_{a'\in\cA}\hat{Q}_t(s',a')$ in \eqref{517909}, by 
\$
\bigl|\softmax_{a'\in\cA}\hat{Q}_t(s',a')\bigr|\le \max_{a'\in\cA}|\hat{Q}_t(s',a')|+\beta^{-1} \cdot\log|\cA|,
\$
where $\beta$ is the parameter of the softmax operator. Here the inequality follows from \eqref{518135}. The additional term $\beta^{-1} \cdot\log|\cA|$ is independent of $t$ and $m$. Thus, we obtain an upper bound of the variance $\EE_{\text{init},\mue}[\|z(t)-\overline{z}(t)\|^2_2]$, which is independent of $t$ and $m$. With \eqref{517646} and Lemma \ref{zdiff}, the proof of Theorem \ref{thm4} follows from that of Theorem \ref{thm3}.
\end{proof}

\begin{algorithm}
\caption{Neural Soft Q-Learning}
\begin{algorithmic}[1]
\STATE \textbf{Initialization:} $b_r\sim\text{Unif}(\{-1,1\})$, $W_r(0)\sim N(0,I_d/d)$ $(r\in[m])$, $\overline{W}=W(0)$,\\
{\color{white}\textbf{Initialization:}} $S_B=\{W\in\real^{md}:\|W-W(0)\|_2\le B\}$ $(B>0)$,\\
{\color{white}\textbf{Initialization:}} exploration policy $\pie$ such that $\pie(a\,|\,s)>0$ for any $(s,a)\in\cS\times\cA$ \\
\STATE \textbf{For} {$t=0$ to $T-2$}:
\STATE \hspace{0.15in} Sample a tuple $(s,a, r, s')$ from the stationary distribution $\mue$ of the exploration policy $\pie$
\STATE \hspace{0.15in} Bellman residual calculation: $\delta \leftarrow \hat{Q}(s,a;W(t))-r-\gamma\softmax_{a'\in\cA}\hat{Q}(s',a';W(t))$
\STATE \hspace{0.15in} TD update: $\tilde{W}(t+1)\leftarrow W(t)-\eta\delta\cdot\nabla_{W}\hat{Q}(s,a;W(t))$ \\
\STATE \hspace{0.15in} Projection: $W(t+1)\leftarrow \argmin_{W\in S_B}\|W-\tilde{W}(t+1)\|_2$
\STATE \hspace{0.15in} Averaging: $\overline{W}\leftarrow \frac{t+1}{t+2}\cdot \overline{W}+\frac{1}{t+2}\cdot W(t+1)$
\STATE \textbf{End For} 
\STATE \textbf{Output:}  $\hat{Q}_\text{out}(\cdot)\leftarrow \hat{Q}(\cdot\,;\overline{W})$
\end{algorithmic}\label{softq}
\end{algorithm}

\subsection{Implication for Policy Gradient}
In this section, we briefly summarize the equivalence between policy gradient algorithms and neural soft Q-learning \citep{schulman2017equivalence, haarnoja2018soft}, which implies that our results are extendable to characterize a variant of the policy gradient algorithm. 

We define $\pi_\theta$ as the Boltzmann policy corresponding to the Q-function $Q_\theta$, which is parametrized  by $\theta$,
\# 
&\pi_{\theta}(a\,|\,s)=\overline{\pi}(a\,|\,s)\cdot \exp\Bigl(\beta\cdot\bigl( Q_{\theta}(s,a)-V_\theta(s)\bigr)\Bigr), \label{430734}\\
&\text{where~} V_\theta(s)=\beta^{-1}\cdot\log\Bigl(\EE_{a\sim\overline{\pi}(\cdot\,|\,s) }\bigl[ \exp\bigl(\beta\cdot Q_{\theta}(s,a)\bigr) \bigr]\Bigr). \label{430735}
\#
Here $\overline{\pi}$ is the uniform policy and $V_\theta$ is the partition function. In the context of neural soft Q-learning, we use the parametrization
\$
Q_{\theta}(s,a) = \hat{Q}(x;W)=\frac{1}{\sqrt{m}}\sum_{r=1}^m b_r\sigma(W_r^\top x).
\$
From \eqref{430734} we have
\# \label{430505}
Q_\theta(s,a)=V_{\theta}(s)+\beta^{-1}\cdot \log\bigl( \pi_{\theta}(a\,|\,s)/\overline{\pi}(a\,|\,s)\bigr).
\#
In the sequel, we show that the population semigradient in soft Q-learning, which is defined in \eqref{0517538}, equals a variant of population policy gradient, given that the exploration policy $\pie$ in \eqref{0517538} is $\pi_\theta$. Recall that the population semigradient in soft Q-learning is given by
\# \label{430506}
&\EE_{(s,a,s')\sim\mu_{\theta}}\bigl[ \nabla_\theta Q_\theta(s,a)\cdot
\bigl(
Q_\theta(s,a)-r(s,a)-\gamma\cdot\softmax_{a'\in\cA} Q_\theta(s',a')
\bigr)
 \bigr]\notag\\
 &\quad=\EE_{(s,a,s')\sim\mu_{\theta}}\bigl[ \nabla_\theta Q_\theta(s,a) \cdot 
\bigl(
Q_\theta(s,a)-r(s,a)
-\gamma\cdot V_\theta(s')
\bigr)\bigr],
\#
where $\mu_\theta$ is the stationary distribution of $\pi_\theta$. For notational simplicity, we define
\# \label{430504}
\xi=r(s,a)-\beta^{-1}\cdot D_{\text{KL}}\bigl(\pi_\theta(\cdot\,|\,s)\,\|\,\overline{\pi}(\cdot\,|\,s)\bigr)+\gamma\cdot V_\theta(s')-V_\theta(s).
\#
Plugging \eqref{430505} and \eqref{430504} into \eqref{430506}, we obtain
\# \label{430512}
&\EE_{(s,a,s')\sim\mu_{\theta}}\bigl[ \nabla_\theta Q_\theta(s,a) \cdot 
\bigl(
Q_\theta(s,a)-r(s,a)
-\gamma\cdot V_\theta(s')
\bigr)\bigr]\\
&\quad=\EE_{(s,a,s')\sim\mu_{\theta}}\Bigl[ \nabla_\theta Q_\theta(s,a)\cdot
\Bigl(\beta^{-1}\cdot \log\bigl( \pi_{\theta}(a\,|\,s)/\overline{\pi}(a\,|\,s)\bigr)
-\beta^{-1}\cdot D_{\text{KL}}\bigl(\pi_\theta(\cdot\,|\,s)\,\|\,\overline{\pi}(\cdot\,|\,s)\bigr)-\xi
\Bigr)\Bigr].\notag
\#
Taking gradient on the both sides of \eqref{430505}, we obtain
\# \label{430513}
\nabla_\theta Q_\theta(s,a)=\nabla_\theta V_\theta(s)+\beta^{-1} \cdot \nabla_\theta \log \bigl( \pi_\theta(a\,|\,s) \bigr).
\#
Plugging \eqref{430513} into the right-hand side of \eqref{430512} yields
\#\label{430516}
&\EE_{(s,a,s')\sim\mu_{\theta}}\bigl[ \nabla_\theta Q_\theta(s,a) \cdot 
\bigl(
Q_\theta(s,a)-r(s,a)
-\gamma\cdot V_\theta(s')
\bigr)\bigr]\\
&\quad =\EE_{(s,a,s')\sim\mu_{\theta}}\Bigl[
\beta^{-1}\cdot\nabla_\theta V_\theta(s) \cdot \Bigl(\log\bigl( \pi_{\theta}(a\,|\,s)/\overline{\pi}(a\,|\,s)\bigr)
-D_{\text{KL}}\bigl(\pi_\theta(\cdot\,|\,s)\,\|\,\overline{\pi}(\cdot\,|\,s)\bigr)
\Bigr)-\nabla_\theta V_\theta(s)\cdot\xi \notag \\
&\quad\qquad 
+\beta^{-2}\cdot  \nabla_\theta \log \bigl( \pi_\theta(a\,|\,s) \bigr)\cdot \log\bigl( \pi_{\theta}(a\,|\,s)/\overline{\pi}(a\,|\,s)\bigr)
-\beta^{-2}\cdot \nabla_\theta\log \bigl( \pi_\theta(a\,|\,s) \bigr) \cdot D_{\text{KL}}\bigl(\pi_\theta(\cdot\,|\,s)\,\|\,\overline{\pi}(\cdot\,|\,s)\bigr)\notag\\
&\quad\qquad -\beta^{-1}\cdot \nabla_\theta\log \bigl( \pi_\theta(a\,|\,s) \bigr) \cdot \xi
\Bigr].\notag
\#
By the definition of the KL-divergence, we have
\$
\EE_{a\sim\pi_\theta(\cdot\,|\,s)}\bigl[ \log\bigl( \pi_{\theta}(a\,|\,s)/\overline{\pi}(a\,|\,s)\bigr)\bigr]=D_{\text{KL}}\bigl(\pi_\theta(\cdot\,|\,s)\,\|\,\overline{\pi}(\cdot\,|\,s)\bigr).
\$
Thus, on the right-hand side of \eqref{430516}, we have 
\#\label{518236}
\EE_{(s,a,s')\sim\mu_{\theta}}\Bigl[
\beta^{-1}\cdot\nabla_\theta V_\theta(s)\Bigl(\log\bigl( \pi_{\theta}(a\,|\,s)/\overline{\pi}(a\,|\,s)\bigr)
-D_{\text{KL}}\bigl(\pi_\theta(\cdot\,|\,s)\,\|\,\overline{\pi}(\cdot\,|\,s)\bigr)
\Bigr)\Bigr]=0.
\#
Also, since 
\#\label{430822}
\EE_{a\sim\pi_\theta(\cdot\,|\,s)}\bigl[ \nabla_\theta\log \bigl( \pi_\theta(a\,|\,s) \bigr)\bigr]=\sum_{a\in\cA} \nabla_\theta \pi_\theta(a\,|\,s)=\nabla_\theta 1= 0,
\#
on the right-hand side of \eqref{430516}, we have
\#\label{518237}
\EE_{(s,a,s')\sim\mu_{\theta}}\Bigl[\beta^{-2}\cdot \nabla_\theta\log \bigl( \pi_\theta(a\,|\,s) \bigr) \cdot D_{\text{KL}}\bigl(\pi_\theta(\cdot\,|\,s)\,\|\,\overline{\pi}(\cdot\,|\,s)\bigr)
\Bigr]=0.
\#
Plugging \eqref{518236} and \eqref{518237} into \eqref{430516}, we obtain 
\# \label{430731}
&\EE_{(s,a,s')\sim\mu_{\theta}}\bigl[ \nabla_\theta Q_\theta(s,a) \cdot 
\bigl(
Q_\theta(s,a)-r(s,a)
-\gamma\cdot V_\theta(s')
\bigr)\bigr]\\
&\quad = \underbrace{\EE_{(s,a,s')\sim\mu_{\theta}}\bigl[
\beta^{-2}\cdot  \nabla_\theta \log \bigl( \pi_\theta(a\,|\,s) \bigr) \cdot \log\bigl( \pi_{\theta}(a\,|\,s)/\overline{\pi}(a\,|\,s)\bigr)
 -\beta^{-1}\cdot \nabla_\theta\log \bigl( \pi_\theta(a\,|\,s) \bigr) \cdot \xi \bigr]}_{\displaystyle \text{(ii)}}\notag\\
 &\quad\qquad+
\underbrace{\EE_{(s,a,s')\sim\mu_{\theta}}[-\nabla_\theta V_\theta(s)\cdot\xi
]}_{\displaystyle \text{(i)}}.\notag
\#
We characterize (i) and (ii) on the right-hand side of \eqref{430731}. For (i), by the definition of $\xi$ in \eqref{430504}, we have
\#\label{518250}
&\EE_{(s,a,s')\sim\mu_{\theta}}[-\nabla_\theta V_\theta(s)\cdot\xi
]\notag\\
&\quad=\EE_{(s,a,s')\sim\mu_{\theta}}\Bigl[- \nabla_\theta V_\theta(s) \cdot \Bigl(
r(s,a)-\beta^{-1}\cdot D_{\text{KL}}\bigl(\pi_\theta(\cdot\,|\,s)\,\|\,\overline{\pi}(\cdot\,|\,s)\bigr)+\gamma\cdot V_\theta(s')-V_\theta(s)
\Bigr)\Bigr]\notag\\
&\quad=\EE_{(s,a,s')\sim\mu_{\theta}}\bigl[-\nabla_\theta 
\bigl( V_\theta(s)-\cT^{\tilde{\pi}}_{\text{KL}} V_{\tilde{\theta}}(s) \bigr)^2\big/2
\bigr]\big|_{\tilde{\pi}=\pi_\theta,\tilde{\theta}=\theta}.
\#
Here the operator $\cT^{\tilde{\pi}}_{\text{KL}}$ is defined as
\$
\cT^{\tilde{\pi}}_{\text{KL}} V(s)=\EE
\bigl[
r(s,a)-\beta^{-1}\cdot D_{\text{KL}}\bigl(\tilde{\pi}(\cdot\,|\,s)\,\|\,\overline{\pi}(\cdot\,|\,s)\bigr)
+\gamma\cdot V(s')
\,\big|\, a\sim\tilde{\pi}(\cdot\,|\,s),s'\sim\cP(\cdot\,|\,s,a)\bigr],
\$
which is the Bellman evaluation operator for the value function $V^{\tilde{\pi}(s)}$ associated with the KL-regularized reward 
\#\label{sreward}
r_{\text{KL}}(s,a,s')=r(s,a)-\beta^{-1}\cdot D_{\text{KL}}\bigl(\tilde{\pi}(\cdot\,|\,s)\,\|\,\overline{\pi}(\cdot\,|\,s)\bigr).\#
By \eqref{518250}, 
  $\EE_{(s,a,s')\sim\mu_{\theta}}[-\nabla_\theta V_\theta(s)\cdot\xi
]$ is the population semigradient for the evaluation of policy $\pi_\theta$. Now we characterize (ii) in \eqref{430731}. First, by the definition of the KL-divergence, we have
\$
\nabla_\theta D_{\text{KL}}\bigl( \pi_\theta(\cdot\,|\,s)\,\|\,\overline{\pi}(\cdot\,|\,s) \bigr)
&=\sum_{a\in\cA}\nabla_\theta \bigl[ \pi_\theta(a\,|\,s) \cdot \log\bigl(\pi_\theta(a\,|\,s)/\overline{\pi}(a\,|\,s) \bigr)
\bigr]\\
&=\sum_{a\in\cA} \nabla_\theta\pi_\theta(a\,|\,s) \cdot \log\bigl(\pi_\theta(a\,|\,s)/\overline{\pi}(a\,|\,s) \bigr)\\
&=\EE_{a\sim\pi_{\theta}(\cdot\,|\,s)}\bigl[ \nabla_{\theta} \log\pi_\theta(a\,|\,s) \cdot \log\bigl(\pi_\theta(a\,|\,s)/\overline{\pi}(a\,|\,s) \bigr) \bigr].
\$
Here the second equality follows from \eqref{430822}. Hence, (ii) in \eqref{430731} takes the form
\$
&\EE_{(s,a,s')\sim\mu_{\theta}}\bigl[
\beta^{-2}\cdot  \nabla_\theta \log \bigl( \pi_\theta(a\,|\,s) \bigr) \cdot \log\bigl( \pi_{\theta}(a\,|\,s)/\overline{\pi}(a\,|\,s)\bigr)
 -\beta^{-1}\cdot \nabla_\theta\log \bigl( \pi_\theta(a\,|\,s) \bigr) \cdot \xi \bigr]\\
 &\quad=
 -\beta^{-1} \cdot \EE_{(s,a,s')\sim\mu_{\theta}}\Bigl[
 \underbrace{\nabla_\theta\log \bigl( \pi_\theta(a\,|\,s) \bigr) \cdot \xi}_{\displaystyle\text{(ii).a}}
 - \underbrace{ \nabla_\theta \Bigl(\beta^{-1}\cdot D_{\text{KL}}\bigl( \pi_\theta(\cdot\,|\,s)\,\|\,\overline{\pi}(\cdot\,|\,s) \bigr)\Bigr)}_{\displaystyle\text{(ii).b}}
 \Bigr].
\$
We show that (ii) in \eqref{430731} is the population policy gradient. For (ii).a, note that $\xi$ defined in \eqref{430504} is an unbiased estimator of the advantage function $Q^{\pi_\theta}(s,a)-V^{\pi_\theta}(s)$ associated with the KL-regularized reward $r_{\text{KL}}$ defined in \eqref{sreward}. If we denote by $J_{\text{KL}}(\pi_\theta)$ the expected total reward, then (ii).a is an estimator of the population policy gradient $\nabla_\theta J_{\text{KL}}(\pi_\theta)$. For (ii).b, it is the gradient of the entropy regularization
\$
\beta^{-1} \cdot \EE_{s\sim\mu_{\theta}}\bigl[ D_{\text{KL}}\bigl( \pi_\theta(\cdot\,|\,s)\,\|\,\overline{\pi}(\cdot\,|\,s) \bigr)\bigr] = \beta^{-1} \cdot \EE_{s\sim\mu_{\theta}} \bigl[H\bigl(\pi_\theta(\cdot\,|\,s)\bigr)\bigr].
\$
Therefore, we recover the policy gradient update in the Q-learning updating scheme. Combining (i) and (ii) in \eqref{430731}, we obtain a variant of the policy gradient algorithm, which is connected with the soft actor-critic algorithm \citep{haarnoja2018soft}. Hence, our global convergence of neural soft Q-learning extends to a variant of the actor-critic algorithm. See Algorithm \ref{softpg} for a detailed description of such an algorithm with $\hat{Q}_\theta$ parametrized by a two-layer neural network, which can also be extended to allow for a multi-layer neural network. In parallel with $V_{\theta}(s)$ in \eqref{430735}, we define
\$
\hat{V}(s;W)=\beta^{-1}\cdot\log\Bigl(\EE_{a\sim\overline{\pi}(\cdot\,|\,s) }\bigl[ \exp\bigl(\beta\cdot \hat{Q}(s,a;W)\bigr) \bigr]\Bigr).
\$

\begin{algorithm}
\caption{Neural Soft Actor-Critic}
\begin{algorithmic}[1]
\STATE \textbf{Initialization:} $b_r\sim\text{Unif}(\{-1,1\})$, $W_r(0)\sim N(0,I_d/d)$ $(r\in[m])$, $\overline{W}=W(0)$,
 \label{pgstep2}\\
{\color{white}\textbf{Initialization:}} $S_B=\{W\in\real^{md}:\|W-W(0)\|_2\le B\}$ $(B>0)$\\
\STATE \textbf{For} {$t=0$ to $T-2$}:
\STATE \hspace{0.15in} Policy Update: $\pi_t(a\,|\,s)\propto \overline{\pi}(a\,|\,s)\cdot \exp(\beta\cdot\hat{Q}(s,a;W(t)))$
\STATE \hspace{0.15in} Sample a tuple $(s,a, r, s')$ from the stationary distribution $\mu_t$ of policy $\pi_t$
\STATE \hspace{0.15in} Reward regularization: $r_{\text{KL}}\leftarrow r-\beta^{-1}\cdot D_{\text{KL}}(\pi_t(\cdot\,|\,s)\,\|\,\overline{\pi}(\cdot\,|\,s))$
\STATE \hspace{0.15in} Bellman residual calculation: $\xi \leftarrow r_{\text{KL}}+\gamma \hat{V}(s';W(t))-\hat{V}(s;W(t))$
\STATE \hspace{0.15in} Actor Update:\\
\STATE \hspace{0.15in} \quad\qquad
 $\tilde{W}(t+1)\leftarrow W(t)+\eta\cdot(\beta^{-1}\cdot\xi\cdot\nabla_{W}\log\pi_t(a\,|\,s)-\beta^{-2}\cdot\nabla_WD_{\text{KL}}(\pi_t(\cdot\,|\,s)\,\|\,\overline{\pi}(\cdot\,|\,s)))$\\
\STATE \hspace{0.15in} Critic update: $\tilde{W}'(t+1)\leftarrow\tilde{W}(t+1)+\eta\cdot\xi\cdot\nabla_W \hat{V}(s;W(t)) $
\STATE \hspace{0.15in} Projection: $W(t+1)\leftarrow \argmin_{W\in S_B}\|W-\tilde{W}'(t+1)\|_2$
\STATE \hspace{0.15in} Averaging: $\overline{W}\leftarrow \frac{t+1}{t+2}\cdot \overline{W}+\frac{1}{t+2}\cdot W(t+1)$
\STATE \textbf{End For} 
\STATE \textbf{Output:}  $\hat{Q}_\text{out}(\cdot)\leftarrow \hat{Q}(\cdot\,;\overline{W})$, $\pi_{\text{out}}(a\,|\,s)\propto \overline{\pi}(a\,|\,s)\cdot \exp(\beta\cdot\hat{Q}_\text{out}(s,a))$
\end{algorithmic}\label{softpg}
\end{algorithm}

%%%%%%%%%%%%%%%%%%%%%%%%%%%%%%%%%%%%%%%%%%%%%%%%%%

% !TEX root = neural_TD.tex
%\begin{flushleft}
\section{Extension to Multi-Layer Neural Networks}\label{deep}
In this section, we generalize our main results in Section \ref{tr} to the setting where the Q-function is parametrized by a multi-layer neural network. Similar to the setting with a two-layer neural network, we represent the state-action pair $(s,a)\in\cS\times\cA$ by a vector $x=\psi(s,a)\in\cX\subseteq\real^d$ with $d>2$, where $\psi$ is a given one-to-one feature map. With a slight abuse of notation, we use $(s,a)$ and $x$ interchangeably. Without loss of generality, we assume that $\|x\|_2=1$. The Q-function is parametrized by
\$
x^{(0)}&=Ax,\quad x^{(h)}=\frac{1}{\sqrt{m}}\cdot\sigma(W^{(h)}x^{(h-1)})~\text{for any}~h\in[H],\quad y=b^\top x^{(H)},
\$
where $A\in\RR^{m\times d}$, $W^{(h)}\in\RR^{m\times m}$, and $b\in\RR^m$ are the weights. Here $x^{(h)}$ corresponds to the $(h+1)$-th hidden layer and $y$ gives $\hat{Q}(x;W)$. For notational simplicity, we define
\$
W = \bigl(\vec(W^{(1)})^\top, \ldots, \vec(W^{(H)})^\top \bigr)^\top \in \RR^{Hm^2}.
\$
 Each entry of $A$ and $\{W^{(h)}\}_{h=1}^H$ is independently initialized by $N(0,2)$, while each entry of $b$ is independently initialized by $N(0,1)$. During training, we only update $W$ using the TD update in \eqref{5131033}, while keeping $A$ and $b$ fixed as the random initialization \citep{allen2018convergence, gao2019convergence}. 

Similar to \eqref{q0}, we redefine the locally linearized Q-function as
\# \label{200305624}
\hat{Q}_0(x;W) = \hat{Q}\bigl(x;W(0)\bigr) + \bigl\la \nabla_W \hat{Q}\bigl(x;W(0)\bigr), W-W(0) \bigr\ra,
\#
where $W(0)$ is the random initialization of $W$. Also, we redefine
\$
S_B&=\bigl\{ W\in\RR^{Hm^2}: \|W^{(h)}-W^{(h)}(0)\|_{\tf}\le B~\text{for any}~h\in[H] \bigr\}, \\
\cF_{B,m} &= \bigl\{ \hat{Q}\bigl(x;W(0)\bigr) + \bigl\la \nabla_W \hat{Q}\bigl(x;W(0)\bigr), W-W(0) \bigr\ra : W\in S_B   \bigr\}.
\$
Correspondingly, we redefine $g(t)$, $\overline{g}(t)$, $g_0(t)$, and $\overline{g}_0(t)$ by plugging the redefined $\hat{Q}$ and $\hat{Q}_0$ into \eqref{514210}, \eqref{518430}, and \eqref{518431}, respectively.  Also, we redefine $W^*$ and $\overline{g}^*_0$ by plugging the redefined $S_B$, $\cF_{B,m}$, $\hat{Q}$, and $\hat{Q}_0$ into \eqref{spdef} and \eqref{516720}, respectively.

%With the new definitions of $\hat{Q}$ and $\hat{Q}_0$ above, we have the new $g(t), g_0(t), \overline{g}(t)$, and $\overline{g}_0(t)$, which are defined in the same way. Also, with the new definitions of $S_B$ and $\cF_{B,m}$, there is the corresponding new approximate stationary point $W^*$ and $\overline{g}^*_0$.

In the sequel, we establish the global convergence of neural TD with a multi-layer neural network. Note that we abandon Assumption \ref{asmp1} at the cost of a slightly worse upper bound of the error of local linearization, which is characterized by the following lemma.

\begin{lemma}\label{1223124}
Let $m=\Omega(d^{3/2}\log^{3/2}(m^{1/2}/B)/(BH^{3/2}))$ and $B=O(m^{1/2}H^{-6}\log^{-3} m)$. With probability at least $1-e^{-\Omega(\log^2 m)}$ with respect to the random initialization, it holds for any $W\in S_B$ and $x\in\RR^d$ with $\|x\|_2=1$ that
\#
\bigl\| \nabla_W \hat{Q}(x;W) - \nabla_W \hat{Q}\bigl(x;W(0)\bigr) \bigr\|_2&=O(B^{1/3}m^{-1/6}H^{5/2}\log^{1/2}m),\notag\\
\|\nabla_W \hat{Q}(x;W)\|_2&=O(H), \notag\\ %\label{2003111254} \\
\bigl| \hat{Q}\bigl(x;W(0)\bigr) \bigr| &= O(\log m).\notag
\#
\end{lemma}
\begin{proof}
See \cite{allen2018convergence, gao2019convergence} for a detailed proof. In detail, following the proofs of Lemmas A.5 and A.6 in \cite{gao2019convergence}, we have that the first and second equalities hold with probability at least $1-O(H)\cdot e^{-\Omega(B^{2/3}m^{2/3}H)}$. Also, following the proof of Theorem 1 in \cite{allen2018convergence}, we have that the third equality holds with probability at least $1-e^{-\Omega(\log^2 m)}$, which concludes the proof of Lemma \ref{1223124}.
\end{proof}

The following lemma replaces Lemmas \ref{qdiff}, \ref{gdiff}, and \ref{bdvar}.

\begin{lemma}\label{1223523}
Under the same condition of Lemma \ref{1223124}, with probability at least $1-e^{-\Omega(\log^2 m)}$ with respect to the random initialization, it holds for any $t\in[T]$ and $x\in\RR^d$ with $\|x\|_2= 1$ that
\#
\bigl| \hat{Q}\bigl(x;W(t)\bigr)-\hat{Q}_0\bigl(x;W(t)\bigr)  \bigr| &= O(B^{4/3}m^{-1/6}H^{3}\log^{1/2}m),\label{20030612441}\\
\| \overline{g}(t) - \overline{g}_0(t) \|_2&=O(B^{4/3}m^{-1/6}H^{4}\log^{3/2} m),\label{20030612442}\\
\| g(t) - \overline{g}(t) \|_2^2&=O(B^2H^5\log^2m).\label{20030612443}
\#
\end{lemma}
\begin{proof}
We prove the three equalities one by one.
\vskip4pt
\noindent
\textbf{Proof of \eqref{20030612441}:}
Note that $W\in S_B$ implies $\|W-W(0)\|_2\le B\sqrt{H}$. Hence, we have
\$
&\hat{Q}\bigl(x;W(t)\bigr)-\hat{Q}_0\bigl(x;W(t)\bigr) \\
&\quad = \hat{Q}\bigl(x;W(0)\bigr) + \int_{0}^1 \bigl\la \nabla_W \hat{Q}\bigl(x;(1-s)W(0)+sW(t)\bigr),W(t)-W(0) \bigr\ra ds \\
&\quad\qquad - \hat{Q}\bigl(x;W(0)\bigr) - \bigl\la \nabla_W \hat{Q}\bigl(x;W(0)\bigr), W(t)-W(0) \bigr\ra \\
&\quad=
\int_{0}^1 \bigl\la \nabla_W \hat{Q}\bigl(x;(1-s)W(0)+sW(t)\bigr)-\nabla_W \hat{Q}\bigl(x;W(0)\bigr),W(t)-W(0) \bigr\ra ds \\
&\quad \le \int_{0}^1 \bigl\| \nabla_W \hat{Q}\bigl(x;(1-s)W(0)+sW(t)\bigr)-\nabla_W \hat{Q}\bigl(x;W(0)\bigr) \bigr\|_2 \cdot \| W(t)-W(0) \|_2 ds.
\$
Applying Lemma \ref{1223124} and the fact that $(1-s)W(0)+sW(t)\in S_B$ for any $s\in[0,1]$, we obtain \eqref{20030612441}.

\vspace{4pt}
\noindent
\textbf{Proof of \eqref{20030612442}:}
Similar to \eqref{518425}, we have
\#\label{1223555}
&\| \overline{g}(t) - \overline{g}_0(t) \|^2_2 \notag\\
&\quad\le 2\Bigl\| \EE_\mu\Bigl[\Bigl(\delta\bigl(x,r,x';W(t)\bigr)-\delta_0\bigl(x,r,x';W(t)\bigr)\Bigr)\cdot \nabla_W\hat{Q}\bigl(x;W(t)\bigr) \Bigr]  \Bigr\|_2^2 \notag \\
&\quad\qquad  + 2\Bigl\| \EE_\mu\Bigl[\delta_0\bigl(x,r,x';W(t)\bigr)\cdot \Bigl(\nabla_W\hat{Q}\bigl(x;W(t)\bigr)-\nabla_W\hat{Q}\bigl(x;W(0)\bigr) \Bigr)\Bigr]  \Bigr\|_2^2 \notag\\
&\quad\le 2\underbrace{\EE_\mu\Bigl[\bigl|\delta\bigl(x,r,x';W(t)\bigr)-\delta_0\bigl(x,r,x';W(t)\bigr)\bigr|^2\Bigr]}_{\displaystyle{\text{(i)}}}
\cdot \underbrace{\EE_\mu\Bigl[ \bigl\|\nabla_W\hat{Q}\bigl(x;W(t)\bigr)   \bigr\|_2^2\Bigr]}_{\displaystyle{\text{(ii)}}} \\
&\quad\qquad  + 2\underbrace{\EE_\mu\Bigl[ \bigl|\delta_0\bigl(x,r,x';W(t)\bigr)\bigr|^2\Bigr]}_{\displaystyle{\text{(iii)}}}
\cdot \underbrace{ \EE_\mu\Bigl[\bigl\|\nabla_W\hat{Q}\bigl(x;W(t)\bigr)-\nabla_W\hat{Q}\bigl(x;W(0)\bigr) \bigr\|_2^2\Bigr] }_{\displaystyle{\text{(iv)}}}. \notag
\#
To upper bound (i), recall the definitions of $\delta(x,r,x';W)$ and $\delta_0(x,r,x';W)$, 
\$
\delta\bigl(x,r,x';W(t)\bigr)&=\hat{Q}\bigl(x;W(t)\bigr)-r-\gamma\hat{Q}\bigl(x';W(t)\bigr),\\
\delta_0\bigl(x,r,x';W(t)\bigr)&=\hat{Q}_0\bigl(x;W(t)\bigr)-r-\gamma\hat{Q}_0\bigl(x';W(t)\bigr).
\$
Following from \eqref{20030612441}, which is proved previously, we have
\$
&\bigl|\delta\bigl(x,r,x';W(t)\bigr)-\delta_0\bigl(x,r,x';W(t)\bigr)\bigr| \\
&\quad\le \bigl|\hat{Q}\bigl(x;W(t)\bigr)-\hat{Q}_0\bigl(x;W(t)\bigr)\bigr|+\bigl|\hat{Q}\bigl(x';W(t)\bigr)-\hat{Q}_0\bigl(x';W(t)\bigr)\bigr| \\
&\quad = O(B^{4/3}m^{-1/6}H^{3}\log^{1/2} m),
\$
which implies
\#\label{12235531}
\EE_\mu\Bigl[ \bigl|\delta\bigl(x,r,x';W(t)\bigr)-\delta_0\bigl(x,r,x';W(t)\bigr)\bigr|^2\Bigr] = O(B^{8/3}m^{-1/3}H^{6}\log m).
\#
To upper bound (ii), by Lemma \ref{1223124} we have
\#\label{12235532}
\EE_\mu\Bigl[\bigl\| \nabla_W\hat{Q}\bigl(x;W(t)\bigr) \bigr\|_2^2 \Bigr] = O(H^2). 
\#
To upper bound (iii), by the triangle inequality we have
\#\label{1223511}
\bigl|\delta_0\bigl(x,r,x';W(t)\bigr)\bigr| \le
\bigl|\hat{Q}_0\bigl(x;W(t)\bigr)\bigr| + |r| + \bigl|\gamma\hat{Q}_0\bigl(x';W(t)\bigr)\bigr|.
\#
By \eqref{200305624}, we have
\#\label{1223512}
\bigl|\hat{Q}_0\bigl(x;W(t)\bigr)\bigr| &\le \bigl|\hat{Q}_0\bigl(x;W(0)\bigr)\bigr| + \big\| \nabla_W \hat{Q}_0\bigl(x, W(0)\bigr) \bigr\|_2 \cdot \| W(t) - W(0) \|_2 \notag \\
&=O(\log m)+O(H)\cdot O(BH^{1/2})=O(BH^{3/2}\log m).
\#
Also, the same upper bound holds for $|\hat{Q}_0(x';W(t))|$. Thus, from \eqref{1223511} we obtain
\#\label{1223513}
\EE_{\mu}\Bigl[ \bigl|\delta_0\bigl(x,r,x';W(t)\bigr)\bigr|^2\Bigr] = O(B^2H^3\log^2m).
\#
To upper bound (iv), by Lemma \ref{1223124} we have
\$
\bigl\|\nabla_W\hat{Q}\bigl(x;W(t)\bigr)-\nabla_W\hat{Q}\bigl(x;W(0)\bigr) \bigr\|_2 = O(B^{1/3}m^{-1/6}H^{5/2}\log^{1/2} m),
\$
which implies
\#\label{12235534}
\EE_\mu\Bigl[\bigl\|\nabla_W\hat{Q}\bigl(x;W(t)\bigr)-\nabla_W\hat{Q}\bigl(x;W(0)\bigr) \bigr\|_2^2\Bigr]=O(B^{2/3}m^{-1/3}H^{5}\log m).
\#
Plugging \eqref{12235531}, \eqref{12235532}, \eqref{1223513}, and \eqref{12235534} into \eqref{1223555}, we obtain
\$
\| \overline{g}(t) - \overline{g}_0(t) \|^2_2&=O(B^{8/3}m^{-1/3}H^{8}\log m) + O(B^{8/3}m^{-1/3}H^{8}\log^3 m)\\
&=O(B^{8/3}m^{-1/3}H^{8}\log^3 m).
\$
\vspace{4pt}
\noindent
\textbf{Proof of \eqref{20030612443}:}
By the redefinition of $g(t)$, we have
\#\label{200306115}
\|g(t)\|_2^2&= \bigl\| \delta\bigl(x,r,x';W(t)\bigr)\cdot \nabla_W \hat{Q}\bigl(x;W(t)\bigr)\bigr\|_2^2\notag \\
&=\bigl|\delta\bigl(x,r,x';W(t)\bigr)\bigr|^2 \cdot \bigl\|  \nabla_W \hat{Q}\bigl(x;W(t)\bigr) \bigr\|_2^2  .
\#
By \eqref{1223511} and \eqref{1223512}, we have
\#\label{200306114}
\bigl|\delta\bigl(x,r,x';W(t)\bigr)\bigr|^2 =O(B^2H^3\log^2m).
\#
Meanwhile, by Lemma \ref{1223124} we have
\#\label{200306113}
\bigl\|  \nabla_W \hat{Q}\bigl(x;W(t)\bigr) \bigr\|_2^2 = O(H^2).
\#
Combining \eqref{200306115}, \eqref{200306114}, and \eqref{200306113}, we obtain $\|g(t)\|_2^2 = O(B^2H^5\log^2m)$. Also, by the definition of $\overline{g}(t)$ and Jensen's inequality, we have
\$
\|\overline{g}(t)\|_2^2 \le \EE_\mu\bigl[\|g(t)\|_2^2\bigr]=O(B^2H^5\log^2m).
\$
Thus, by the triangle inequality, we obtain $\| g(t) - \overline{g}(t) \|_2^2=O(B^2H^5\log^2m)$.

Therefore, we conclude the proof of Lemma \ref{1223523}.
\end{proof}

Now we present the global convergence of neural TD with a multi-layer neural network.

\begin{theorem}\label{1223956}
Let $m=\Omega(d^{3/2}\log^{3/2}(m^{1/2}/B)/(BH^{3/2}))$ and $B=O(m^{1/2}H^{-6}\log^{-3} m)$. We set $\eta=1/\sqrt{T}$ and $H=O(T^{1/4})$ in Algorithm \ref{td0}. With probability at least $1-e^{-\Omega(\log^2 m)}$ with respect to the random initialization,  the output $\hat{Q}_{\text{out}}$ of Algorithm \ref{td0} satisfies
\$
\EE_{\overline{W},\mu}\bigl[
\bigl(\hat{Q}_{\text{out}}(x)-\hat{Q}_0(x;W^*)\bigr)^2 \bigr]=O(B^2H/\sqrt{T}+B^{8/3}m^{-1/6} H^8  \log^3 m),
\$
where the expectation is taken with respect to the randomness of $\overline{W}$ in Algorithm \ref{td0} and $x \sim \mu$.
\end{theorem}
\begin{proof}
We first reestablish Lemma \ref{sdl}. Similar to \eqref{1223731}, for any $W(t)\ (t\in[T])$, the convexity of $S_B$ implies
\#\label{1225406}
\|W(t+1)-W^*\|^2_2&=\bigl\|\Pi_{S_B}\bigl(W(t)-\eta\cdot g(t)\bigr)-\Pi_{S_B}(W^*-\eta\cdot\overline{g}^*_0)\bigr\|_2^2\\
&\le \bigl\|\bigl(W(t)-\eta\cdot g(t)\bigr)-(W^*-\eta\cdot\overline{g}^*_0)\bigr\|_2^2\notag\\
&=\|W(t)-W^*\|_2^2-2\eta\cdot\bigl(g(t)-\overline{g}_0^*\bigr)^\top \bigl(W(t)-W^*\bigr)+\eta^2\cdot\|g(t)-\overline{g}_0^*\|_2^2.\notag
\#
Taking expectation on both sides with respect to the tuple $(x,r,x')\sim\mu$ conditional on $W(t)$, we obtain
\#\label{5161100deep}
&\EE_\mu\bigl[\|W(t+1)-W^*\|_2^2\,\big|\,W(t)\bigr]\notag\\
&\quad \le\|W(t)-W^*\|_2^2-2\eta\cdot\bigl(\overline{g}(t)-\overline{g}_0^*\bigr)^\top \bigl(W(t)-W^*\bigr) +\eta^2\cdot \EE_\mu\bigl[\|g(t)-\overline{g}_0^*\|^2_2\,\big|\,W(t)\bigr].
\#
For the inner product $(\overline{g}(t)-\overline{g}_0^*)^\top (W(t)-W^*)$ on the right-hand side of \eqref{5161100deep}, following the same proof of \eqref{5161023} and \eqref{57137}, we have
\#\label{0225104}
&\bigl(\overline{g}(t)-\overline{g}_0^*\bigr)^\top \bigl(W(t)-W^*\bigr) \\
&\quad=\bigl(\overline{g}_0(t)-\overline{g}_0^*\bigr)^\top \bigl(W(t)-W^*\bigr)+ \bigl(\overline{g}(t)-\overline{g}_0\bigr)^\top \bigl(W(t)-W^*\bigr)\notag\\
&\quad\ge (1-\gamma)\cdot \EE_\mu\Bigl[
\Bigl(\hat{Q}_0\bigl(x;W(t)\bigr)-\hat{Q}_0(x;W^*)\Bigr)^2\,\Big|\, W(t)\Bigr] -O(BH^{1/2}) \cdot \|\overline{g}(t)-\overline{g}_0(t)\|_2. \notag
\#
Meanwhile, for $\EE_\mu[\|g(t)-\overline{g}_0^*\|^2_2\,|\,W(t)]$ on the right-hand side of \eqref{5161100deep}, we have the decomposition
\#\label{0225105}
\EE_\mu\bigl[
\|g(t)-\overline{g}_0^*\|^2_2\,\big|\,W(t)
\bigr] &=
\|\overline{g}(t)-\overline{g}_0^*\|^2_2+
\EE_{\mu}\bigl[
\|g(t)-\overline{g}(t)\|^2_2
\,\big|\,W(t) \bigr] \notag\\
&\le O(H^2)\cdot\EE_\mu\Bigl[
\Bigl(\hat{Q}_0\bigl(x;W(t)\bigr)-\hat{Q}_0(x;W^*)\Bigr)^2\,\Big|\,W(t)\Bigr] \\
&\qquad  + 2\|\overline{g}(t)-\overline{g}_0(t)\|_2^2  + \EE_{\mu}\bigl[
\|g(t)-\overline{g}(t)\|^2_2
\,\big|\,W(t) \bigr]. \notag
\#
Here the inequality follows from \eqref{57138} and \eqref{57139}, where we plug $\| \nabla_W \hat{Q}(x;W)\|_2=O(H)$ into \eqref{57139} instead of $\| \nabla_W \hat{Q}(x;W)\|_2\le 1$. Combining \eqref{0225104} with \eqref{0225105} and taking expectation on the both sides of \eqref{5161100deep}, we obtain the following inequality, which corresponds to Lemma \ref{sdl},
\#\label{1223803}
&\EE_{W(t+1)}\bigl[\|W(t+1)-W^*\|_2^2\bigr]\\
&\quad\le \EE_{W(t)}\bigl[\|W(t)-W^*\|_2^2\bigr]-\bigl(2\eta(1-\gamma)-O(\eta^2 H^2)\bigr)\cdot\EE_{W(t), \mu}\Bigl[
\Bigl(\hat{Q}_0\bigl(x;W(t)\bigr)-\hat{Q}_0(x;W^*)\Bigr)^2\Bigr] \notag\\
&\quad\qquad +\underbrace{\EE_{W(t)}\bigl[2\eta^2\cdot\|\overline{g}(t)-\overline{g}_0(t)\|^2_2+O(\eta BH^{1/2})\cdot\|\overline{g}(t)-\overline{g}_0(t)\|_2 \bigr]}_{\displaystyle \text{Error of Local Linearization}}+\underbrace{\EE_{W(t), \mu} \bigl[ \eta^2\cdot \|g(t)-\overline{g}(t)\|^2_2\bigr]}_{\displaystyle \text{Variance of Semigradient}}. \notag
\#
Here the expectation on the left-hand side is taken with respect to the randomness of $W(t+1)$, which is determined by $W(t)$ and the tuple $(x,r,x')\sim\mu$ drawn at the current iteration. Applying Lemma \ref{1223523} to \eqref{1223803}, we have
\#\label{1223804}
&\bigl(2\eta(1-\gamma)-O(\eta^2 H^2)\bigr)\cdot\EE_{W(t), \mu}\Bigl[
\Bigl(\hat{Q}_0\bigl(x;W(t)\bigr)-\hat{Q}_0(x;W^*)\Bigr)^2\Bigr]  \\
&\quad \le \EE_{W(t)}\bigl[\|W(t)-W^*\|_2^2\bigr] - \EE_{W(t+1)}\bigl[\|W(t+1)-W^*\|_2^2\bigr] \notag\\
&\quad\qquad + O(\eta^2 B^{8/3}m^{-1/3} H^8 \log^3 m + \eta B^{7/3}m^{-1/6}H^{9/2}\log^{3/2}m + \eta^2 B^2H^5\log^2 m )\notag\\
&\quad \le
 \EE_{W(t)}\bigl[\|W(t)-W^*\|_2^2\bigr] - \EE_{W(t+1)}\bigl[\|W(t+1)-W^*\|_2^2\bigr] 
 + O(\eta B^{8/3}m^{-1/6} H^8 \log^3 m).\notag
\#
Let $\eta=1/\sqrt{T}$ and $H=O(T^{1/4})$. We have
\$
2\eta(1-\gamma)-O(\eta^2 H^2)=\Omega(1/\sqrt{T}).
\$
Telescoping \eqref{1223804} for $t=0,1,\ldots, T-1$, we obtain
\#\label{1225429}
&\frac{1}{T}\sum_{t=0}^{T-1}\EE_{W(t), \mu}\Bigl[
\Bigl(\hat{Q}_0\bigl(x;W(t)\bigr)-\hat{Q}_0(x;W^*)\Bigr)^2\Bigr]  \notag\\
&\quad\le O(1/\sqrt{T}) \cdot \|W(0)-W^*\|_2^2 + O(B^{8/3}m^{-1/6} H^8 \log^3 m)\notag\\
&\quad = O(B^2H/\sqrt{T}+B^{8/3} m^{-1/6} H^8  \log^3 m).
\#
Here the expectation is taken conditional on the random initialization of $\hat{Q}$. Following the same proof of Theorem \ref{mainthm2}, that is, using the triangle inequality and the upper bound of $|\hat{Q}(x;W(t)) - \hat{Q}_0(x;W(t))|$ in Lemma \ref{1223523}, we conclude the proof of Theorem \ref{1223956}.
\end{proof}

%\end{flushleft}

%%%%%%%%%%%%%%%%%%%%%%%%%%%%%%%%%%%%%%%%%%%%%%%%%%

% !TEX root = neural_TD.tex
%\begin{flushleft}
\section{Extension to Markov Sampling}\label{secmarkov}

Previously, we assume that the tuples $(x,r,x')$ in Algorithm \ref{td0} are independently sampled from the stationary distribution $\mu$ of policy $\pi$. In this section, we weaken such an assumption by allowing the tuples to be sequentially sampled from the $\beta$-mixing Markov chain induced by policy $\pi$. Our analysis extends Section 8 in \cite{bhandari2018finite}, which focuses on the setting with linear function approximation. In contrast, we stick to the same setting as in Appendix \ref{deep}, where the Q-function is parametrized by a multi-layer neural network. For notational simplicity, we omit the conditioning on the random initialization in all the following expectations.

The following assumption states that the Markov chain of states is $\beta$-mixing.
\begin{assumption}\label{asmpm}
There exist constants $\iota>0$ and $\beta\in(0,1)$ such that
\$
\sup_{s\in\cS} d_{\text{TV}} \bigl(\PP_t(\cdot\,|\,s_0=s), \mu_{\cS} \bigr) \le \iota \cdot \beta^t,
\$
where $\PP_t(\cdot\,|\,s_0=s)$ is the conditional distribution of $s_t$ given $s_0=s$, $\mu_\cS$ is the marginal distribution of $s$ under the stationary distribution $\mu$, and $d_{\text{TV}}$ denotes the total variation distance.
\end{assumption}

In the following, we establish the counterpart of Theorem \ref{1223956} under Markov sampling.
Taking conditional expectation on the both sides of \eqref{1225406}, we have
\#\label{1225417}
&\EE\bigl[\|W(t+1)-W^*\|_2^2\,\big|\,W(t)\bigr]\notag\\
&\quad \le\|W(t)-W^*\|_2^2-2\eta\cdot\bigl(\overline{g}(t)-\overline{g}_0^*\bigr)^\top \bigl(W(t)-W^*\bigr) 
 +\eta^2\cdot \EE\bigl[\|g(t)-\overline{g}_0^*\|^2_2\,\big|\,W(t)\bigr]\\
&\quad\qquad
+ \underbrace{2\eta\cdot\bigl(\overline{g}(t)-\EE[g(t)\,|\,W(t)]\bigr)^\top \bigl(W(t)-W^*\bigr)}_{\displaystyle \text{Markov Sampling Bias}}.\notag
\#
Here recall that $\overline{g}(t)$ is the population semigradient defined with respect to the stationary distribution. Rearranging terms in \eqref{1225417}, we have
\#\label{0225338}
&\bigl(\overline{g}(t)-\overline{g}_0^*\bigr)^\top \bigl(W(t)-W^*\bigr)\\
&\quad \le (2\eta)^{-1}\cdot\Bigl( \|W(t)-W^*\|_2^2 - \EE\bigl[\|W(t+1)-W^*\|_2^2\,\big|\,W(t)\bigr] \Bigr)
 +\eta/2\cdot \EE\bigl[\|g(t)-\overline{g}_0^*\|^2_2\,\big|\,W(t)\bigr]\notag\\
&\quad\qquad
+ \bigl(\overline{g}(t)-\EE[g(t)\,|\,W(t)]\bigr)^\top \bigl(W(t)-W^*\bigr).\notag
\#
Plugging \eqref{0225104} and \eqref{0225105} into \eqref{0225338}, and taking expectation with respect to the current iterate $W (t)$ and the tuple $(x,r,x')$ drawn at the current iteration, similar to \eqref{1223803} we obtain
\$
&\EE\bigl[\|W(t+1)-W^*\|_2^2\bigr]\\
&\quad\le \EE\bigl[\|W(t)-W^*\|_2^2\bigr]-\bigl(2\eta(1-\gamma)-O(H^2\eta^2)\bigr)\cdot\EE_{\mu}\Bigl[
\Bigl(\hat{Q}_0\bigl(x;W(t)\bigr)-\hat{Q}_0(x;W^*)\Bigr)^2\Bigr] \\
&\quad\qquad +\underbrace{\EE\bigl[2\eta^2\cdot\|\overline{g}(t)-\overline{g}_0(t)\|^2_2+2\eta BH^{1/2}\cdot\|\overline{g}(t)-\overline{g}_0(t)\|_2 \bigr]}_{\displaystyle \text{Error of Local Linearization}}+\underbrace{\EE\bigl[\eta^2\cdot\|g(t)-\overline{g}(t)\|^2_2\bigr]}_{\displaystyle \text{Variance of Semigradient}}  \\
&\quad\qquad + \underbrace{2\eta\cdot\EE\bigl[ \bigl(\overline{g}(t)-g(t)\bigr)^\top \bigl(W(t)-W^*\bigr)\bigr]}_{\displaystyle \text{Markov Sampling Bias}}. 
\$
%Here the expectation $\EE[\|g(t)-\overline{g}(t)\|^2_2]$ has the same upper bound as that of $\EE_\mu[\|g(t)-\overline{g}(t)\|^2_2]$ in Lemma \ref{1223523}, since its proof does not rely on the property of distribution $\mu$. 
Thus, similar to \eqref{1225429} we obtain
\#\label{1225933}
&\frac{1}{T}\sum_{t=0}^{T-1}\EE_{\mu}\Bigl[
\Bigl(\hat{Q}_0\bigl(x;W(t)\bigr)-\hat{Q}_0(x;W^*)\Bigr)^2\Bigr]  \\
&\quad = O(B^2HT^{-1/2}+B^{8/3}H^8m^{-1/6}  \log^3 m) + O\Bigl(\frac{1}{T}\sum_{t=0}^{T-1}\EE\bigl[\bigl(\overline{g}(t)-g(t)\bigr)^\top \bigl(W(t)-W^*\bigr)\bigr]\Bigr).\notag
\#
To upper bound the left-hand side of \eqref{1225933}, we upper bound $\EE[(\overline{g}(t)-g(t))^\top(W(t)-W^*)]$ in the following lemma.

\begin{lemma}\label{1225741} Let $m=\Omega(d^{3/2}\log^{3/2}(m^{1/2}/B)/(BH^{3/2}))$ and $B=O(m^{1/2}H^{-6}\log^{-3} m)$. We set $\eta=1/\sqrt{T}$ and $H=O(T^{1/4})$ in Algorithm \ref{td0}, where the tuples $(x_t,r_t,x_{t+1})$ are sampled from a Markov chain satisfying Assumption \ref{asmpm}. With probability at least $1-e^{-\Omega(\log^2 m)}$ with respect to the random initialization, for all $t=0,1,\ldots,T-1$, it holds that
\$
\EE\bigl[\bigl(\overline{g}(t)-g(t)\bigr)^\top \bigl(W(t)-W^*\bigr)\bigr] = O(B^2H^5/\sqrt{T}\cdot\log^2 m\cdot \log T + B^{7/3}m^{-1/6}H^{9/2}\log^{3/2} m).
\$
\end{lemma}
\begin{proof}
For notational simplicity, we define
\$
g(t,W)=\delta(x_t, r_t, x_{t+1};W)\cdot \nabla_W \hat{Q}(x_t,W),\quad &\tilde{g}(W)=\EE_\mu[\delta(x, r, x';W) \cdot\nabla_W \hat{Q}(x;W)],\\
g_0(t,W)=\delta_0(x_t, r_t, x_{t+1};W)\cdot \nabla_W \hat{Q}_0(x_t,W),\quad &\tilde{g}_0(W)=\EE_\mu[\delta_0(x, r, x';W)\cdot \nabla_W \hat{Q}_0(x;W)].
\$
In the following, we prove that the function $\zeta_t(W)=(\tilde{g}(W)-g(t,W))^\top(W-W^*)$ is bounded and approximately Lipschitz continuous. Then Lemma \ref{1225741} is a direct application of Lemmas 10 and 11 in \cite{bhandari2018finite}.

For any $W\in S_B$, by the definition of $\zeta_t(W)$ and the Cauchy-Schwarz inequality, we have
\$
| \zeta_t(W) | &= \bigl|\bigl(\tilde{g}(W)-g(t,W)\bigr)^\top(W-W^*)\bigr| \\
&\le \|\tilde{g}(W)-g(t,W)\|_2\cdot \|W-W^*\|_2
=O(B^2H^3\log m).
\$
Here we use the fact that $\|W-W^*\|\le 2BH^{1/2}$ and $\|\tilde{g}(W)-g(t,W)\|_2=O(BH^{5/2}\log m)$, which follows from the same proof of \eqref{20030612443} in Lemma \ref{1223523}.

Also, for any $W, W'\in S_B$, using the triangle inequality and the Cauchy-Schwarz inequality, we have
\#\label{1225841}
&| \zeta_t(W) - \zeta_t(W') |  \notag\\
&\quad= \bigl|\bigl(\tilde{g}(W)-g(t,W)\bigr)^\top(W-W^*) - \bigl(\overline{g}(W')-g(t,W')\bigr)^\top(W'-W^*)\bigr| \notag\\
&\quad= \bigl|\bigl(\tilde{g}(W)-g(t,W)\bigr)^\top(W-W') \notag\\
&\quad\qquad + \bigl(\tilde{g}(W)-g(t,W) - \overline{g}(W') + g(t,W')\bigr)^\top(W'-W^*)\bigr| \notag\\
&\quad\le \| \tilde{g}(W)-g(t,W)\|_2\cdot \|W-W'\|_2 \\
&\quad\qquad + \bigl( \|\tilde{g}(W)-\overline{g}(W')\|_2 +  \| g(t,W) - g(t,W')\|_2\bigr)\cdot \|W'-W^*\|_2, \notag
\#
where $\| \tilde{g}(W)-g(t,W)\|_2=O(BH^{5/2}\log m)$ following the same proof of \eqref{20030612443} in Lemma \ref{1223523}. Meanwhile, we have
\$
 &\|g(t,W)-g(t,W')\|_2 \notag\\
 &\quad = \| \delta(x_t, r_t, x_{t+1};W)\cdot \nabla_W \hat{Q}(x_t,W) - \delta(x_t, r_t, x_{t+1};W') \cdot\nabla_W \hat{Q}(x_t,W') \|_2 \notag\\
 &\quad \le
 \| \delta(x_t, r_t, x_{t+1};W) - \delta(x_t, r_t, x_{t+1};W')  \|_2 \cdot \| \nabla_W \hat{Q}(x_t,W) \|_2 \notag\\
 &\quad\qquad + \| \delta(x_t, r_t, x_{t+1};W') \|_2 \cdot \| \nabla_W \hat{Q}(x_t,W)  - \nabla_W \hat{Q}(x_t,W')  \|_2 \notag\\
 &\quad = O(H^2)\cdot \|W-W'\|_2 + O(BH^{3/2}\log m) \cdot O(B^{1/3}m^{-1/6}H^{5/2}\log^{1/2} m) \notag\\
 &\quad = O(H^2)\cdot \|W-W'\|_2 + O(B^{4/3}m^{-1/6}H^{4}\log^{3/2} m),
\$
where the first inequality follows from the triangle inequality and the Cauchy-Schwarz inequality, while the second equality is implied by Lemma \ref{1223124}. 
The same upper bound holds for $\|\tilde{g}(W)-\overline{g}(W')\|_2$, plugging which into \eqref{1225841} yields
\$
| \zeta_t(W) - \zeta_t(W') | \le O(BH^{5/2}\log m) \cdot \|W-W'\|_2 + O(B^{7/3}m^{-1/6}H^{9/2}\log^{3/2} m).
\$
Following the proof of the first inequality of Lemma 11 in \cite{bhandari2018finite}, we conclude the proof of Lemma \ref{1225741}. The only difference with Lemma 11 in \cite{bhandari2018finite} is that $\zeta_t(W)$ is only approximately Lipschitz continuous rather than exactly Lipschitz continuous, which incurs the additional term $O(B^{7/3}H^{9/2}m^{-1/6}\log^{3/2} m)$ in Lemma \ref{1225741}.
\end{proof}

Applying Lemma \ref{1225741} to \eqref{1225933}, we obtain the following theorem under Markov sampling.
\begin{theorem}[Convergence of Stochastic Update]\label{12271201}
Let $m=\Omega(d^{3/2}\log^{3/2}(m^{1/2}/B)/(BH^{3/2}))$ and $B=O(m^{1/2}H^{-6}\log^{-3} m)$. We set $\eta=1/\sqrt{T}$ and $H=O(T^{1/4})$ in Algorithm \ref{td0}, where the tuples $(x_t,r_t,x_{t+1})$ are sampled from a Markov chain satisfying Assumption \ref{asmpm}. With probability at least $1-e^{-\Omega(\log^2 m)}$ with respect to the random initialization,  the output $\hat{Q}_{\text{out}}$ of Algorithm \ref{td0} satisfies
\$
\EE_{\overline{W},\mu}\bigl[
\bigl(\hat{Q}_{\text{out}}(x)-\hat{Q}_0(x;W^*)\bigr)^2 \bigr]=O\bigl((B^2H^5/\sqrt{T}+B^{8/3}m^{-1/6}H^8 )\cdot \log^3 m\cdot\log T\bigr).
\$
\end{theorem}
\begin{proof}
By Lemma  \ref{1225741}, we have that the left-hand side of \eqref{1225933} satisfies
\$
&\frac{1}{T}\sum_{t=0}^{T-1}\EE_{\mu}\Bigl[
\Bigl(\hat{Q}_0\bigl(x;W(t)\bigr)-\hat{Q}_0(x;W^*)\Bigr)^2\Bigr]  \\
&\quad=O(B^2H/\sqrt{T}+B^{8/3}m^{-1/6}H^8  \log^3 m+B^2H^5/\sqrt{T}\cdot\log^2 m\cdot\log T + B^{7/3}m^{-1/6}H^{9/2}\log^{3/2} m)\\
&\quad = O\bigl((B^2H^5/\sqrt{T}+B^{8/3}m^{-1/6}H^8 )\cdot \log^3 m\cdot\log T\bigr).
\$
Hence, similar to the proof of Theorem \ref{mainthm2}, using the triangle inequality and the upper bound of $|\hat{Q}(x;W) - \hat{Q}_0(x;W)|$ given by \eqref{20030612441} in Lemma \ref{1223523}, we conclude the proof of Theorem \ref{12271201}.
\end{proof}
%\end{flushleft}

%%%%%%%%%%%%%%%%%%%%%%%%%%%%%%%%%%%%%%%%%%%%%%%%%%

\section{Auxiliary Lemmas}\label{auxiliarylemmas}
Under Assumption \ref{asmp1}, we establish the following auxiliary lemmas on the random initialization $W(0)$ and the stationary distribution $\mu$, which plays a key role in quantifying the error of local linearization. 
\begin{lemma}\label{mulm1}
There exists a constant $c_1>0$ such that for any random vector $W$ with $\|W-W(0)\|_2\le B$, it holds that
\$
\EE_{\text{init},\mu}\Bigl[ \frac{1}{m}\sum_{r=1}^m \ind\{|W_r(0)^\top x|\le  \|W_r-W_r(0)\|_2\} \Bigr]\le c_1B\cdot m^{-1/2}.
\$
\end{lemma}
\begin{proof}
By Assumption \ref{asmp1}, we have
\#\label{57129}
&\EE_{\text{init},\mu}\Bigl[ \frac{1}{m}\sum_{r=1}^m \ind\{|W_r(0)^\top x|\le  \|W_r-W_r(0)\|_2\} \Bigr]\notag\\
&\quad\le 
\EE_{\text{init}}\Bigl[ \frac{1}{m}\sum_{r=1}^m  c_0\cdot\|W_r-W_r(0)\|_2/\|W_r(0)\|_2 \Bigr].
\#
Applying H\"{o}lder's inequality to the right-hand side, we obtain
\#\label{4101136}
&\EE_{\text{init},\mu}\Bigl[ \frac{1}{m}\sum_{r=1}^m \ind\{|W_r(0)^\top x|\le  \|W_r-W_r(0)\|_2\} \Bigr]\notag\\
&\quad\le{c_0}/{m}\cdot \EE_{\text{init}}\Bigl[ \Bigl(\sum_{r=1}^m \|W_r-W_r(0)\|_2^2 \Bigr)^{1/2}\cdot\Bigl( \sum_{r=1}^m \frac{1}{\|W_r(0)\|_2^2}\Bigr)^{1/2}\Bigr]\notag\\
&\quad\le c_0 B \cdot m^{-1/2}\cdot \EE_{w\sim N(0,I_d/d)}\bigl[1/\|w\|_2^2\bigr]^{1/2},
\#
where the second inequality follows from
\#\label{4101247}
\EE_{\text{init}}\Bigl[ \Bigl(
\sum_{r=1}^m \frac{1}{\|W_r(0)\|_2^2}\Bigr)^{1/2}
\Bigr]\le \EE_{\text{init}}\Bigl[ 
\sum_{r=1}^m \frac{1}{\|W_r(0)\|_2^2}
\Bigr]^{1/2}=\sqrt{m}\cdot\EE_{w\sim N(0,I_d/d)}\bigl[1/\|w\|_2^2\bigr]^{1/2}.
\#
Setting $c_1=c_0\cdot \EE_{w\sim N(0,I_d/d)}[1/\|w\|_2^2]^{1/2}$, we complete the proof of Lemma \ref{mulm1}.
\end{proof}

\begin{lemma}\label{mulm2}
There exists a constant $c_2>0$ such that for any random vector $W$ with $\|W-W(0)\|_2\le B$, it holds that
\$
\EE_{\text{init}}\biggl[ \EE_{\mu}\bigl[\hat{Q}_0(x)^2\bigr]\cdot\EE_{\mu}\Bigl[\frac{1}{m}\sum_{r=1}^m \ind\{|W_r(0)^\top x|\le  \|W_r-W_r(0)\|_2\} \Bigr]\biggr]\le c_2B\cdot m^{-1/2}.
\$
\end{lemma}
\begin{proof}
By the definition of $\hat{Q}_0(x) = \hat{Q}_0(x; W(0))$ in \eqref{q0}, we have
\$
\EE_{\mu}\bigl[\hat{Q}_0(x)^2\bigr]
=1/m\cdot \EE_{\mu}\Bigl[
\sum_{r=1}^{m} \sigma\bigl(W_r(0)^\top x\bigr)^2+\sum_{r\neq s} b_rb_s \sigma\bigl(W_r(0)^\top x\bigr)\sigma\bigl(W_s(0)^\top x\bigr)
 \Bigr]. 
\$
Following the same derivation of \eqref{57129} and \eqref{4101136}, we have
\$
&\EE_{\text{init}}\biggl[ \EE_{\mu}\bigl[\hat{Q}_0(x)^2\bigr]\cdot\EE_{\mu}\Bigl[\frac{1}{m}\sum_{r=1}^m \ind\{|W_r(0)^\top x|\le  \|W_r-W_r(0)\|_2\} \Bigr]\biggr]\\
&\quad\le \EE_{\text{init}}\biggl[ 1/m\cdot 
\EE_{\mu}\Bigl[
\sum_{r=1}^{m} \sigma\bigl(W_r(0)^\top x\bigr)^2+\sum_{r\neq s} b_rb_s \sigma\bigl(W_r(0)^\top x\bigr)\sigma\bigl(W_s(0)^\top x\bigr)
 \Bigr] \\
 &\quad\qquad\qquad\cdot c_0/m\cdot \Bigl( \sum_{r=1}^m \|W_r-W_r(0)\|_2^2 \Bigr)^{1/2}\cdot \Bigl( \sum_{r=1}^m \frac{1}{\|W_r(0)\|_2^2}\Bigr)^{1/2}
\biggr].
\$
Note that $b_r$ and $b_s$ are independent of $W(0)$ and $\EE_{\text{init}}[b_rb_s]=0$. Thus, we obtain
\$
&\EE_{\text{init}}\biggl[ \EE_{\mu}\bigl[\hat{Q}_0(x)^2\bigr]\cdot\EE_{\mu}\Bigl[\frac{1}{m}\sum_{r=1}^m \ind\{|W_r(0)^\top x|\le  \|W_r-W_r(0)\|_2\} \Bigr]\biggr]\\
&\quad\le c_0B/m^2 \cdot \EE_{\text{init}}\biggl[ 
\EE_{\mu}\Bigl[
\sum_{r=1}^{m} \sigma\bigl(W_r(0)^\top x\bigr)^2
 \Bigr] \cdot \Bigl( \sum_{r=1}^m \frac{1}{\|W_r(0)\|_2^2}\Bigr)^{1/2}
\biggr].
\$
By the definition of $\sigma(W_r(0)^\top x)$ and the fact that $\|x\|_2=1$, we have
\$
\EE_{\mu}\Bigl[
\sum_{r=1}^{m} \sigma\bigl(W_r(0)^\top x\bigr)^2
 \Bigr]\le
 \sum_{r=1}^m \|W_r(0)\|_2^2.
\$
Hence, it holds that
\#\label{518530}
&\EE_{\text{init}}\biggl[ \EE_{\mu}\bigl[\hat{Q}_0(x)^2\bigr]\cdot\EE_{\mu}\Bigl[\frac{1}{m}\sum_{r=1}^m \ind\{|W_r(0)^\top x|\le  \|W_r-W_r(0)\|_2\} \Bigr]\biggr]\notag\\
&\quad\le c_0B/m^2 \cdot \EE_{\text{init}}\Bigl[ 
\Bigl( \sum_{r=1}^m \|W_r(0)\|_2^2\Bigr)\cdot \Bigl( \sum_{r=1}^m \frac{1}{\|W_r(0)\|_2^2}\Bigr)^{1/2}
\Bigr]\notag\\
&\quad\le
 c_0B/m^2 \cdot \EE_{\text{init}}\Bigl[ 
\Bigl( \sum_{r=1}^m \|W_r(0)\|_2^2\Bigr)^2\Bigr]^{1/2}\cdot \EE_{\text{init}}\Bigl[ \sum_{r=1}^m \frac{1}{\|W_r(0)\|_2^2}\Bigr]^{1/2}.
\#
By \eqref{4101247} and the fact that 
\$
\EE_{\text{init}}\Bigl[ 
\Bigl( \sum_{r=1}^m \|W_r(0)\|_2^2\Bigr)^2\Bigr]=m\cdot \EE_{w\sim N(0,I_d/d)}\bigl[\|w\|_2^4\bigr]+m(m-1)\cdot \EE_{w\sim N(0,I_d/d)}\bigl[\|w\|_2^2\bigr]^2=O(m^2),
\$
the right-hand side of \eqref{518530} is $O(Bm^{-1/2})$. Setting 
\$c_2=c_0\cdot
\Bigl( \EE_{w\sim N(0,I_d/d)}\bigl[\|w\|_2^4\bigr]+\EE_{w\sim N(0,I_d/d)}\bigl[\|w\|_2^2\bigr]^2\Bigr)^{1/2}\cdot\EE_{w\sim N(0,I_d/d)}\bigl[1/\|w\|_2^2\bigr]^{1/2}, 
\$
 we complete the proof of Lemma \ref{mulm2}.
\end{proof}

\begin{lemma}\label{mulm3}
For any random vector $W$ with $\|W-W(0)\|_2\le B$, we have
\#\label{44146}
\EE_{\text{init}}\biggl[ \EE_{s\sim\mu}\bigl[\max_{a\sim\cA}\hat{Q}_0(s,a)^2\bigr]\cdot\EE_{\mu}\Bigl[\frac{1}{m}\sum_{r=1}^m \ind\{|W_r(0)^\top x|\le  \|W_r-W_r(0)\|_2\} \Bigr]\biggr]=O(Bm^{-1/2}),\notag\\
\EE_{\text{init}}\biggl[ \EE_{s\sim\mu}\bigl[\softmax_{a\sim\cA}\hat{Q}_0(s,a)^2\bigr]\cdot\EE_{\mu}\Bigl[\frac{1}{m}\sum_{r=1}^m \ind\{|W_r(0)^\top x|\le  \|W_r-W_r(0)\|_2\} \Bigr]\biggr]=O(Bm^{-1/2}).\notag
\#
\end{lemma}
\begin{proof}
The proof mirrors that of Lemma \ref{mulm2}. We utilize the fact that $\cA$ is finite, so that the expectation of the maximum can be upper bounded by a finite sum of expectations,
\$
 &\EE_{s\sim\mu}\bigl[\max_{a\in\cA}\hat{Q}_0(s,a)^2\bigr]\le
  \EE_{s\sim\mu}\Bigl[\sum_{a\in\cA}\hat{Q}_0(s,a)^2\Bigr]\le \sum_{a\in\cA} \EE_{s\sim\mu}\bigl[\hat{Q}_0(s,a)^2\bigr].\\
  &\EE_{s\sim\mu}\bigl[\softmax_{a\in\cA}\hat{Q}_0(s,a)^2\bigr]\le \sum_{a\in\cA} \EE_{s\sim\mu}\bigl[\hat{Q}_0(s,a)^2\bigr]+\beta^{-1}\cdot\log|\cA|.
\$
For each expectation $\EE_{s\sim\mu}[\hat{Q}_0(s,a)^2]$, note that the distribution of $(s, a)$ is independent of the initialization. Hence, the same proof of Lemma \ref{mulm2} is applicable, as $\EE_{s\sim\mu}[\hat{Q}_0(s,a)^2]$ plays the same role of $\EE_{\mu}[\hat{Q}_0(x)^2]$. Thus, we obtain the same upper bound in Lemmas \ref{mulm1} and \ref{mulm2} except for an extra factor involving $|\cA|$, which however does not change the order of $m$.
\end{proof}

\subsection{Proof of Lemma \ref{maxqdiff}}\label{proof:aux}
\begin{proof}
In the proof of Lemma \ref{qdiff}, we show that, for any $s\in\cS$ and $a\in\cA$,
\$
\bigl|\hat{Q}_t(s,a)-\hat{Q}_0\bigl(s,a;W(t)\bigr)\bigr|^2&\le 
\frac{4B^2}{m} \sum_{r=1}^m \ind\{|W_r(0)^\top x|\le  \|W_r(t)-W_r(0)\|_2\}.
\$
Taking maximum over $a$, we obtain
\$
\max_{a\in\cA}\bigl|\hat{Q}_t(s,a)-\hat{Q}_0\bigl(s,a;W(t)\bigr)\bigr|^2&\le 
\frac{4B^2}{m} \sum_{r=1}^m \max_{a\in\cA}\ind\{|W_r(0)^\top \psi(s,a)|\le  \|W_r(t)-W_r(0)\|_2\}.
\$
Taking expectation with respect to the random initialization and the stationary distribution of $s$, we obtain
\$
&\EE_{\text{init},s\sim\mue}\Bigl[ \max_{a\in\cA}\bigl|\hat{Q}_t(s,a)-\hat{Q}_0\bigl(s,a;W(t)\bigr)\bigr|^2\Bigr]\\
&\quad
\le \EE_{\text{init},s\sim\mue}\Bigl[ \frac{4B^2}{m} \sum_{r=1}^m \max_{a\in\cA}\ind\{|W_r(0)^\top \psi(s,a)|\le  \|W_r(t)-W_r(0)\|_2\}\Bigr].
\$
By Assumption \ref{asmp2}, it holds that
\$
&\EE_{\text{init},s\sim\mue}\Bigl[ \max_{a\in\cA}\bigl|\hat{Q}_t(s,a)-\hat{Q}_0\bigl(s,a;W(t)\bigr)\bigr|^2\Bigr] \le
\EE_{\text{init}}\Bigl[ 
\frac{4B^2}{m} \sum_{r=1}^m  c_3 \cdot \|W_r(t)-W_r(0)\|_2/\|W_r(0)\|_2
\Bigr].
\$
Applying H\"{o}lder's inequality to the right-hand side, we obtain
\$
&\EE_{\text{init},s\sim\mue}\Bigl[ \max_{a\in\cA}\bigl|\hat{Q}_t(s,a)-\hat{Q}_0\bigl(s,a;W(t)\bigr)\bigr|^2\Bigr]\\
&\quad\le{4B^2c_3}/{m}\cdot\EE_{\text{init}}\Bigl[ \Bigl(\sum_{r=1}^m \|W_r-W_r(0)\|_2^2 \Bigr)^{1/2}\cdot\Bigl( \sum_{r=1}^m \frac{1}{\|W_r(0)\|_2^2}\Bigr)^{1/2}\Bigr]\\
&\quad\le {4B^3c_3}\cdot m^{-1/2}\cdot \EE_{w\sim N(0,I_d/d)}\bigl[1/\|w\|_2^2\bigr]^{1/2}.
\$
Setting $c_4=4c_3\cdot \EE_{w\sim N(0,I_d/d)}[1/\|w\|_2^2]^{1/2}$, we finish the proof of Lemma \ref{maxqdiff}.
\end{proof}

\end{appendix}

\end{document}